\documentclass{article} 

\usepackage[numbers]{natbib}
\usepackage{fullpage}

\usepackage{amsmath,amsfonts,bm}

\def\eqref#1{equation~\ref{#1}}
\def\Eqref#1{Equation~\ref{#1}}

\def\1{\bm{1}}

\def\eps{{\epsilon}}

\DeclareMathAlphabet{\mathsfit}{\encodingdefault}{\sfdefault}{m}{sl}
\SetMathAlphabet{\mathsfit}{bold}{\encodingdefault}{\sfdefault}{bx}{n}

\newcommand{\E}{\mathbb{E}}

\newcommand{\R}{\mathbb{R}}

\DeclareMathOperator*{\argmin}{arg\,min}

\usepackage{amsmath,amssymb}
\usepackage{hyperref}
\usepackage{url}
\usepackage{xcolor} 
\usepackage{mathtools}
\usepackage{wrapfig}

\definecolor{DarkGreen}{rgb}{0.1,0.5,0.1}

\newcommand{\mimic}{\text{MIMIC-III }}
\newcommand{\dpgan}{\text{DP-auto-GAN }} 

\newcommand{\wh}[1]{\widehat{#1}}
\newcommand{\wb}[1]{\overline{#1}}

\DeclarePairedDelimiter\adfloor{\lfloor}{\rfloor}
\DeclarePairedDelimiter\abs{\lvert}{\rvert}

\def\BB{\mathcal{B}}

\def\MM{\mathcal{M}}\def\NN{\mathcal{N}}

\def\XX{\mathcal{X}}

\def\Ebb{\mathbb{E}}
\def\Ibb{\mathbb{I}}

\def\Rbb{\mathbb{R}}

\usepackage{multirow}
\usepackage{amsthm}
\usepackage{thm-restate}
\usepackage{subcaption}

\newtheorem{theorem}{Theorem}
\newtheorem{proposition}[theorem]{Proposition}

\newtheorem{definition}[theorem]{Definition}

\def\M{\mathcal{M}}

\usepackage{graphicx}

\usepackage{algorithm}
\usepackage[noend]{algpseudocode}

\usepackage{xparse}
\makeatletter
\NewDocumentCommand{\LeftComment}{s m}{%
  \Statex \IfBooleanF{#1}{\hspace*{\ALG@thistlm}}\(\triangleright\) #2}
\makeatother

\begin{document}

\title{Differentially Private Synthetic Mixed-Type Data Generation For Unsupervised Learning}


\author{Uthaipon Tao Tantipongpipat\footnotemark[1] \footnotemark[6] \and Chris Waites\footnotemark[2] \footnotemark[6] \and Digvijay Boob\footnotemark[3] \and Amaresh Ankit Siva\footnotemark[4] \and Rachel Cummings\footnotemark[5]}

\renewcommand{\thefootnote}{\fnsymbol{footnote}}

\footnotetext[1]{Email: \texttt{uthaipon@gmail.com}. This work was completed while the author was a student at the Georgia Institute of Technology. Supported in part by NSF grants AF-1910423 and AF-1717947.}
\footnotetext[2]{Stanford University. Email: \texttt{waites@stanford.edu}. This work was completed while the author was a student at the Georgia Institute of Technology. Supported in part by a President's Undergraduate Research Award from the Georgia Institute of Technology.}
\footnotetext[3]{Southern Methodist University. Email: \texttt{dboob@mail.smu.edu}. This work was completed while the author was a student at the Georgia Institute of Technology. Supported in part by NSF grant CCF-1909298.}
\footnotetext[4]{Amazon. Email: \texttt{ankitsiv@amazon.com}.  This work was completed while the author was a student at the Georgia Institute of Technology.}
\footnotetext[5]{Georgia Institute of Technology. Email: \texttt{rachelc@gatech.edu}. Supported in part by a Mozilla Research Grant, a Google Research Fellowship, a JPMorgan Chase Faculty Award, and NSF grant CNS-1850187. Part of this work was completed while the author was visiting the Simons Institute for the Theory of Computing.}
\footnotetext[6]{Equal contribution}


\renewcommand{\thefootnote}{\arabic{footnote}}

\maketitle

\begin{abstract}
We introduce the DP-auto-GAN framework for synthetic data generation, which combines the low dimensional representation of autoencoders with the flexibility of Generative Adversarial Networks (GANs).  This framework can be used to take in raw sensitive data and privately train a model for generating synthetic data that will satisfy similar statistical properties as the original data.  This learned model can generate an arbitrary amount of synthetic data, which can then be freely shared due to the post-processing guarantee of differential privacy.  Our framework is applicable to unlabeled \emph{mixed-type data}, that may include binary, categorical, and real-valued data.  We implement this framework on both  binary data (MIMIC-III) and  mixed-type data (ADULT), and  compare its performance with existing  private algorithms on  metrics in unsupervised settings. We also introduce a new quantitative metric  able to detect diversity, or lack thereof, of synthetic data.
\end{abstract}

\section{Introduction}


As data storage and analysis are becoming more cost effective, and data become more complex and unstructured, there is a growing need for sharing large datasets for research and learning purposes. This is in stark contrast to the previous statistical model where a data curator would hold datasets and answer specific queries from (potentially external) analysts. Sharing entire datasets allows analysts the freedom to perform their analyses in-house with their own devices and toolkits, without having to pre-specify the analyses they wish to perform.  However, datasets are often proprietary or sensitive, and they cannot be shared directly.  This motivates the need for \emph{synthetic data generation}, where a new dataset is created that shares the same statistical properties as the original data.  These data may not be of a single type: all binary, all categorial, or all real-valued; instead they may be of \emph{mixed-types}, containing data of multiple types in a single dataset.  These data may also be unlabeled, requiring techniques for \emph{unsupervised learning}, which is typically a more challenging task than \textit{supervised} learning when data are labeled.

Privacy challenges naturally arise when sharing highly sensitive datasets about individuals. Ad hoc anonymization techniques have repeatedly led to severe privacy violations when sharing ``anonymized'' datasets.  Notable examples include the Netflix Challenge \citep{NS08}, AOL Search Logs \citep{nyt}, and Massachusetts State Health data \citep{Ohm10}, where linkage attacks to publicly available auxiliary datasets were used to reidentify individuals in the dataset. Even deep learning models have been shown to inadvertently memoize sensitive personal information such as Social Security Numbers during training \citep{carlini2019secret}.

Differential privacy (DP) \citep{DMNS06} (formally defined in Section \ref{prelim}) has become the de facto gold standard of privacy in the computer science literature.  Informally, it bounds the extent to which an algorithm depends on a single datapoint in its training set. The guarantee ensures that any differentially privately learned models do not overfit to individuals in the database, and therefore cannot reveal sensitive information about individuals.  It is an information theoretic notion that does not rely on any assumptions of an adversary's computational power or auxiliary knowledge.  Furthermore, it has been shown empirically that training machine learning models with differential privacy protects against membership inference and model inversion attacks  \citep{triastcyn2018generating,carlini2019secret}. Differentially private algorithms have been deployed at large scale in practice by organizations such as Apple, Google, Microsoft, Uber, and the U.S. Census Bureau.




Much of the prior work on differentially private synthetic data generation has been either theoretical algorithms for highly structured classes of queries \citep{BLR08,HR10} or based on deep generative models such as Generative Adversarial Networks (GANs) or autoencoders.  These architectures have been primarily designed for either all-binary or all-real-valued datasets, and have focused on the \emph{supervised} setting.

In this work we introduce the \emph{DP-auto-GAN} framework, which combines the low dimensional representation of autoencoders with the flexibility of GANs.  This framework can be used to take in raw sensitive data, and privately train a model for generating synthetic data that satisfies similar statistical properties as the original data.  This learned model can be used to generate arbitrary amounts of publicly available synthetic data, which can then be freely shared due to the post-processing guarantees of differential privacy.  We implement this framework on both unlabeled binary data (for comparison with previous work) and unlabeled mixed-type data.  We also introduce new metrics for evaluating the quality of synthetic mixed-type data in unsupervised settings, and empirically evaluate the performance of our algorithm according to these metrics on two datasets.

\subsection{Our Contributions}

 
 In this work, we provide two main contributions: a new algorithmic framework for privately generating synthetic data, and empirical evaluations of our algorithmic framework showing improvements over prior work. Along the way, we also develop a novel privacy composition method with tighter guarantees, and we generalize previous metrics for evaluating the quality of synthetic datasets to the unsupervised mixed-type data setting.  Both of these contributions may be of independent interest.


\textbf{Algorithmic Framework.}
We propose a new data generation architecture which combines the versatility of an autoencoder \citep{kingma2013auto} with the recent success of GANs on complex data. Our model extends previous autoencoder-based DP data generation \citep{abay2018privacy,chen2018differentially} by removing an assumption that the distribution of the latent space follows a Gaussian mixture distribution. Instead, we incorporate GANs into the autoencoder framework so that the generator must learn the true latent distribution against the discriminator. We describe the composition analysis of differential privacy when the training consists of optimizing both autoencoders and GANs (with different noise parameters). 


\textbf{Empirical Results.}
We empirically evaluate the performance of our algorithmic framework on the MIMIC-III medical dataset \citep{johnson2016mimic} and UCI ADULT Census dataset \citep{uci_ml}, and compare against previous approaches in the literature \citep{frigerio2019differentially,xie2018differentially,abay2018privacy,acs2018differentially}.  Our experiments show that our algorithms perform better and obtain substantially improved $\eps$ values of \(\eps \approx 1\), compared to  \(\eps \approx 200\) in prior work \citep{xie2018differentially}. The performance of our algorithm remains high along a variety of quantitative and qualitative metrics, even for small values of  \(\eps\), corresponding to strong privacy guarantees. Our code is publicly available for future use and research.


\subsection{Related Work on Differentially Private Data Generation} \label{sec:related-work}

%
Early work on differentially private synthetic data generation was focused primarily on theoretical algorithms for solving the \emph{query release problem} of privately and accurately answering a large class of pre-specified queries on a given database. It was discovered that generating synthetic data on which the queries could be evaluated allowed for better privacy composition than simply answering all the queries directly \citep{BLR08,HR10,HLM12,GGHRW14}. Bayesian inference has also been used for differentially private data generation \citep{zhang2017privbayes,ping2017datasynthesizer} by estimating the correlation between features. See \cite{surendra2017review} for a survey of techniques used in private synthetic data generation.

More recently, \cite{abadi2016deep} introduced a framework for training deep learning models with differential privacy, which involved adding Gaussian noise to a clipped (norm-bounded) gradient in each training step. \cite{abadi2016deep} also introduced the \emph{moment accountant} privacy analysis, which provided a tighter Gaussian-based privacy composition and allowed for significant improvements in accuracy. It was later defined in terms of \emph{Renyi Differential Privacy (RDP)} \citep{mironov2017renyi}, which is a slight variant of differential privacy designed for easy composition. Much of the work that followed used deep generative models, and can be broadly categorized into two types: autoencoder-based and GAN-based. Our algorithmic framework is the first to combine both DP GANs and autoencoders.

\textbf{Differentially Private Autoencoder-Based Models.} A variational autoencoder (VaE) \citep{kingma2013auto} is a generative model that compresses high-dimensional data to a smaller space called \textit{latent space}. The compression is commonly achieved through deep models and can be differentially privately trained \citep{chen2018differentially,acs2018differentially}. VaE makes the (often unrealistic) assumption that the \textit{latent distribution} is Gaussian. \citet{acs2018differentially} uses Restricted Boltzmann machine (RBM) to learn the latent Gaussian distribution, and \citet{abay2018privacy} uses expectation maximization to learn a Gaussian mixture. Our work extends this line of work by additionally incorporating the generative model GANs which have also been shown to be successful in learning latent distributions. 

\textbf{Differentially Private GANs.}
GANs are generative models proposed by \citet{GAN14} that have been shown success in generating several different types of data \citep{mogren2016c,saito2017temporal,salimans2016improved,jang2016categorical,kusner2016gans,wang2018graphgan,xu2019modeling,lim2018doping,park2018data}.
As with other deep models, GANs can be trained privately using the aforementioned private stochastic gradient descent (formally introduced in Section \ref{sec:dp-sgd}). In this work, we focus on and compare to previous works where DP have been applied.   

Variants of DP GANs have been used for synthetic data generation, including the Wasserstein GAN (WGAN) \citep{arjovsky2017wasserstein,Gulrajani17} and DP-WGAN \citep{uclanesl_dp_wgan,triastcyn2018generating} that use a Wasserstein-distance-based loss function in training \citep{arjovsky2017wasserstein,Gulrajani17,uclanesl_dp_wgan,triastcyn2018generating}; the conditional GAN (CGAN) \citep{mirza2014conditional} and DP-CGAN \citep{torkzadehmahani2019dp} that operate in a supervised (labeled) setting and use labels as auxiliary information in training; and Private Aggregation of Teacher Ensembles (PATE) \citep{papernot2016semi,papernot2018scalable} for the semi-supervised setting of multi-label classification when some unlabelled public data are available (or PATEGAN \citep{jordon2018pate} when no public data are available).  Our work focuses on the unsupervised setting where data are unlabeled, and no (relevant) labeled public data are available.

These existing works on differentially private synthetic data generation are summarized in Table \ref{tab:data-gen}.

\begin{table}[h]
\caption{Algorithmic frameworks for differentially private synthetic data generation. Our new algorithmic framework (in \textbf{bold}) is the first to combine both DP GANs and autoencoders into one framework by using GANs to learn a generative model in the latent space.}
\label{tab:data-gen}
\begin{center}
\begin{tabular}{|p{1.5cm}|p{2cm}|p{6.5cm}|}
\hline
\multirow{2}{*}{\bf \ Types}  &\multicolumn{2}{c|}{\bf Algorithmic framework} \\ \cline{2-3}
& \multicolumn{1}{c|}{\textbf{Architecture}} & \multicolumn{1}{c|}{\textbf{Variants}}
\\ \hline
\multirow{8}{\linewidth}{Deep generative models}         &\multirow{4}{\linewidth}{DPGAN \citep{abadi2016deep} } & PATEGAN \citep{jordon2018pate} \\
& & DP Wasserstein GAN \citep{uclanesl_dp_wgan}\\
& & DP Conditional GAN \citep{torkzadehmahani2019dp} \\
& & Gumbel-softmax for categorical data \citep{frigerio2019differentially} \\ \cline{2-3}
& \multirow{3}{\linewidth}{Autoencoder}    & DP-VAE (standard Gaussian as a generative model in latent space) \citep{chen2018differentially,acs2018differentially} \\
& &                RBM generative models in latent space \citep{acs2018differentially} \\
& &                Mixture of Gaussian model in latent space \citep{abay2018privacy} 
\\ \cline{2-3}
& \multicolumn{2}{l|}{Autoencoder and DPGAN \textbf{(ours)}}
\\ \hline
Other  models & \multicolumn{2}{p{8.5cm}|}{SmallDB \citep{BLR08}, PMW \citep{HR10}, MWEM \cite{HLM12}, DualQuery \cite{GGHRW14}, DataSynthesizer \citep{ping2017datasynthesizer}, PrivBayes \citep{zhang2017privbayes}}\\ \hline
\end{tabular}
\end{center}
\end{table}

\paragraph{Differentially Private Generation of Mixed-Type Data.} 
Next we describe the three most relevant recent works on privately generating synthetic mixed-type data. For a full overview of related work, see Appendix \ref{app.rel}. \cite{abay2018privacy} considers the problem of generating mixed-type labeled data with $k$ possible labels. Their algorithm, DP-SYN, partitions the dataset into $k$ sets based on the labels and trains a DP autoencoder on each partition. Then the DP expectation maximization (DP-EM) algorithm of \cite{park17EM} is used to learn the distribution in the latent space of encoded data of the given label-class. 
The main workhorse, DM-EM algorithm, is designed and analyzed for Gaussian mixture models and more general factor analysis models. \cite{chen2018differentially} works in the same setting, but replaces the DP autoencoder and DP-EM with a DP variational autoencoder (DP-VAE). Their algorithm assumes that the mapping from real data to the Gaussian distribution can be efficiently learned by the encoder.
Finally, \cite{frigerio2019differentially} uses a Wasserstein GAN (WGAN)  to generate differentially private mixed-type synthetic data, which uses a Wasserstein-distance-based loss function in training. Their algorithmic framework privatizes the WGAN using DP-SGD, similar to previous approaches for image datasets \citep{zhang2018differentially,xie2018differentially}. The methodology of \cite{frigerio2019differentially} for generating mixed-type synthetic data involves two main ingredients: changing discrete (categorical) data to binary data using one-hot encoding, and adding an output softmax layer to the WGAN generator for every discrete variable.

Our framework is distinct from these three approaches. We use a differentially private autoencoder which, unlike DP-VAE of \cite{chen2018differentially}, does not require mapping data to a Gaussian distribution. This allows us to reduce the dimension of the problem handled by the WGAN, hence escaping the issues of high-dimensionality from the one-hot encoding of \cite{frigerio2019differentially}. We also use DP-GAN, replacing DP-EM in \cite{abay2018privacy}, to  learn more complex distributions in the latent or encoded space.

\textbf{NIST Differential Privacy Synthetic Data Challenge.} The National Institute of Standards and Technology (NIST) recently hosted a challenge to find methods for privately generating synthetic mixed-type data \cite{NIST2018Match3}, using excerpts from the Integrated Public Use Microdata Sample (IPUMS) of the 1940 U.S. Census Data as training and test datasets.  Four of the winning solutions have been made publicly available with open-source code \cite{nistcode}.  However, all of these approaches are highly tailored to the specific datasets and evaluation metrics used in the challenge, including specialized data pre-processing methods and hard-coding details of the dataset in the algorithm.  As a result, they do not provide general-purpose methods for differentially private synthetic data generation, and it would be inappropriate--if not impossible--to use any of these algorithms as baseline for other datasets such as ones we consider in this paper.


\paragraph{Evaluation Metrics for Synthetic Data.} Various evaluation metrics have been considered in the literature to quantify the quality of the synthetic data (see \citet{charest2011can} for a survey). The metrics can be broadly categorized into two groups: \emph{supervised} and \emph{unsupervised}.  Supervised evaluation metrics are used when there are clear distinctions between features and labels of the dataset, e.g., for healthcare applications, a person's disease status is a natural label. In these settings, a predictive model is typically trained on the synthetic data, and its accuracy is measured with respect to the real (test) dataset. Unsupervised evaluation metrics are used when no feature of the data can be decisively termed as a label. Recently proposed metrics include \emph{dimension-wise probability} for binary data \citep{choi2017generating}, which compares the marginal distribution of real and synthetic data on each individual feature, and \emph{dimension-wise prediction} which measures how closely synthetic data captures relationships between features in the real data. This metric was proposed for binary data, and we extend it here to mixed-type data. Recently, \citet{NIST2018Match3} used a 3-way marginal evaluation metric which used three random features of the real and synthetic datasets to compute the total variation distance as a statistical score. See Appendix \ref{app.met} for more details on both categories of metrics, including Table \ref{tab:metric} which summarizes the metrics' applicability to various data types.

\section{Preliminaries on Differential Privacy} \label{prelim}

In the setting of differential privacy, a dataset \(X\) consists of \(m\) individuals' sensitive information, and two datasets are neighbors if one can be obtained from the other by the addition or deletion of one datapoint. Differential privacy requires that an algorithm produce similar outputs on neighboring datasets, thus ensuring that the output does not overfit to its input dataset, and that the algorithm learns from the population but not from the individuals.

\begin{definition}[Differential privacy \citep{DMNS06}]
For $\eps,\delta>0$, an algorithm $\mathcal{M}$ is  \emph{$(\eps, \delta)$-differentially private} if for any pair of neighboring databases $X,X'$ and any subset $S$ of possible outputs produced by \(\M\), 
$$ \Pr[\mathcal{M}(X) \in S] \leq e^\eps \cdot \Pr[\mathcal{M}(X') \in S] + \delta .$$
\end{definition}

A smaller value of $\epsilon$ implies stronger privacy guarantees (as the constraint above binds more tightly), but usually corresponds with decreased accuracy, relative to non-private algorithms or the same algorithm run with a larger value of $\epsilon$. Differential privacy is typically achieved by adding random noise that scales with the \emph{sensitivity} of the computation being performed, which is the maximum change in the output value that can be caused by changing a single entry.  Differential privacy has strong \emph{composition guarantees}, meaning that the privacy parameters degrade gracefully as additional algorithms are run on the same dataset.  It also has a \emph{post-processing} guarantee, meaning that any function of a differentially private output will retain the same privacy guarantees.

%
%
%

\subsection{Differentially Private Stochastic Gradient Descent (DP-SGD)} \label{sec:dp-sgd}




Training deep learning models reduces to minimizing some (empirical) loss function \(f(X;\theta):=\frac 1m\sum_{i=1}^m f(x_i;\theta)\) on a dataset \(X=\{x_i\in\R^n\}_{i=1}^m\). Typically \(f\) is a nonconvex function, and a common method to minimize \(f\) is by iteratively performing stochastic gradient descent (SGD) on a batch $B$ of sampled data points:
\begin{align}
B &\leftarrow \text{\textsc{BatchSample}}(X) \nonumber \\
\theta &\leftarrow \theta - \eta\cdot \tfrac {1}{|B|}\textstyle\sum_{i\in B} \nabla_\theta f(x_i,\theta) \label{eq:descent}
\end{align}
The size of \(B\) is typically fixed as a moderate number to ensure quick computation of gradient, while maintaining that \(\frac {1}{|B|}\sum_{i\in B} \nabla f(x_i,\theta)\) is a good estimate of true gradient \(\nabla_\theta f(X;\theta)\).



To make SGD private, \citet{abadi2016deep} proposed to first clip the gradient of each sample to ensure the \(\ell_2\)-norm is at most \(C\): 
\begin{equation*}
\text{\textsc{Clip}}(x,C):=x \cdot \min\left(1,C/{||x||_2}\right).
\end{equation*}
Then a multivariate Gaussian noise parametrized by noise multiplier \(\psi\) is added before taking an average across the batch, leading to noisy-clipped-averaged gradient estimate $g$:
\begin{equation*}
g \leftarrow\textstyle\frac{1}{|B|}\left(\sum_{i\in B}\text{\textsc{Clip}}( \nabla_\theta f(x_i,\theta),C) 
+\NN(0,C^2\psi^2I)\right).  
\end{equation*}
The quantity \(g\) is now private and can be used for the descent step \(\theta \leftarrow \theta-\eta \cdot g\) in place of \Eqref{eq:descent}.

\paragraph{Performance Improvements.} In general, the descent step can be performed using other optimization methods---such as Adam or RMSProp---in a private manner, by replacing the gradient value with \(g\) in each step. Also, one does not need to clip the individual gradients, but can instead clip the gradient of a group of datapoints, called a \emph{microbatch} \citep{mcmahan2018general}. Mathematically, the batch \(B\) is partitioned into microbatches \(B_1,\ldots,B_k\) each of size \(r\), and the gradient clipping is performed on the average of each microbatch: 
\begin{equation*}
g \leftarrow\frac {1}{k}\left(\textstyle\sum_{i=1 }^k \text{\textsc{Clip}}(\nabla_\theta f(X_{B_i},\theta) 
,C)+\NN(0,C^2\psi^2I)\right).  
\end{equation*}
Standard DP-SGD corresponds to setting \(r=1\), but setting higher values of \(r\) (while holding \(|B|\) fixed) significantly decreases the runtime and reduces the accuracy, and does not impact privacy significantly for large dataset. Other clipping strategies have also been suggested. We refer the interested reader to \cite{mcmahan2018general}  for more details of clipping and other optimization strategies.

The improved moment accountant privacy analysis by \cite{abadi2016deep} (which has been implemented in \citet{TensorflowPrivacy} and is widely used in practice) obtains a tighter privacy bound when data are subsampled, as in SGD. This analysis requires independently sampling each datapoint with a fixed probability \(q\) in each step.  
 
The DP-SGD framework (Algorithm \ref{alg:dp-sgd}) is generically applicable to private non-convex optimization. In our proposed model, we use this framework to train the autoencoder and GAN.

\begin{algorithm}
    \caption{\textsc{DP-SGD} (one iteration step)}
    \label{alg:dp-sgd}
    \begin{algorithmic}[1] 
        \State \textbf{parameter input}:
Dataset \(X=\{x_i\}_{i=1}^m\), deep learning model parameter \(\theta\), learning rate \(\eta\), loss function \(f\), optimization method \textsc{Optim}, batch sampling rate \(q\) (for the batch expectation size \(b=qm\)), clipping norm \(C\), noise multiplier \(\psi\), microbatch size \(r\)
\State \textbf{goal}: differentially privately  train one step of the model parametrized by \(\theta\) with \textsc{Optim}    
\Procedure{DP-SGD}{}
        \Procedure{SampleBatch}{$X,q$}
        \State \(\BB\leftarrow \{\}\)
               \For{\(i=1\ldots m\)}
               \State Add \(x_i\) to \(\BB\) with probability \(q\)
               \EndFor
               \State Return $\BB$
        \EndProcedure
                \State Partition \(\BB\) into \(B_1,\ldots,B_k\) each of size \(r\) (ignoring the dividend)
                \State \(\hat k \leftarrow \frac{qm}{r}\) \Comment{ an estimate of \(k\)}
        \State \(g \leftarrow\frac {1}{\hat{k}}\left(\sum_{i=1 }^k \text{\textsc{Clip}}(\nabla_\theta f(X_{B_i},\theta) 
,C)+\NN(0,C^2\psi^2I)\right)\)
        \State \(\theta \leftarrow \text{\textsc{Optim}}_{}(\theta,g,\eta)\)
\EndProcedure
\end{algorithmic}
\end{algorithm}

\subsection{Renyi Differential Privacy Accountant}\label{sec:rdp}

A variant notion of differential privacy, known as \emph{Renyi Differential Privacy (RDP)} \citep{mironov2017renyi}, is often used to analyze privacy for DP-SGD. A randomized mechanism \(\MM\) is \((\alpha,\eps)\)-RDP if for all neighboring databases \(X,X'\) that differ in at most one entry,
\begin{equation*}
RDP(\alpha) := D_\alpha(\MM(X)||\MM(X'))\leq \eps,
\end{equation*}
where \(D_\alpha(P||Q):=\frac{1}{\alpha-1} \log \E_{x\sim X} \left( \frac{P(x)}{Q(x)} \right)^\alpha\) is the \emph{Renyi divergence} or \emph{Renyi entropy} of order \(\alpha\) between two distributions \(P\) and \(Q\). Renyi divergence is better tailored to tightly capture the privacy loss from the Gaussian mechanism that is used in DG-SGD, and is a common analysis tool for DP-SGD literature. To compute the final \((\eps,\delta)\)-differential privacy parameters from iterative runs of DP-SGD, one must first compute the subsampled Renyi Divergence, then compose privacy under RDP, and then convert the RDP guarantee into DP.  

\textbf{Step 1: Subsampled Renyi Divergence.} Given sampling rate \(q\) and noise multiplier \(\psi\), one can obtain RDP privacy parameters as a function of \(\alpha\geq 1\) for one run of DP-SGD \citep{mironov2017renyi}. We denote this function by
\(\text{RDP}_{T=1}(\cdot)\), which will depend on $q$ and $\psi$.

\textbf{Step 2: Composition of RDP.}
When DP-SGD is run iteratively, we can compose the Renyi privacy parameter across all runs using the following proposition.
\begin{proposition}[\citep{mironov2017renyi}] \label{prop:rdp-compose}
If \(\MM_1,\MM_2\) respectively satisfy \((\alpha,\eps_1),(\alpha,\eps_2)\)-RDP for  \(\alpha\geq 1\), then the composition of two mechanisms \((\MM_2(X), \MM_1(X))\) satisfies \((\alpha,\eps_1+\eps_2)\)-RDP. 
\end{proposition}
Hence, we can compute RDP privacy parameters for \(T\) iterations of DP-SGD as $ \text{\textsc{RDP-Account}}(T,q_,\psi):=T\cdot \text{RDP}_{T=1}(\cdot)$.

\textbf{Step 3: Conversion to \((\eps,\delta)\)-DP.}
After obtaining an expression for the overall RDP privacy parameter values, any \((\alpha,\eps)\)-RDP guarantee can be converted into  \((\eps,\delta)\)-DP.
\begin{proposition}[\citep{mironov2017renyi}] \label{prop:rdp-dp}
If \(\MM \) satisfies \((\alpha,\eps)\)-RDP for  \(\alpha > 1\), then for all \(\delta>0\), \(\MM\) satisfies \((\eps+\frac{\log 1/\delta}{\alpha-1},\delta)\)-DP.
\end{proposition}
Since the \(\eps\) privacy parameter of RDP is also a function of \(\alpha\), this last step involves optimizing for the \(\alpha\) that achieves smallest privacy parameter in Proposition \ref{prop:rdp-dp}.



\section{Algorithmic Framework}\label{s.algo}

The overview of our algorithmic framework DP-auto-GAN is shown in Figure \ref{fig:alg-frame}, and the full details are given in Algorithm \ref{alg:all}. Details of subroutines  in Algorithm \ref{alg:all}  can be found in Appendix \ref{app.algodetail}. 

The algorithm takes in $m$ raw data points, and \emph{pre-processes} these points into $m$ vectors $x_1,\ldots,x_m \in \R^n$ to be read by DP-auto-GAN, where usually $n$ is very large. For example, categorical data may be pre-processed using one-hot encoding, or text may be converted into high-dimensional vectors.  Similarly, the output of DP-auto-GAN can be \emph{post-processed }from $\R^n$ back to the data's original form.  We assume that this pre- and post-processing can done based on public knowledge, such as possible categories for qualitative features and reasonable bounds on quantitative features, and therefore do not incur a privacy cost.

\begin{figure}[h]
\begin{center}
\includegraphics[width=\textwidth]{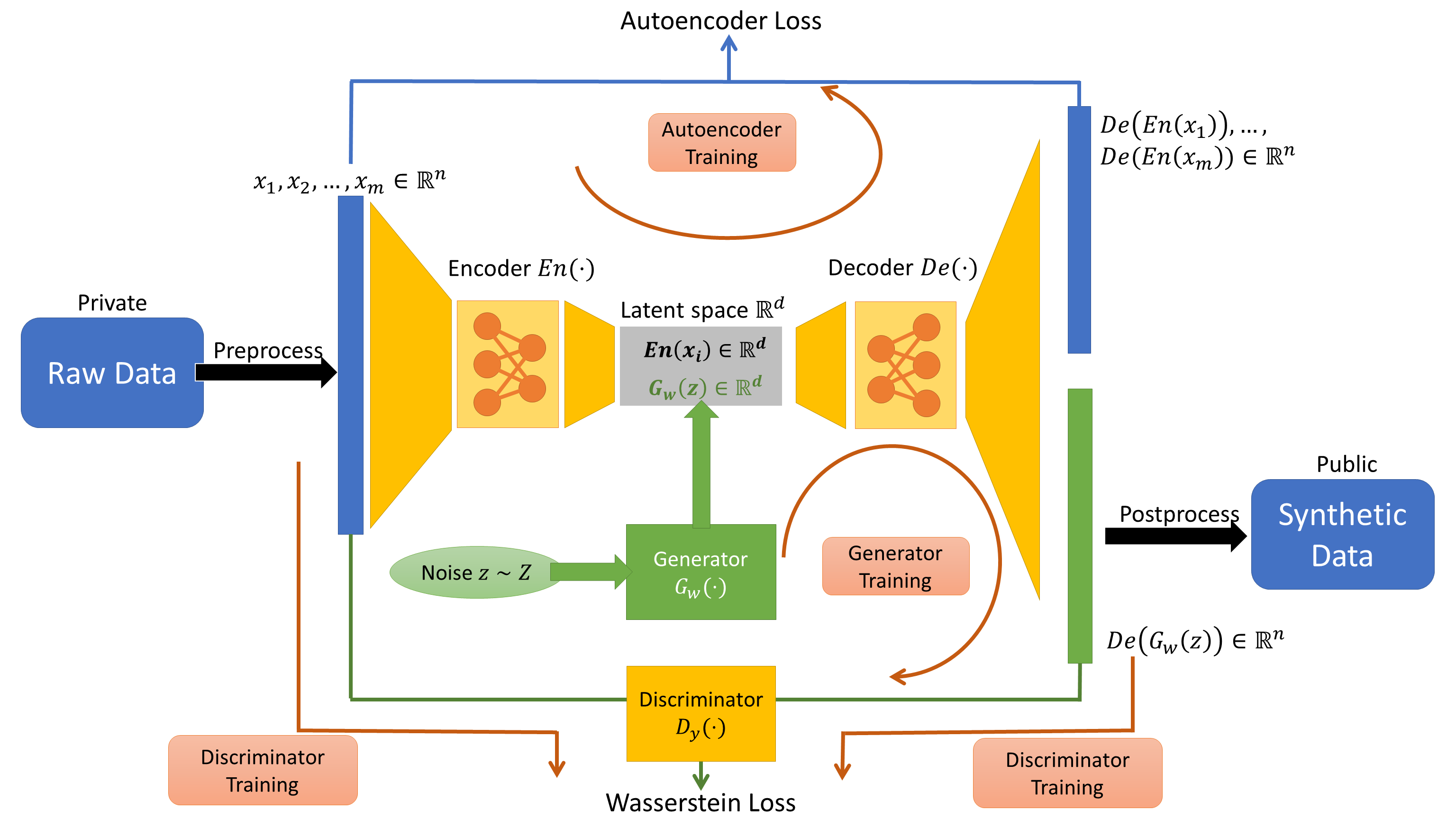}
\end{center}
\caption{The summary of our DP-auto-GAN algorithmic framework. Pre- and post-processing (in black) are assumed to be public knowledge. Generator (in green) is trained without noise, whereas encoder, decoder, and discriminator (in yellow) are trained with noise. The four red arrows indicate how data are forwarded for autoencoder, generator, and discriminator training. After training, the generator and decoder are released to the public to generate synthetic data. 
}
\label{fig:alg-frame}
\end{figure}

Within the DP-auto-GAN, there are two main components: the \emph{autoencoder} and the GAN.  The autoencoder serves to reduce the dimension of the data to $d\ll n$ before it is fed into the GAN. The GAN consists of a \emph{generator} that takes in noise $z$ sampled from a distribution $Z$ and produces $G_w(z)\in\R^d$, and a \emph{discriminator} $D_y(\cdot):\R^n\rightarrow \{0,1\}$.   Because of the autoencoder, the generator only needs to synthesize data based on the latent distribution in $\R^d$, which is much easier than synthesizing in $\R^n$. Both components of our architecture, as well as our algorithm's overall privacy guarantee, are described in the remainder of this section.



\begin{algorithm}
    \caption{\textsc{DPautoGAN} (full procedure)}
    \label{alg:all}
    \begin{algorithmic}[1] 
\State \textbf{architecture input:} Sensitive dataset \(D\in\XX^m\) where \(\XX\) is the (raw) data universe, preprocessed data dimension \(n\), latent space dimension \(d\), preprocessing function \(Pre:\XX\rightarrow \R^n\), post-processing function \(Post:\R^n\rightarrow \XX\), encoder architecture \(En_\phi:\R^n\rightarrow\R^d\) parameterized by \(\phi\), decoder architecture \(De_\theta:\R^d\rightarrow\R^n\) parameterized by \(\theta\), generator's noise distribution \(Z\) on sample space \(\Omega(Z)\), generator architecture \(G_w:\Omega( Z)\rightarrow \R^d\) parameterized by \(w\), discriminator architecture \(D_y:\R^n\rightarrow\{0,1\}\) parameterized by $y$. 
\State \textbf{autoencoder training parameters}: Learning rate \(\eta_1\), number of iteration rounds (or optimization steps) \(T_1\), loss function \(L_\text{auto}\), optimization method \textsc{Optim}\(_\text{auto}\) batch sampling rate \(q_1\) (for batch expectation size \(b_1=q_1m\)), clipping norm \(C_1\), noise multiplier \(\psi_1\), microbatch size \(r_1\)
\State \textbf{generator training parameters}: Learning rate \(\eta_2\), batch size \(b_2\), loss function \(L_G\), optimization method \textsc{Optim}\(_G\), number of generator iteration rounds (or optimization steps) \(T_2\)
\State \textbf{discriminator training parameters}: Learning rate \(\eta_3\), number of discriminator iterations per generator step \(t_D\), loss function \(L_D\), optimization method \textsc{Optim}\(_D\), batch sampling rate \(q_3\) (for batch expectation size \(b_3=q_3m\)), clipping norm \(C_3\), noise multiplier \(\psi_3\), microbatch size \(r_3\)
\State \textbf{privacy parameter} \(\delta>0\)
\State  {\bf procedure}{ DPautoGAN}
            \State \hspace{1em}\(X
            \leftarrow Pre(D)\)
            \State \hspace{1em}Initialize \(\phi,\theta,w,y\) for \(En_\phi,De_\theta,G_w,D_y\)\\
\hspace{1em}\Comment {\textit{Phase 1: autoencoder training}}
\For {\(t=1\ldots T_1\)}
            \State \textsc{DPTrain}$_{\textsc{auto}}$(\(X\),  \(En\),  \(De\), autoencoder training parameters)            
            \EndFor
\LeftComment{\textit{Phase 2: GAN training}}            
                \For{\(t=1\ldots T_2\)}
                \For{\(j=1\ldots t_D\)} 
                \LeftComment{\textit{(privately) train $D_y$ for \(t_D\) iterations}}
                \State \textsc{DPTrain}$_\textsc{Discriminator}$( $X, Z, G,De, D$, discriminator training parameters) 
                \EndFor
            \State \textsc{Train}$_\textsc{Generator}$(\(Z,G,De,D\), generator training parameters)  
            \EndFor
        \Comment{\textit{Privacy accounting }}
        \State $\text{RDP}_\text{auto}(\cdot) \leftarrow$ \textsc{RDP-Account}(\(T_1,q_1,\psi_1,r_1\))         
        \State $\text{RDP}_D(\cdot) \leftarrow$ \textsc{RDP-Account}(\(T_2\cdot t_D,q_3,\psi_3,r_3\))
        \State \(\eps\leftarrow\)\textsc{Get-Eps}($\text{RDP}_\text{auto}(\cdot)+\text{RDP}_D(\cdot)$)
        \State {\bf return} model \((G_w,De_\theta)\), privacy \((\eps,\delta)\)
    \end{algorithmic}
\end{algorithm}


%
%

\subsection{Autoencoder Framework and Training}

An autoencoder consists of an encoder $En_\phi(\cdot):\R^n\rightarrow \R^d$ and a decoder $De_\theta(\cdot):\R^d\rightarrow \R^n$ parametrized by  weights \(\phi,\theta\) respectively. The architecture of the autoencoder assumes that high-dimensional data \(x_i\in\R^n\) can be represented compactly in a low-dimensional \emph{latent space} \(\R^d\). The encoder \(En_\phi\) is trained to find such low-dimensional representations, and the decoder \(De_\theta\) maps\ \(En_\phi(x_i)\)  in the latent space back to \(x_i\). A natural measure of the information preserved in this process is the error between the decoder's image and the original \(x_i\). A good autoencoder should minimize the distance $\text{dist}(De_\theta(En_\phi(x_i)),x_i)$
for each point \(x_i\) for an appropriate distance function dist. Our autoencoder uses binary cross entropy loss: 
dist\((x,y)=-\sum_{j=1}^n y_{(j)}\log(x_{(j)})-\sum_{j=1}^n (1-y_{(j)})\log(1-x_{(j)})\), where \(x_{(j)}\) is the \(j\)th coordinate of  the data $x \in [0,1]^n$ after our preprocessing. 


This motivates a definition of a (true) loss function
\(\E_{x\sim Z_X}[\text{dist}(De_\theta(En_\phi(x)),x)]\) when data are drawn independently from an underlying distribution \(Z_X\). The corresponding empirical loss function when we have an access to sample \(\{x_i\}_{i=1}^m\) is
\begin{equation} \label{eq:auto-loss}
L_\text{auto}(\phi,\theta):=\textstyle\sum_{i=1}^m \text{dist}(De_\theta(En_\phi(x_i)),x_i).
\end{equation}
Finding a good autoencoder requires optimizing $\phi$ and $\theta$ to yield small empirical loss in \Eqref{eq:auto-loss}.

We minimize \Eqref{eq:auto-loss} privately using DP-SGD (Section \ref{sec:dp-sgd}). Our approach follows previous work on private training of autoencoders \citep{chen2018differentially,acs2018differentially,abay2018privacy} by adding noise to both the encoder and decoder.
The full description of our autoencoder training is given in Algorithm \ref{alg:auto} in Appendix \ref{app.algodetail}.
In our DP-auto-GAN framework, the autoencoder is trained first until completion, and is then fixed while training the GAN.

\subsection{GAN Framework and Training} 

A GAN consists of a generator \(G_w\) and a discriminator \(D_y:\R^n\rightarrow\{0,1\}\), parameterized respectively by  weights $w$ and $y$. The aim of the generator \(G_w\) is to synthesize (fake)\ data similar to the real dataset, while the discriminator aims to determine whether an input \(x_i\) is from the generator's synthesized data (and assign label \(D_y(x_i)=0\)) or is real data (and assign label \(D_y(x_i)=1\)).
The generator is seeded with a random noise \(z\sim Z\) that contains no information about the real dataset, such as a multivariate Gaussian vector, and aims to generate a distribution \(G_w(z)\) that is hard for \(D_y\) to distinguish from the real data. Hence, the generator wants to minimize the probability that \(D_y\) makes a correct guess,
\(\E_{z\sim Z} [1-D_y(G_w(z))] \). The discriminator wants to maximize its probability of a correct guess, which is \(\E_{z\sim Z} [1-D_y(G_w(z))]\) when the datum is fake and \(\E_{x\sim Z_X} [D_y(x)]\) when it is real.

We extend the binary output of \(D_y\) to a continuous range \([0,1]\), with the value indicating the confidence that a sample is real.  We use the zero-sum objective for the discriminator and generator  \citep{arjovsky2017wasserstein}, which is motivated by the Wasserstein distance of two distributions.  Although the proposed Wasserstein objective cannot be computed exactly, it can be approximated by optimizing:
\begin{equation}
\textstyle\min_y \textstyle\max_w O(y,w):=\Ebb_{x \sim Z_X} [D_y(x)] -\Ebb_{z \sim Z} [D_y(G_w(z))]. \label{eq:WGAN}
\end{equation}
We optimize \Eqref{eq:WGAN} privately using the DP-SGD framework described in Section \ref{sec:dp-sgd}. We differ from prior work on DP-GANs in that our generator \(G_w(\cdot)\) outputs data \(G_w(z)\) in the latent space \(\R^d\), which needs to be decoded by the fixed (pre-trained) \(De_\theta\) to \(De_\theta(G_w(z))\) before being fed into the discriminator \(D_y(z)\). The gradient \(\nabla_w G_w\) is obtained by backpropagation through this additional component \(De_\theta(\cdot)\). The full description of our GAN training is given in Algorithm \ref{alg:gan} in Appendix \ref{app.algodetail}.

After this two-phase training (of the autoencoder and GAN), the noise distribution \(Z\), trained generator \(G_w(\cdot)\), and trained decoder \(De_\theta(\cdot)\) are released to the public.  The public can sample \(z\sim Z\) to obtain a synthesized datapoint \(De_\theta(G_y(z))\) repeatedly to obtain a synthetic dataset of any desired size.

\subsection{Privacy Accounting}

We use Renyi Differential Privacy (RDP) of \cite{mironov2017renyi}, to account for privacy in each phase of training as in prior works. Our autoencoder and GAN are trained privately by clipping gradients and adding noise to the encoder, decoder, and discriminator.
Since the generator only accesses data through the discriminator's (privatized) output and \(De_\theta\) is first trained privately and then fixed during GAN training,  the trained parameters of generator are also private by post-processing guarantees of differential privacy. 
Privacy accounting is therefore required for only two parts that access real data \(X\): training of the autoencoder and of the discriminator. In each training procedure, we apply the RDP accountant described in Section \ref{sec:rdp}, to analyze privacy of the DP-SGD training algorithm. 

The RDP accountant is a function \(r:[1,\infty)\rightarrow\R_+\) and guarantees \((\eps,\delta)\)-DP for any given \(\delta>0\) with \(\eps=\min_{\alpha>1} r(\alpha)+\frac{\log{1/\delta}}{\alpha-1}\) (\cite{mironov2017renyi}; also used in Tensorflow Privacy \citep{TensorflowPrivacy}).  Hence, at the end of two-phase training, we have two RDP\ accountants \(r_1,r_2\). We compose two RDP accountants \textit{before} converting the combined accountant into \((\eps,\delta)\)-DP. Note that another method used in DP-SYN \citep{abay2018privacy} first converts \(r_i\) to \((\epsilon_i,\delta_i)\)-DP and then combines them  into \((\eps_1+\eps_2,\delta_1+\delta_2)\)-DP by basic composition \citep{DMNS06}. For completeness,  we show that composing RDP accountants first always results in a better privacy analysis.

\begin{restatable}{lemma}{rdpbetter}\label{lem:rdp-better}
Let \(\M_1,\M_2\) be any mechanisms and $r_1,r_2:[1,\infty)\rightarrow \R_+\cup\{\infty\}$ be functions such that \(\M_1,\M_2\) are \((\alpha,r_1(\alpha))\)- and \((\alpha,r_2(\alpha))\)-RDP, respectively. Let \(\delta\in(0,1]\) and let 
\[
\eps_1=\min_{\alpha>1} r_1(\alpha)+\frac{\log(2/\delta)}{\alpha-1},\quad \eps_2=\min_{\alpha>1} r_2(\alpha)+\frac{\log(2/\delta)}{\alpha-1}, \quad
\text{ and } \quad \eps =\min_{\alpha>1} r_1(\alpha)+r_2(\alpha)+\frac{\log(1/\delta)}{\alpha-1}.
\]
Then $\M_1$ is $(\epsilon_1,\delta/2)$-DP, $\M_2$ is $(\epsilon_2,\delta/2)$-DP, and the composition $\M = (\M_1,\M_2)$ is $(\epsilon,\delta)$-DP. If \(\epsilon_1\) and \(\epsilon_2\) are finite, then \(\epsilon<\epsilon_1+\epsilon_2\). 
\end{restatable}

In practice, we observe that composing at the RDP level first in Lemma \ref{lem:rdp-better} reduces privacy cost by \(\approx30\%\). An empirical application of Lemma \ref{lem:rdp-better} to our parameter setting is illustrated in Figure \ref{fig:rdp-eps}. The figure, theoretical support of this approximate privacy reduction factor, and the proof of Lemma \ref{lem:rdp-better} are in Appendix \ref{app.rdp}. 

\section{Experiments}\label{s.exp}

In this section, we empirically evaluate the performance of our \dpgan framework on the \mimic  \citep{johnson2016mimic} and ADULT \citep{uci_ml} datasets, which have been used in prior works on differentially private synthetic data generation. We compare against these prior approaches using a variety of qualitative and quantitative evaluation metrics, including some from prior work and some novel metrics we introduce. We target $\delta = 10^{-5}$ in all settings. For more details on the evaluation metrics used, see Appedix \ref{app.expmetrics}.  All experimental details and additional experimental results can be found in Appendices \ref{app.mimictrain} and \ref{app.adulttrain}, and our code is available at \url{https://github.com/DPautoGAN/DPautoGAN}.

\subsection{Binary Data}\label{s.binary}

\mimic \citep{johnson2016mimic} is a binary dataset consisting of medical records of 46K intensive care unit (ICU) patients over 11 years old with 1071 features.
Even though  \dpgan  can handle mixed-type data, we  evaluate it first on  \mimic since this dataset has been used in similar non-private \citep{choi2017generating} and private \citep{xie2018differentially} GAN frameworks.  We apply the same evaluation metrics used in these papers, namely  dimension-wise probability and dimension-wise prediction. Prediction is defined by AUROC score of a logistic regression classifier. 

\begin{figure}[h]
        \centering
        \begin{subfigure}{0.245\textwidth}
                \includegraphics[width = 0.95\linewidth]{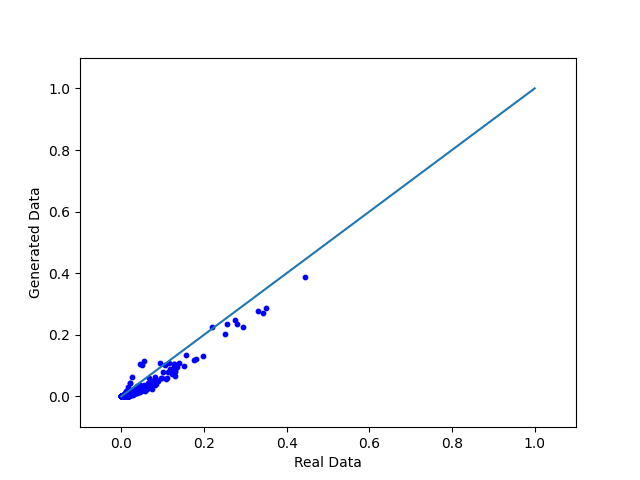}
                \caption{$\eps = \infty$}
        \end{subfigure}
        \begin{subfigure}{0.245\textwidth}
        \includegraphics[width = 0.95\linewidth]{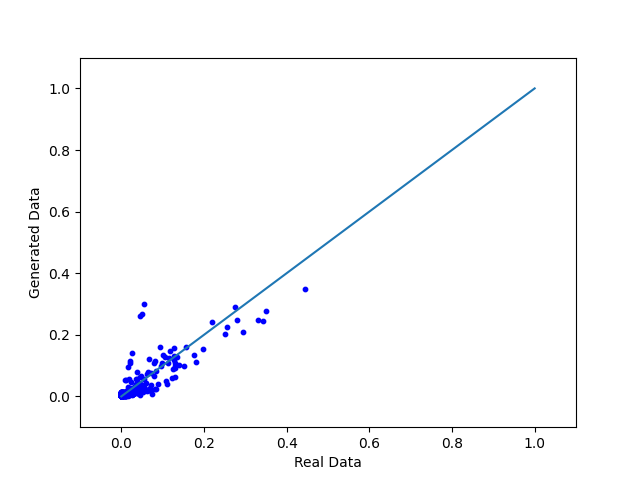}
        \caption{$\eps = 2.70$}
\end{subfigure}
        \begin{subfigure}{0.245\textwidth}
        \includegraphics[width = 0.95\linewidth]{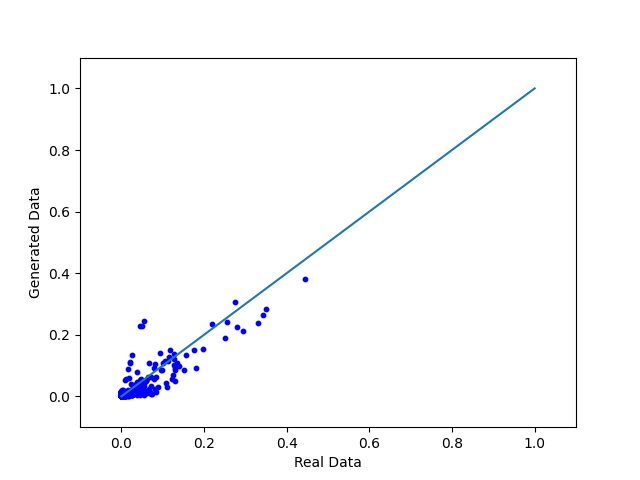}
        \caption{$\eps = 1.33$}
\end{subfigure}
        \begin{subfigure}{0.245\textwidth}
        \includegraphics[width = 0.95\linewidth]{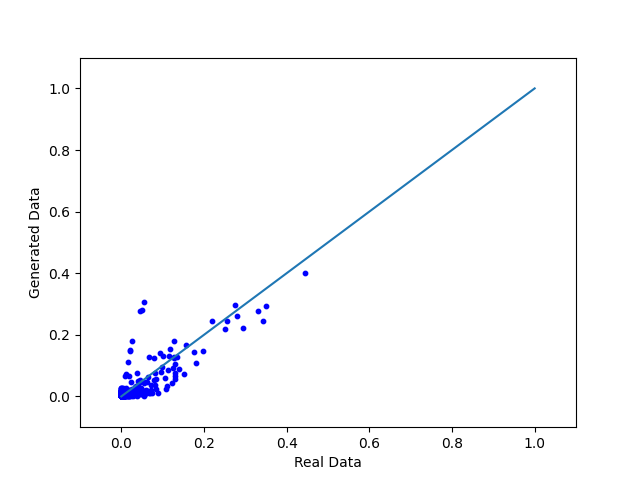}
        \caption{$\eps = 0.81$}
\end{subfigure}
\caption{Dimension-wise probability scatterplots for different values of $\eps$. Each point represents one of the 1071 features in the \mimic dataset. The $x$ and $y$ coordinates of each point are the proportion of 1s in real and synthetic datasets of a feature, respectively. The line $y=x$, which represents ideal performance, is shown in each plot. Note that even for small $\epsilon$ values, performance is not degraded much relative to the non-private method. Compare with Figure 4 in \cite{xie2018differentially}, which provides worse performance for $\epsilon \in [96, 231]$.}
\label{fig:dim_prob}
\end{figure}


 \paragraph{Dimension-Wise Probability.} Figure \ref{fig:dim_prob} shows the dimension-wise probability of \dpgan for different \(\epsilon\). Each point in the figure corresponds to a feature in the dataset, and the $x$ and $y$ coordinates respectively show the proportion of 1s in the real and synthetic datasets.  Points closer to the $y=x$ line correspond to better performance, because this indicates the distribution is similar in the real and synthetic datasets.  As shown in Figure \ref{fig:dim_prob}, the proportion of 1's in the marginal distribution for is similar on the real and synthetic datasets in the non-private ($\eps = \infty$) and private settings.  The marginal distributions of the privately generated data from DP-auto-GAN remain a close approximation of the real dataset, even for small values of $\epsilon$, because nearly all points fall close to the line $y = x$. We note that our results are significantly stronger than the ones obtained in \cite{xie2018differentially} with $\eps \in [96.5,231]$ because we obtain dramatically better performance with $\eps$ values that are two orders of magnitude smaller. For visual performance comparison, see Figure 4 of \cite{xie2018differentially}.


\begin{figure}[h]
        \centering
        \begin{subfigure}{0.245\textwidth}
                \includegraphics[width = 0.95\linewidth]{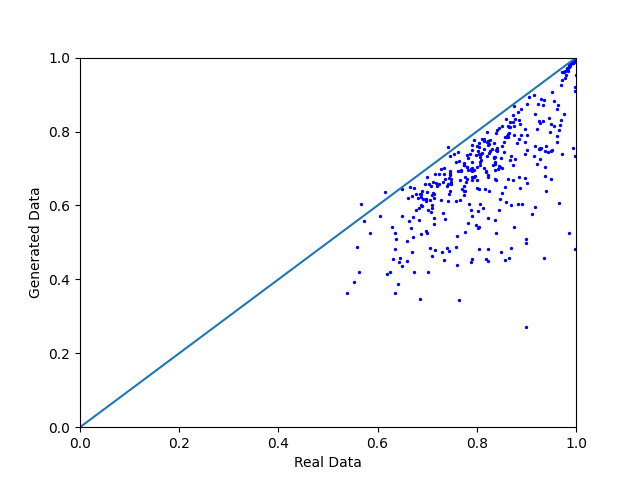}
                \caption{$\eps = \infty$}
        \end{subfigure}
        \begin{subfigure}{0.245\textwidth}
        \includegraphics[width = 0.95\linewidth]{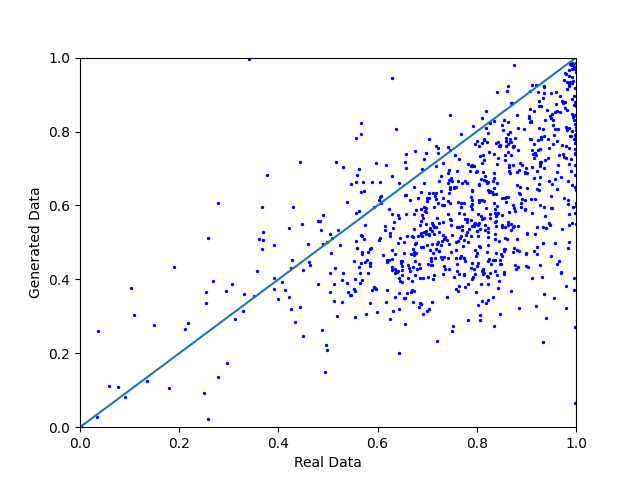}
        \caption{$\eps = 2.70$}
\end{subfigure}
        \begin{subfigure}{0.245\textwidth}
        \includegraphics[width = 0.95\linewidth]{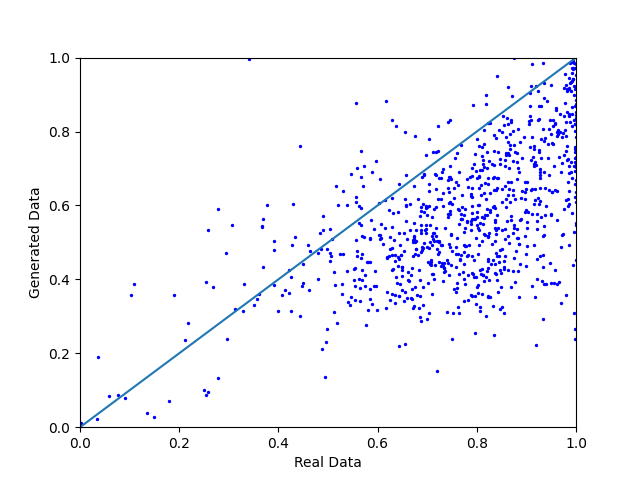}
        \caption{$\eps = 1.33$}
\end{subfigure}
        \begin{subfigure}{0.245\textwidth}
        \includegraphics[width = 0.95\linewidth]{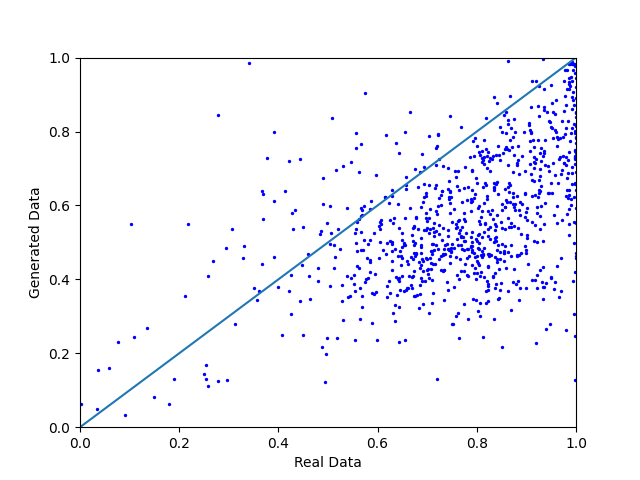}
        \caption{$\eps = 0.81$}
\end{subfigure}
\caption{Dimension-wise prediction scatterplots for different values of $\eps$. Each point represents one of the 1071 features in the \mimic dataset. The $x$ and $y$ coordinates of each point represent the AUROC score of a logistic regression classifier trained on real and synthetic datasets, respectively. The line $y = x$ corresponds to the ideal performance.  Again we note that even for small $\epsilon$ values, performance is not degraded much relative to the non-private method. Compare with Figure 5 in \cite{xie2018differentially}, which provides worse performance for $\epsilon \in [96, 231]$.}
\label{fig:dim_pred_mimic}
\end{figure}

\paragraph{Dimension-Wise Prediction.} Figure \ref{fig:dim_pred_mimic} shows dimension-wise prediction using DP-auto-GAN for different values of $\eps$. Each point in the figure corresponds to a feature in the dataset, and the $x$ and $y$ coordinates respectively show the AUROC score of a logistic regression classifier trained on the real and synthetic datasets, and points closer to the $y=x$ line still correspond to better performance. As shown in the figure, for $\eps = \infty$, many points are concentrated along the lower side of line $y = x$, which indicates that the AUROC score of the real dataset is only marginally higher than that of the synthetic dataset. When privacy is added, there is a gradual shift downwards relative to the line $y =x$, with larger variance in the plotted points, indicating that AUROC scores of real and synthetic data show more difference when privacy is introduced. Surprisingly, there is little degradation in performance for smaller $\eps$ values, including $\eps = 0.81$. For sparse features with few 1's in the data, the generative model will output all 0's for that feature, making AUROC ill-defined. (See Appendix \ref{app.dwpsparse} for more details.) We follow \citet{xie2018differentially} by excluding those features from dimension-wise  prediction plots.

Our results for DP-auto-GAN under this metric are also significantly stronger than the ones obtained in \cite{xie2018differentially} with much larger $\eps$ values of $\eps \in [96.5,231]$; for visual performance comparison, see Figure 5 of \cite{xie2018differentially}.  Our probability and prediction plots of  \dpgan are either comparable to or better than  \cite{xie2018differentially}, with our  prediction plots detecting many more sparse features. The performance of \dpgan degrades only slightly as \(\eps\) decreases and is achieved at much smaller \(\eps\) values, giving a roughly 100x improvement in privacy compared to \cite{xie2018differentially}.

\subsection{Mixed-Type Data}\label{s.mixed}

ADULT dataset \citep{uci_ml} is an extract of the U.S. Census of 48K working adults, consisting of mixed-type data: nine categorical features (one of which is a binary label) and four continuous.  This dataset has been used to evaluate DP-WGAN \citep{frigerio2019differentially} and DP-SYN \citep{abay2018privacy}.  We compare \dpgan against these methods, as well as DP-VAE \citep{acs2018differentially}. We target \(\eps=1.01,0.51,0.36\). For DP-SYN, we allow \(\eps=1.4,0.8,0.5\) because their implementation uses standard privacy composition, which is looser than than RDP composition (Lemma \ref{lem:rdp-better}). These larger $\eps$ values provide comparable privacy guarantees to the smaller $\eps$ values achieved by RDP composition, and allow for a fair comparison of architectures without modifying the implementation in \cite{abay2018privacy}. For more details, see Appendices \ref{app.rdp} and \ref{app.dpsyn}.

\begin{figure}[h!]
        \centering
        \begin{subfigure}{0.31\textwidth}
                \includegraphics[width = 0.95\linewidth]{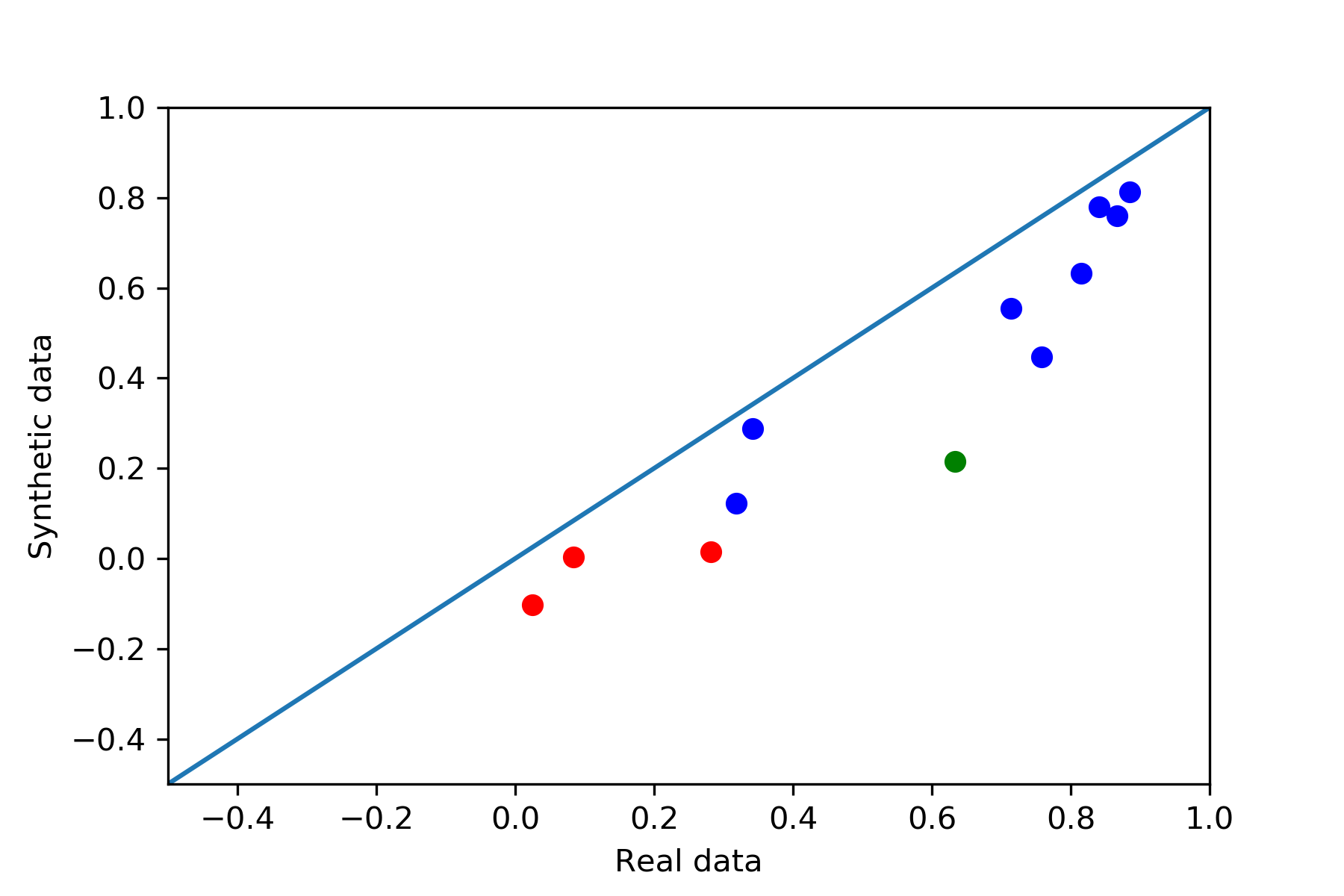}
                \caption{DPautoGAN $\eps = 1.01$}
        \end{subfigure}
        \begin{subfigure}{0.31\textwidth}
                \includegraphics[width = 0.95\linewidth]{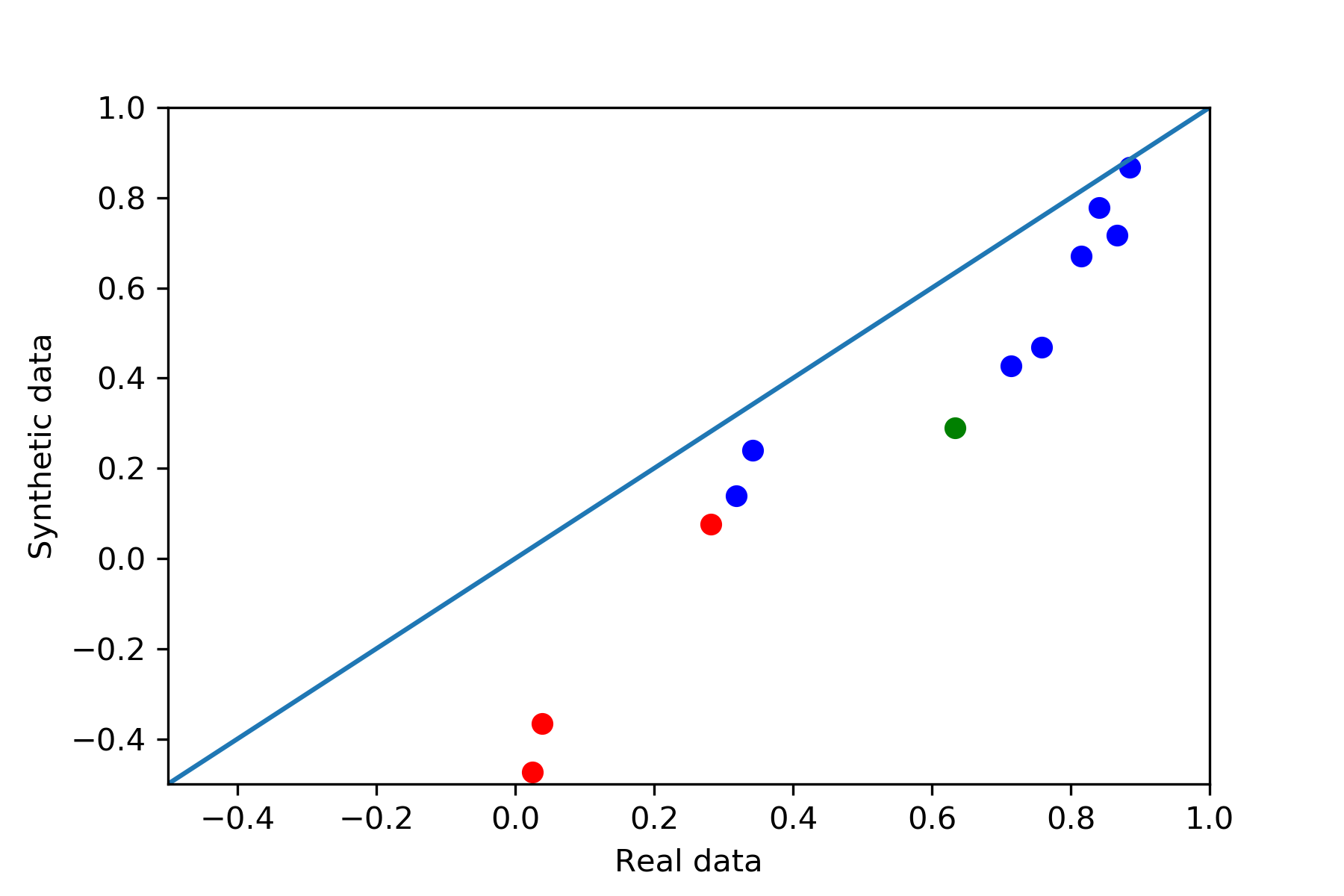}
                \caption{DPautoGAN $\eps = 0.51$}
        \end{subfigure}
        \begin{subfigure}{0.31\textwidth}
                \includegraphics[width = 0.95\linewidth]{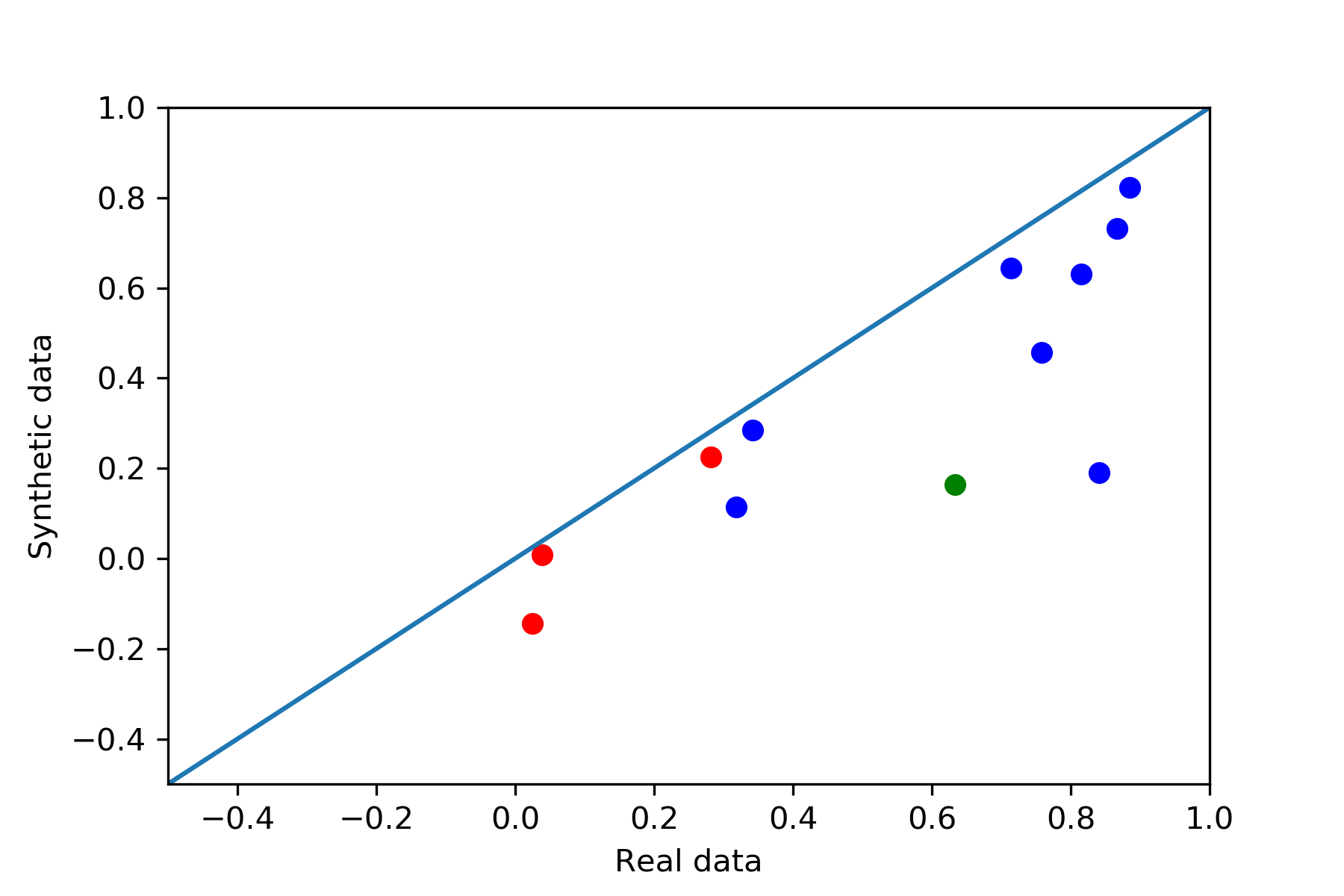}
                \caption{DPautoGAN $\eps = 0.36$}
        \end{subfigure}
        
                \begin{subfigure}{0.31\textwidth}
                \includegraphics[width = 0.95\linewidth]{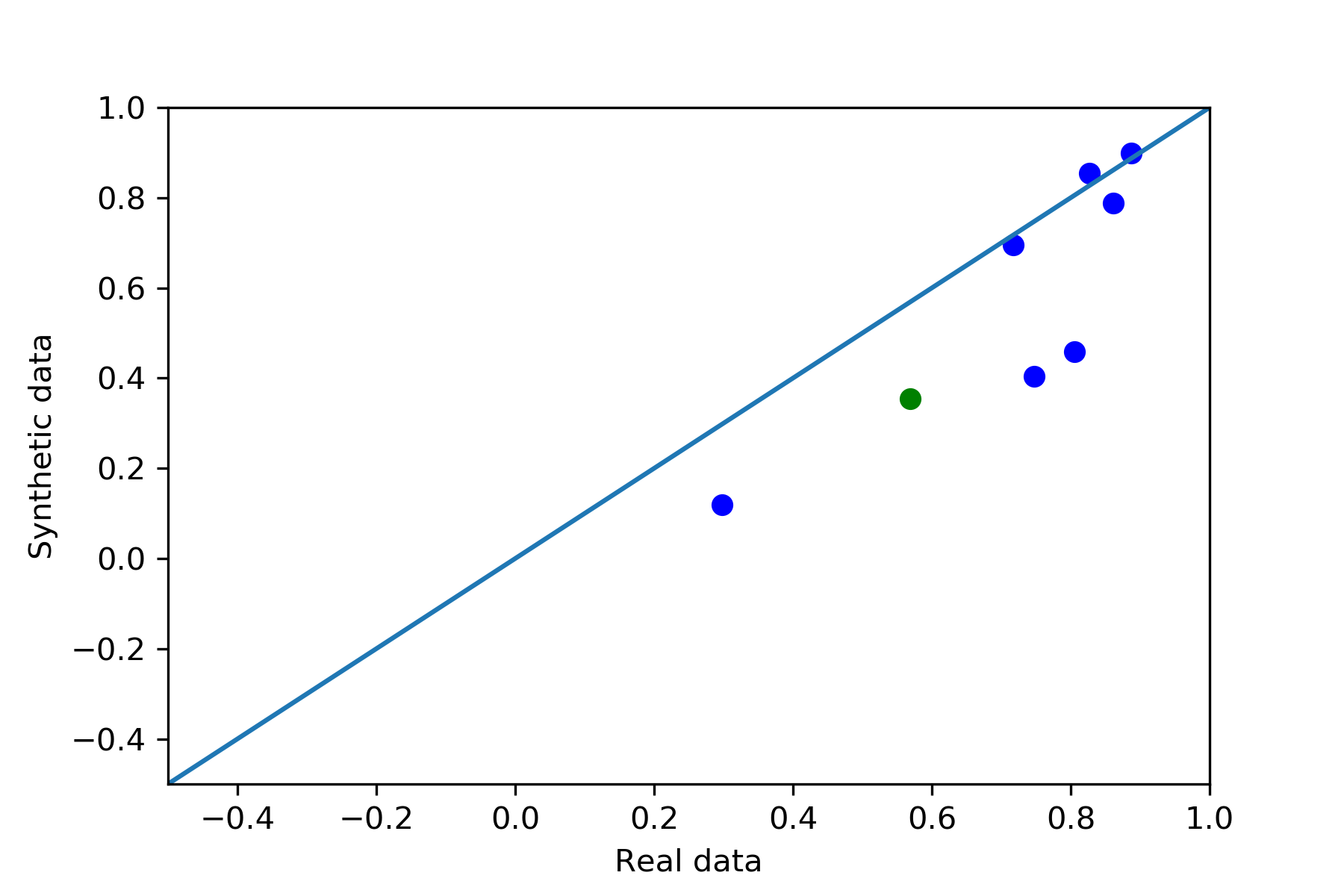}
                \caption{DP-WGAN $\eps = 1.01$}
        \end{subfigure}
        \begin{subfigure}{0.31\textwidth}
                \includegraphics[width = 0.95\linewidth]{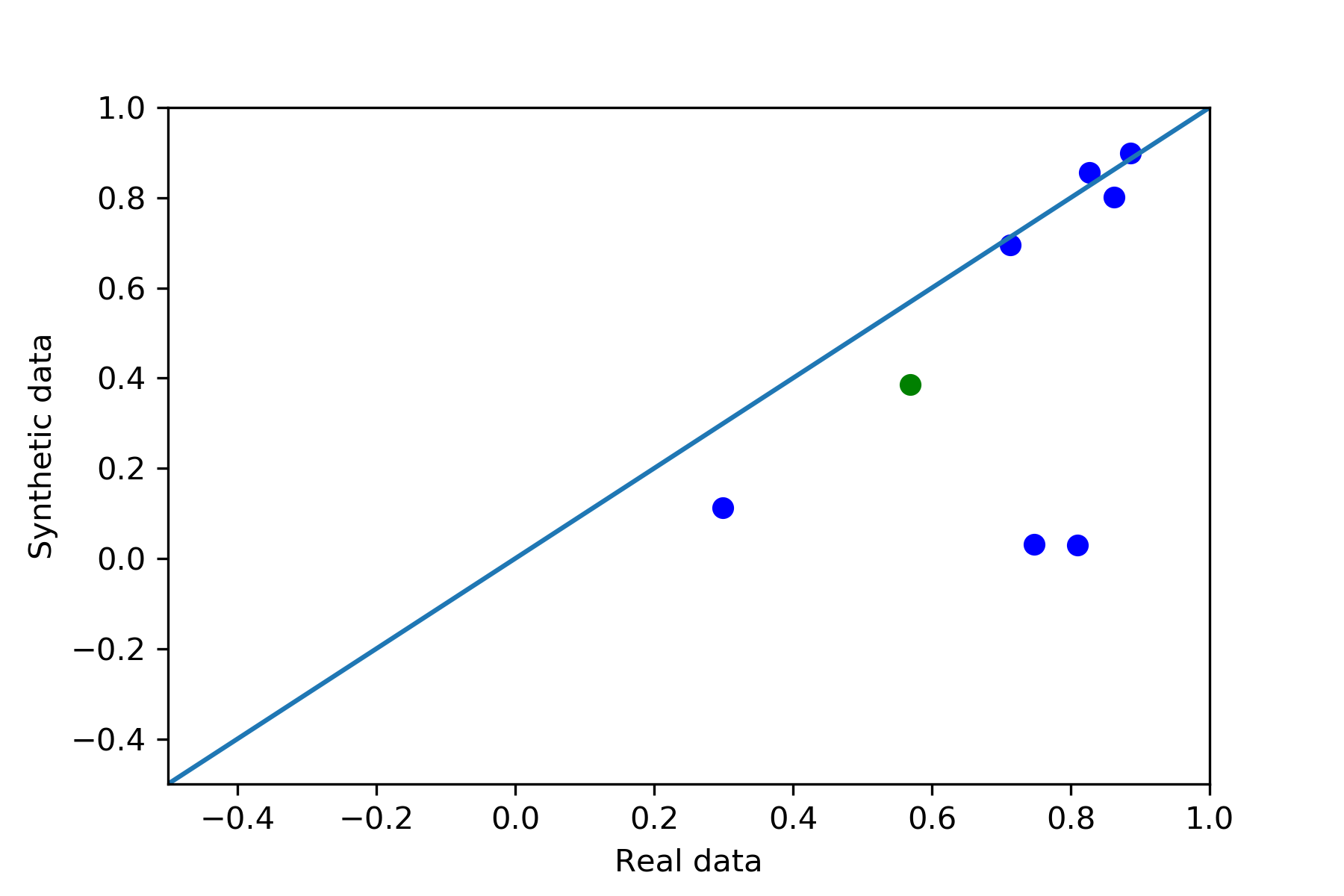}
                \caption{DP-WGAN $\eps = 0.51$}
        \end{subfigure}
        \begin{subfigure}{0.31\textwidth}
                \includegraphics[width = 0.95\linewidth]{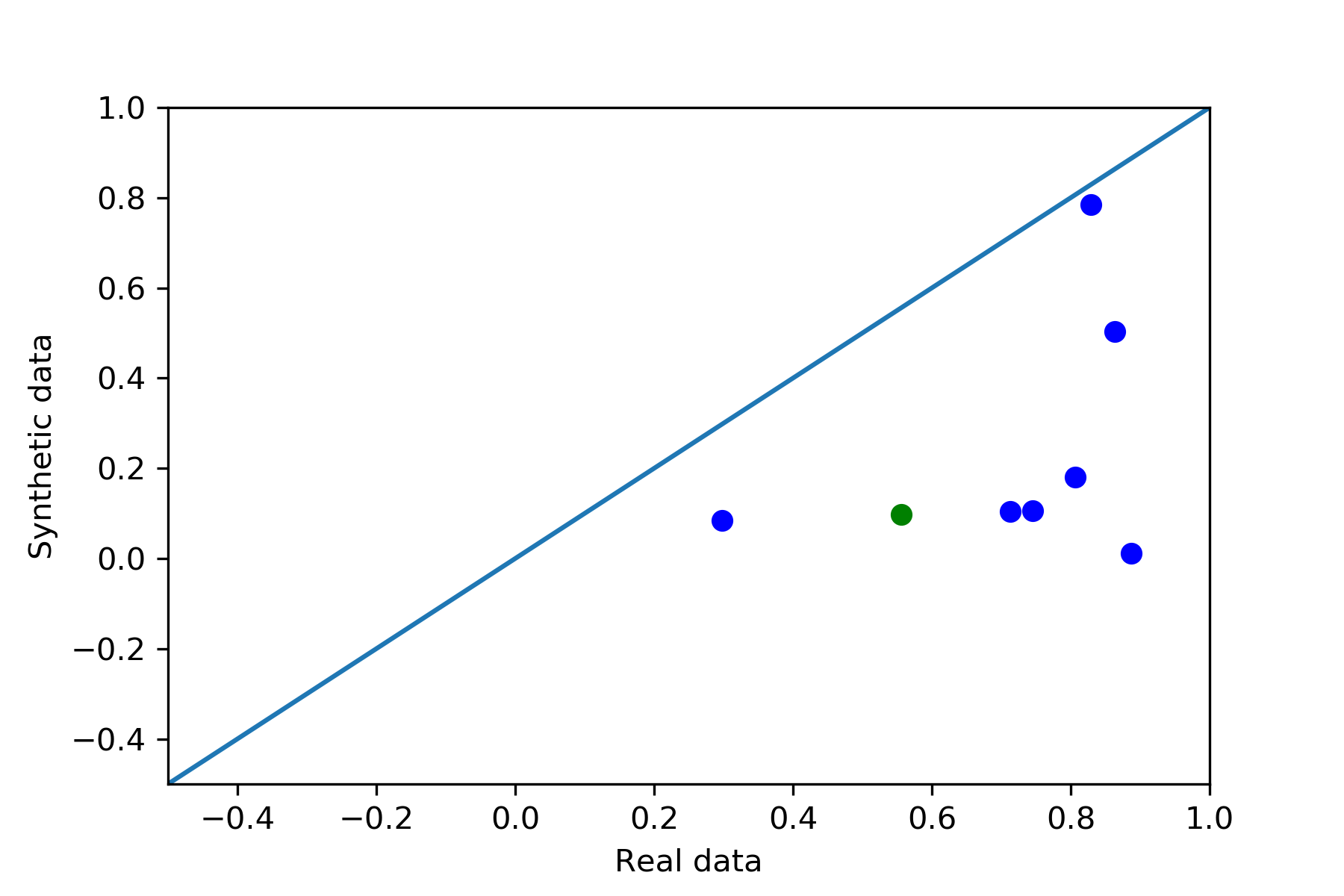}
                \caption{DP-WGAN $\eps = 0.36$}
        \end{subfigure}
        
         \begin{subfigure}{0.31\textwidth}
                \includegraphics[width = 0.95\linewidth]{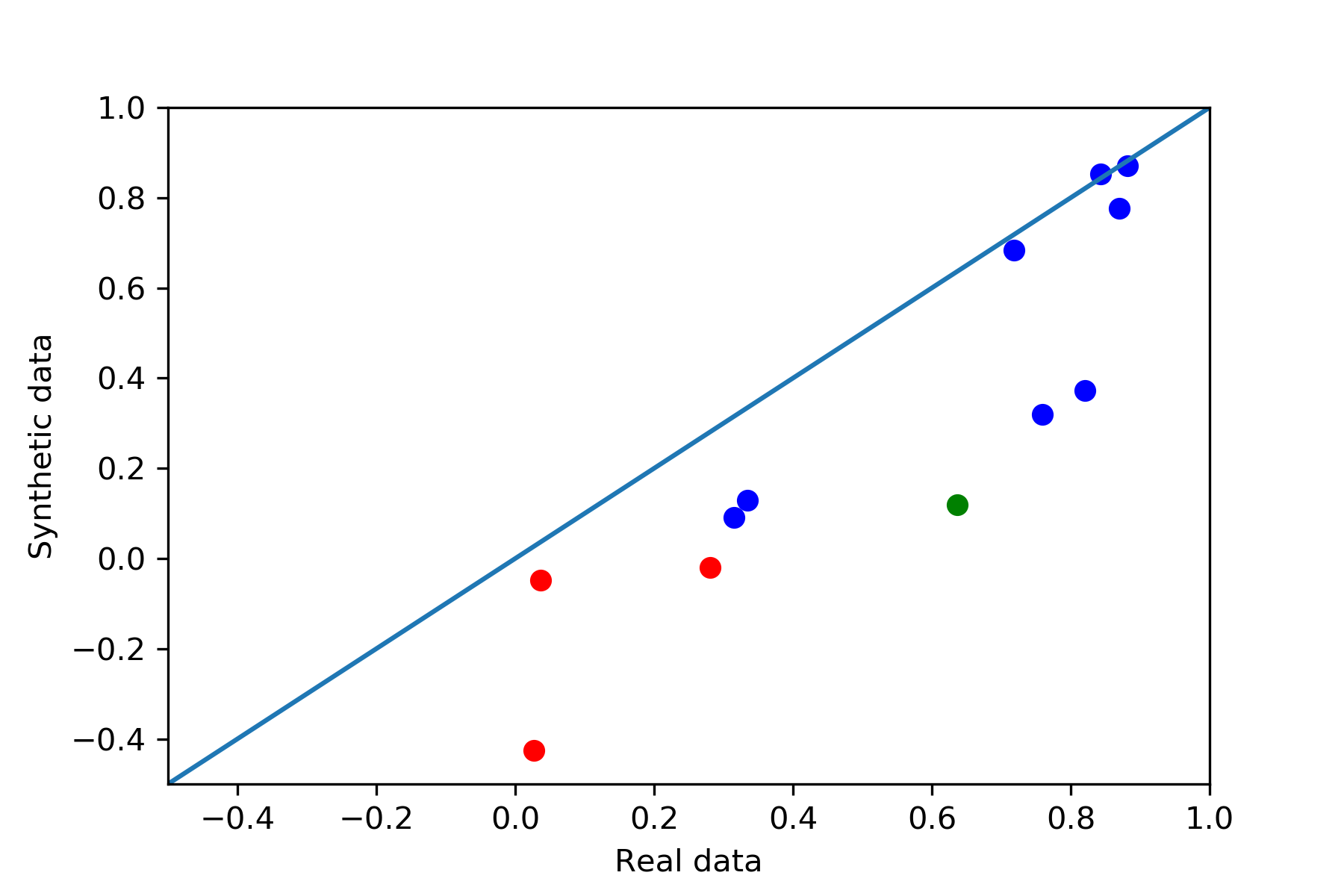}
                \caption{DP-VAE $\eps = 1.01$}
        \end{subfigure}
        \begin{subfigure}{0.31\textwidth}
                \includegraphics[width = 0.95\linewidth]{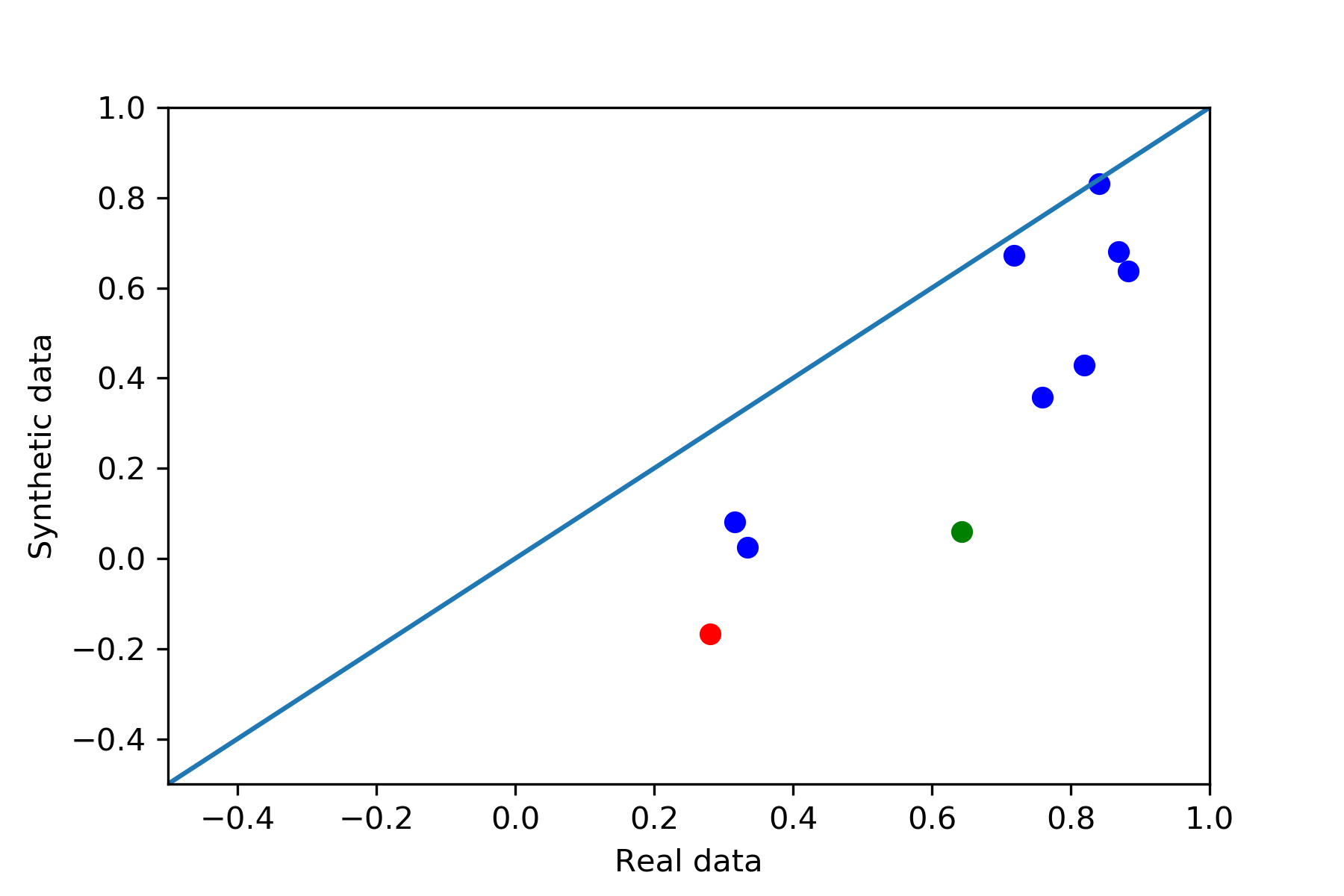}
                \caption{DP-VAE $\eps = 0.51$}
        \end{subfigure}
        \begin{subfigure}{0.31\textwidth}
                \includegraphics[width = 0.95\linewidth]{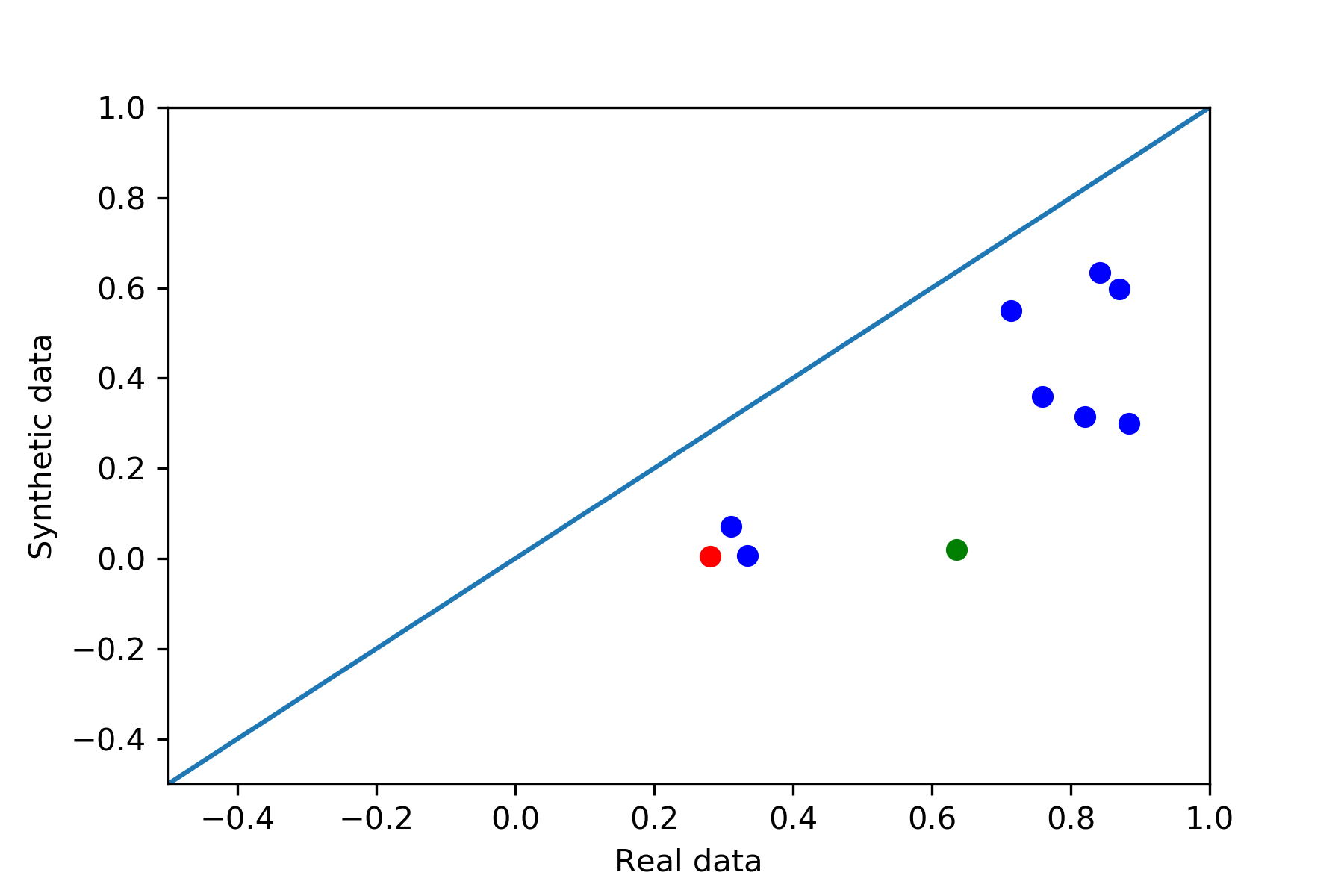}
                \caption{DP-VAE $\eps = 0.36$}
        \end{subfigure}
        
         \begin{subfigure}{0.31\textwidth}
                \includegraphics[width = 0.95\linewidth]{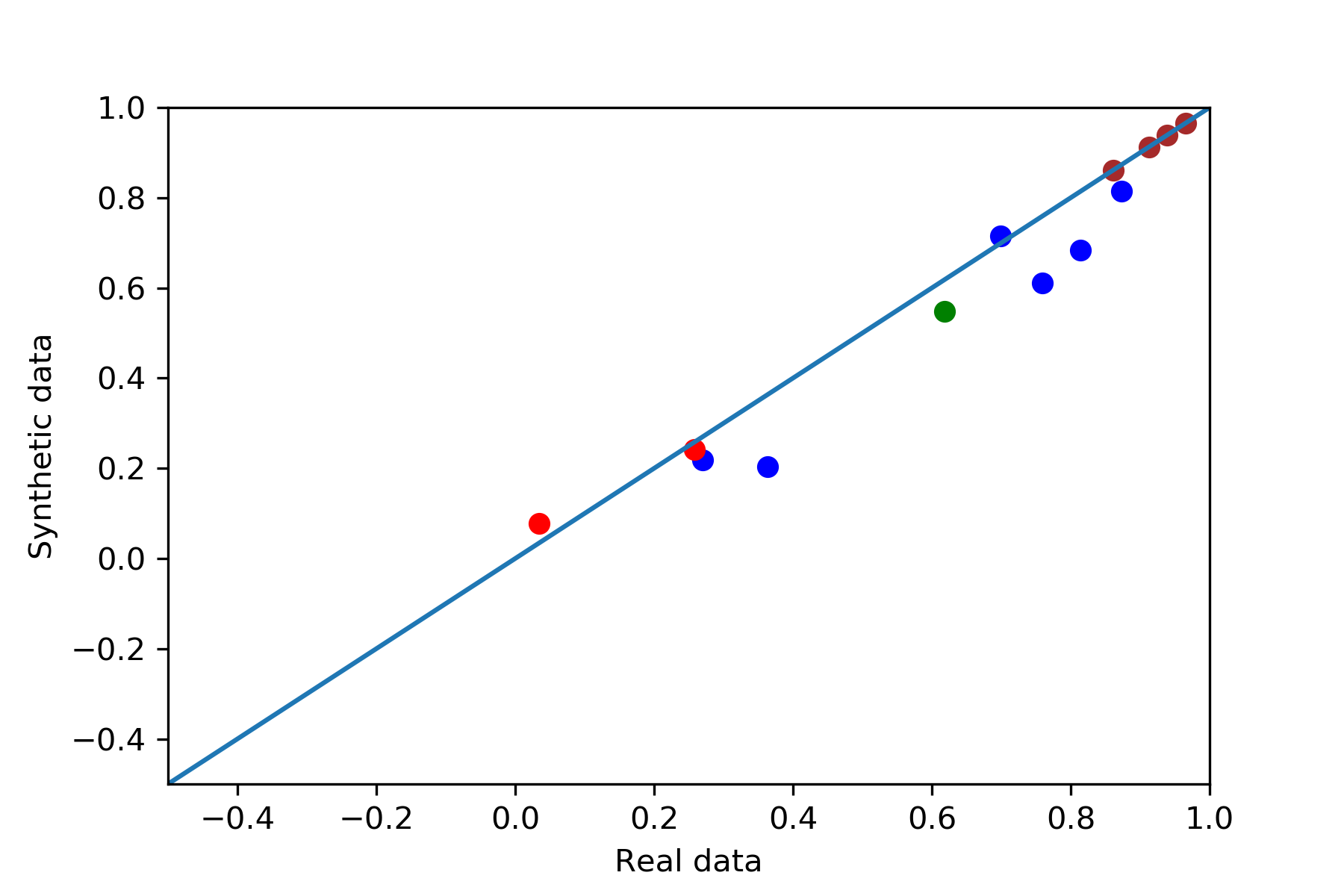}
                \caption{DP-SYN $\eps = 1.40$}
        \end{subfigure}
        \begin{subfigure}{0.31\textwidth}
                \includegraphics[width = 0.95\linewidth]{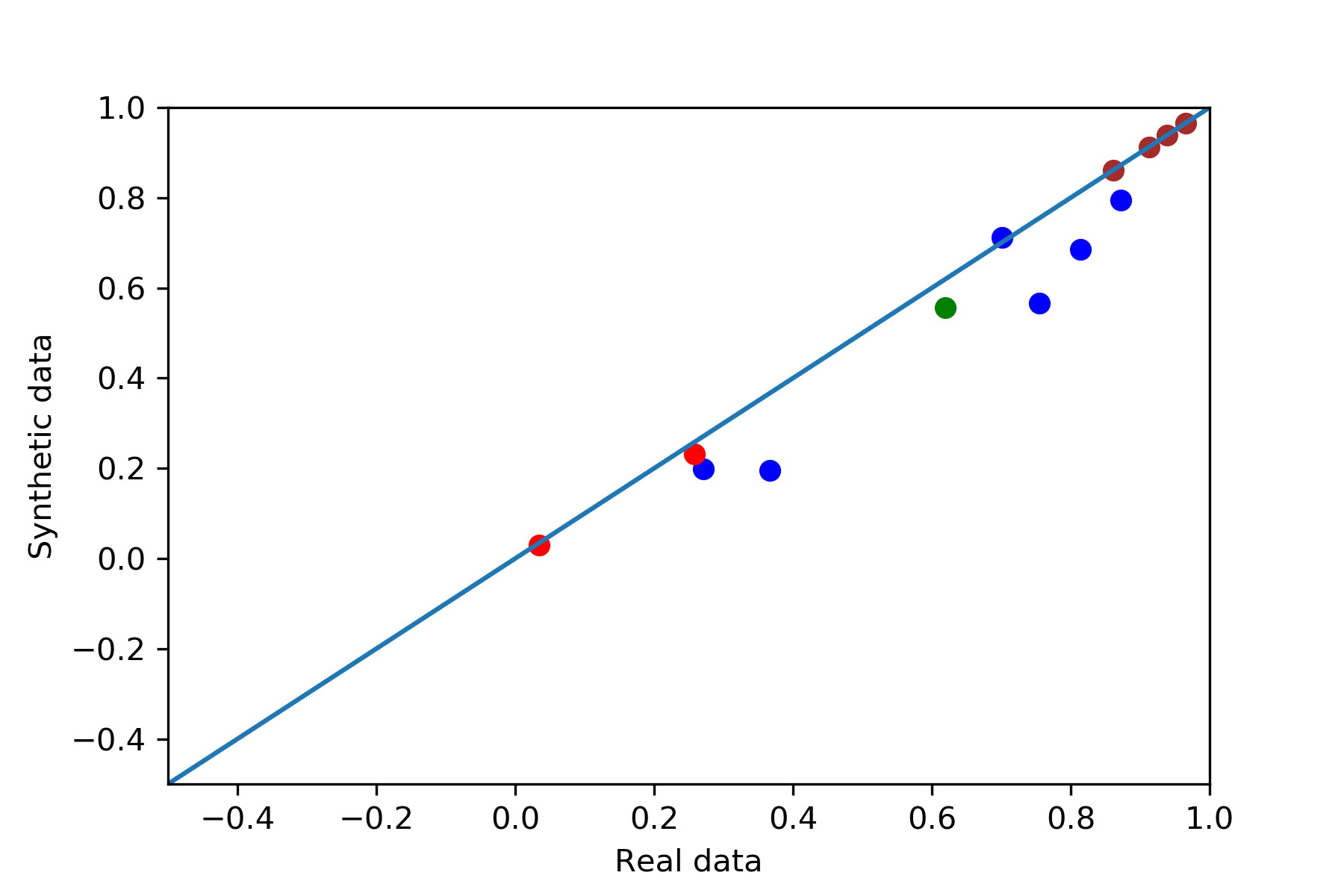}
                \caption{DP-SYN $\eps = 0.80$}
        \end{subfigure}
        \begin{subfigure}{0.31\textwidth}
                \includegraphics[width = 0.95\linewidth]{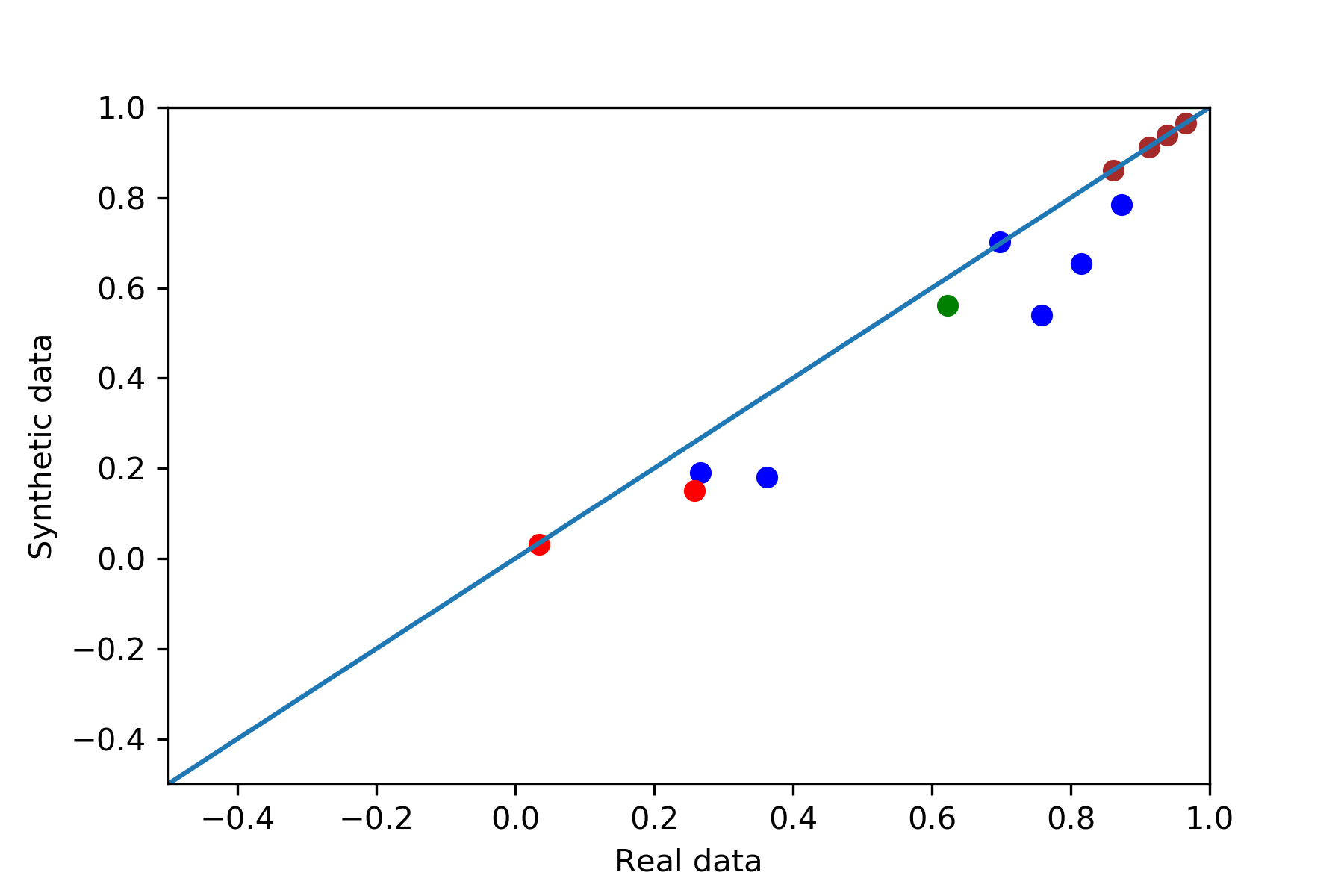}
                \caption{DP-SYN $\eps = 0.50$}
        \end{subfigure}
        \caption{Dimension-wise prediction scatterplot of all (applicable) features of ADULT dataset for different $\eps$ values and algorithms. The line $y=x$ represents ideal performance. Blue, green, and red points respectively correspond to unlabeled categorical, labeled binary, and continuous features. Brown points indicate the synthetic data exhibit no diversity (i.e., all data points have the same category). Note that DP-SYN has several features without diversity. Red points with \(R^2\) scores close to zero in the original data have unstable (and  unmeaningful) synthetic \(R^2\) scores due to the sparse nature of those features in the original data, and some of these \(R^2\) scores fall outside of the plotted range. The implementation of DP-WGAN in \cite{frigerio2019differentially} did not allow continuous features, and the implementation of DP-SYN in \cite{abay2018privacy} converted two continuous features to categorical; see Appendix \ref{app.adulttrain} for more details.}
        \label{fig:dim_pred_adult}
\end{figure}

\paragraph{Dimension-Wise Prediction.} Figure \ref{fig:dim_pred_adult} compares the performance of \dpgan with these three prior algorithms for the task of dimension-wise prediction. For categorical features (represented by blue points and a single green point), we use a random forest classifier for prediction as in \cite{frigerio2019differentially}, and we measure performance using $F_1$ score, which is more appropriate than AUROC for multi-class prediction. For continuous features (represented by red points), we used Lasso regression  and report \(R^2\) scores.  The green point corresponds to the \emph{salary feature} of the data, which is real-valued but treated as binary based on the condition $> \$50k$, which was similarly used as a binary label in \cite{frigerio2019differentially}.  We use brown points to indicate the categorical features for which the synthetic data exhibit no diversity, where all synthetic data points have the same category. We explore metrics for measuring diversity later in this section.

Note that in Figure \ref{fig:dim_pred_adult}, there are not four red points in each plot (corresponding to the four continuous features of the dataset). While AUROC for the binary features is always supported on $[0,1]$, the $R^2$ score for real-valued features can be negative if the predictive model is poor, and these values for these missing points fell outside the range of Figure \ref{fig:dim_pred_adult}. These features are explored later in Figure \ref{fig:1-way-hist-adult}, using 1-way marginals as a qualitative metric.
 
Each point in Figure \ref{fig:dim_pred_adult} corresponds to one feature, and the $x$ and $y$ coordinates respectively show the accuracy score on the real data and the synthetic data. Figure \ref{fig:dim_pred_adult} shows that \dpgan achieves considerable performance for all $\eps$ values tested. As expected, its performance degrades as \(\eps\) decreases, but not substantially. DP-WGAN \citep{frigerio2019differentially} performs well at \(\epsilon=1.01\), but its performance degrades rapidly with smaller $\eps$. This is consistent with \cite{frigerio2019differentially}, which uses higher \(\eps=3,7\). \dpgan outperforms DP-VAE \citep{acs2018differentially} across all \(\eps\) values. DP-SYN \citep{abay2018privacy} is able to capture relationships between features well even for small \(\eps\) using this metric.  

\paragraph{1-Way Marginal and Diversity Divergence.} While DP-SYN has good dimension-wise prediction, this does not capture \textit{diversity}, a  concern of bias known for DP-SGD (\cite{bagdasaryan2019differential}). For features with a large majority class and many minority classes, the classifier often predicts the majority class with probability one. We found that for four features, DP-SYN generates data from only one class, whereas all other algorithms do not behave this way for any feature. Lack of diversity in synthetic data can raise fairness concerns, as societal decisions based on the private synthetic data will inevitably ignore minority groups.

We start by turning to 1-way marginal as a method of evaluation, which is able to detect such issues and give another perspective of synthetic data. 
Figure \ref{fig:adult_hist} shows histograms of synthetic data from the four algorithms on two categorical features: marital-status and race. Marital-status distributes more evenly across categories, and DP-VAE,  DP-SYN and \dpgan are able to learn this distribution well. Race, on the other hand, has an 85.5\% majority; DP-SYN only generated data from the majority class, whereas \dpgan and DP-VAE were able to detect the existence of minority classes. DP-WGAN suffered similar issues on the marital status feature. 

Figure \ref{fig:hist-country} shows similar histograms for the native-country feature of the ADULT dataset, where the majority class constitutes $>$85\% of the population. Each of the 41 minority classes in native-country constitute less than 2\% in the original dataset, with most of them weighing less than 0.2\% of the population. \dpgan and DP-WGAN are able to capture some minority classes. DP-VAE was unable to accurately learn the majority structure and significantly overestimates the weight on all minority classes, which greatly impacts the estimate of the majority class. DP-SYN did not capture an existence of any minority classes.

\begin{figure}[h]
        \centering
                 \begin{subfigure}{0.49\textwidth}
                \includegraphics[width = \linewidth]{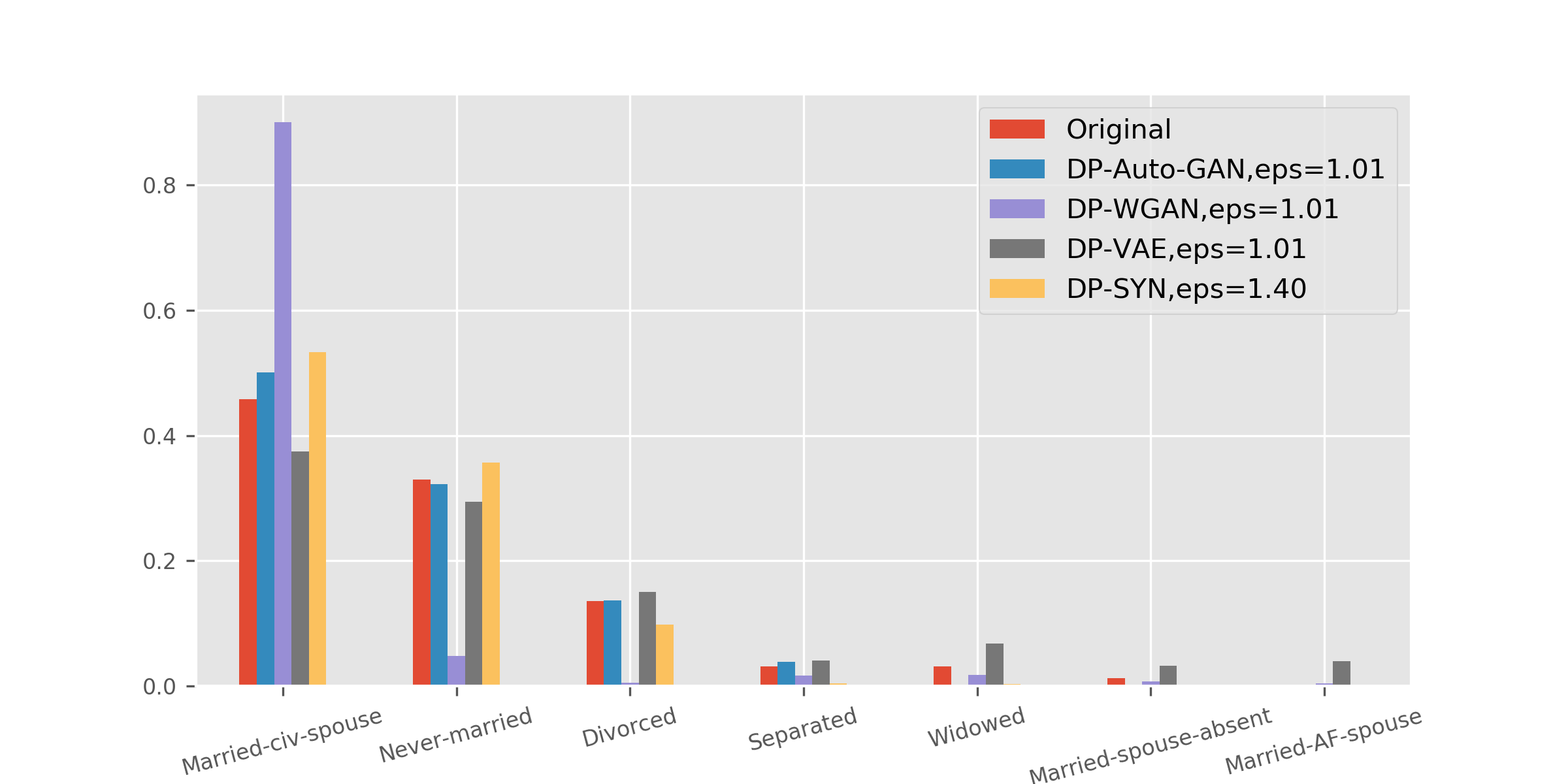}
                \caption{Marital-status}
        \end{subfigure}
        \begin{subfigure}{0.49\textwidth}
                \includegraphics[width = \linewidth]{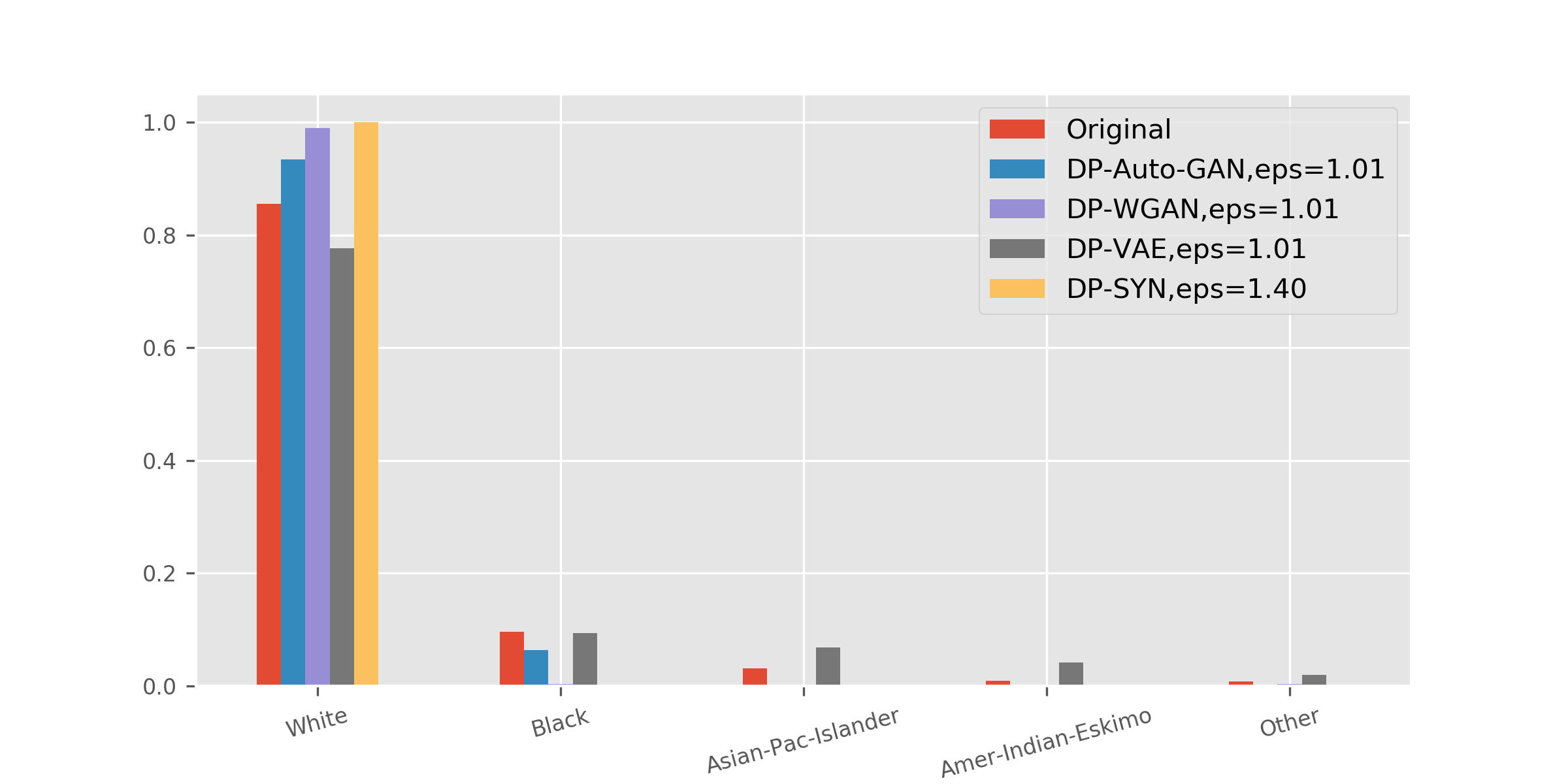}
                \caption{Race}
        \end{subfigure}
                \caption{Histograms of synthetic data generated by different algorithms.    }
                \label{fig:adult_hist}
\end{figure}

\begin{figure}[h]
        \centering
        \begin{subfigure}{0.09\textwidth}
        \includegraphics[width = \linewidth]{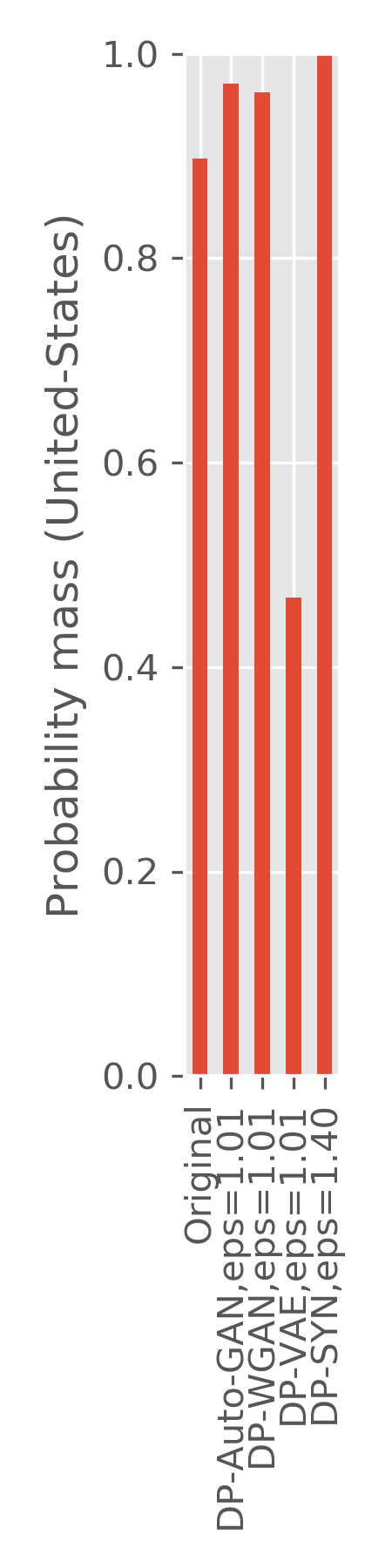}
        \caption{US}
        \end{subfigure} 
        \begin{subfigure}{0.9\textwidth}
        \includegraphics[width = \linewidth]{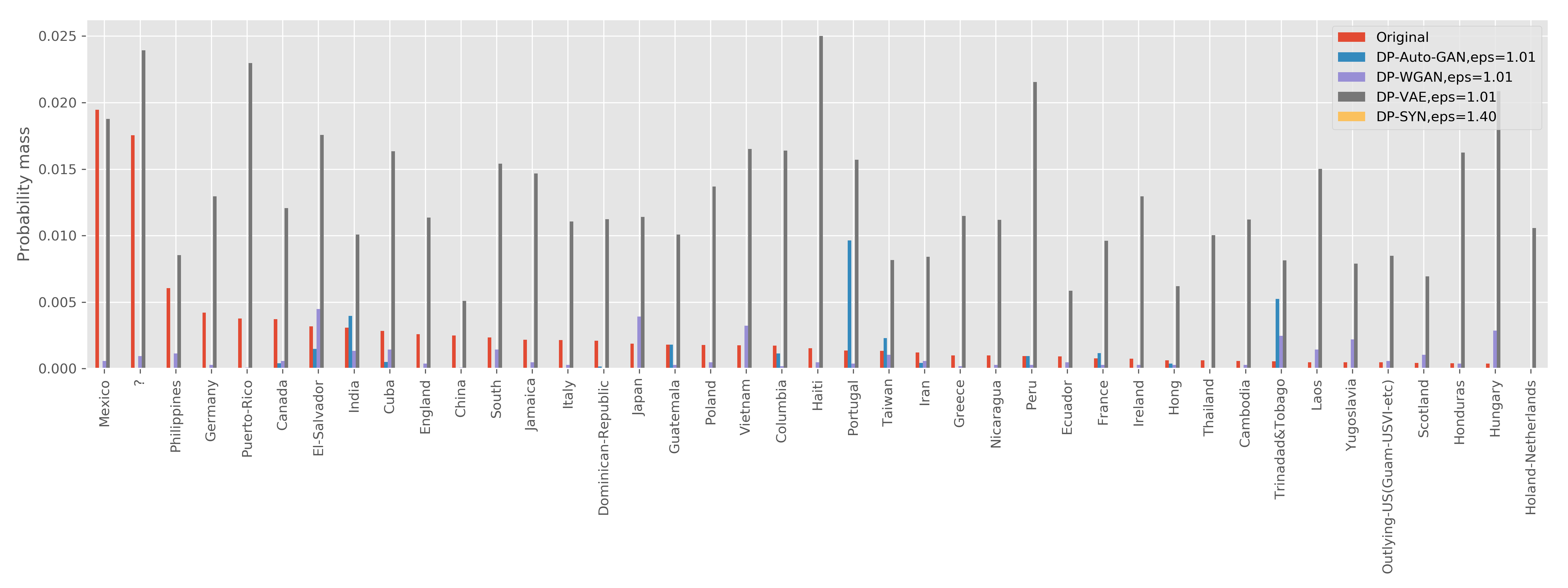}
        \caption{Minority classes}
        \end{subfigure}         
        \caption{Histogram of native-country features of the original data and synthetic data generated by different algorithms.}
\label{fig:hist-country}     
\end{figure}

\begin{table*}[tbh]
        \caption{Diversity measures \(\textrm{JSD}\) and \(D_{KL}^\mu\) on different features of ADULT data and the sum of divergences across all eight applicable  categorical features (All). Recall that \(p_1\) is the maximum probability across all categories of that feature in the original data. Smaller values for the diversity measures imply more diverse synthetic data. For each row (feature), the smallest value for each setting of \(\eps\) is highlighted in bold.}
        \label{tab:diversity-JSD}
        \centering
        \begin{tabular}{c|ccc|ccc|ccc|ccc}\hline
 & \multicolumn{3}{c}{\textbf{DP-auto-GAN}} & \multicolumn{3}{c}{\textbf{DP-WGAN}} & \multicolumn{3}{c}{\textbf{DP-VAE}} & \multicolumn{3}{c}{\textbf{DP-SYN}} \\
\(\eps\) \emph{values} & \emph{0.36} & \emph{0.51} & \emph{1.01}  & \emph{0.36} & \emph{0.51} & \emph{1.01} & \emph{0.36} & \emph{0.51} & \emph{1.01} & \emph{0.50} & \emph{0.80} & \emph{1.40}\\\hline \multicolumn{13}{c}{JSD Diversity Measure} \\
\hline \hline
Marital & .025 & .043 & \textbf{.014} & .119 & .624 & .136 & .139 & .043 & .021 & \textbf{.017} & \textbf{.013} & .017 \\ \hline
Race & \textbf{.021} & \textbf{.014} & .016 & .081 & .053 & .040 & .095 & .031 & \textbf{.011} & .053 & .053 & .053 \\ \hline
All & {0.33} & \textbf{0.23} & \textbf{0.19} & 1.29 & 2.41 & 0.73 & 0.80 & 0.44 & 0.23 & \textbf{0.25} & 0.27 & 0.28 \\ \hline
\multicolumn{13}{c}{\(D_{KL}^\mu\) Diversity Measure, with \(\mu=e^{-\frac{1}{1-p_1}}\)} \\
\hline \hline
Marital & {.019} & .053 & \textbf{.005} & .165 & 1.16 & .290 & .207 & .044 & .017 & \textbf{.017} & \textbf{.011} & .012 \\\hline
Race & \textbf{.125} & \textbf{.064} & .089 & .262 & .465 & .277 & .315 & .102 & \textbf{.038} & .465 &.465 & .465 \\\hline
All & \textbf{0.81} & \textbf{0.48} & \textbf{0.53} & 5.26 & 6.39 & 1.53 & 2.52 & 1.17 & 0.58 & 0.99 & 1.00 & 1.02 \\\hline
\end{tabular}  
\end{table*}

A standard measure for diversity between the original distribution \(P\) and synthetic distribution \(Q\) includes Kullback-Leibler (KL) divergence \(D_{KL}(P||Q)\). Under differential privacy the support of \(P\)  is a private information,  so the private synthetic data inherently cannot ensure its support to align with the original data. This makes \(D_{KL}(P||Q)\) and \(D_{KL}(Q||P)\) (and related metrics such as Inception score \citep{salimans2016improved}) undefined. One alternative is Jensen--Shannon divergence (JSD) \citep{jeffreys1946invariant,GAN14}: \(\mathrm{JSD}(P||Q):=\frac12D_{KL}(P||Q)+\frac12D_{KL}(Q||P)\) which is always defined and nonnegative. We use this metric to evaluate the diversity of the synthetic data. 

In addition, we propose another diversity measure, \textit{\(\mu\)-smoothed Kullback-Leibler (KL) divergence}  between the original distribution \(P\) and synthetic distribution \(Q\): 
$$D_{KL}^\mu(P||Q):=\textstyle\sum_{x \in supp(P)} (P(x)+\mu)\log(\frac{P(x)+\mu}{Q(x)+\mu}),$$ for small \(\mu>0\).  \(D_{KL}^\mu\) maintains the desirable property that \(D_{KL}^\mu\geq0\) and is zero if and only if \(P=Q\). Smaller \(\mu\) implies stronger penalties for missing minority categories in the synthetic data, and the penalty approaches \(\infty\) as \(\mu\rightarrow0\). This allows \(\mu\) as a knob to adjust the penalty necessary in private setting. 
In our settings, we are concerned with one category dominating in the original distribution \(P\) (e.g., as in Figure \ref{fig:adult_hist}), say \(P=(p_1,\ldots,p_k)\) with  high \(p_1=\max_ip_i\), and when the synthetic distribution \(Q=(q_1,\ldots,q_k)=(1,0,\ldots,0)\) supports only one single category. Then, we have 
$D_{KL}^\mu(P||Q)=\sum_{i=2}^k (p_i+\mu)\log(p_i+\mu)-(p_1+\mu)\log(\frac{p_1+\mu}{1+\mu})-(\sum_{i=2}^k (p_i+\mu))\log\mu$. 
For small \(\mu>0\), 
\(\log\mu\) dominates \(\log(p_i+\mu)\) and \(\log(\frac{p_1+\mu}{1+\mu})\), so the dominating term is \(\left(\sum_{i=2}^k (p_i+\mu)\right)\log\mu\approx(1-p_1)\log\mu\). Hence, we use \(\mu=e^{-\frac{1}{1-p_1}}\) so that this term is a constant, thus normalizing scores across features.

Table \ref{tab:diversity-JSD} reports the diversity divergences of all four algorithms for marital-status, race, and the sum across  eight  categorical features. One out of the nine categorical features are not used due to a difference in preprocessing of DP-SYN; see Appendix \ref{app.adulttrain} for details. Both measures are able to detect the lost of diversity in DP-SYN in race, and identify DP-auto-GAN as generating more diverse data than the prior methods for most features and $\eps$ values. 

We note that predictive scores may also not be appropriate for continuous features when no good classifier exists to predict the feature, even in the original dataset. In our setting, we found three continuous features with \(R^2\) scores close to zero even with more complex regression models, and with negative \(R^2\) scores on synthetic data, which is not meaningful. For those features, 1-way marginals (histograms, explored next) are preferred to prediction scores. 

In general, we suggest that an evaluation of synthetic data should be based on probability measures (distributions of data) and not predictive scores of models. Models may be a source of not only unpredictability and instability, but also of bias and unfairness.


\paragraph{Histograms for Continuous Features with Small \(R^2\) Scores.} For three continuous features in the ADULT dataset (capital gain, capital loss, and hours worked per week), we were not able to find a regression model with good fit (as measured by $R^2$ score) for the latter three features (capital gain, capital loss, and hours worked per week) in terms of the other features even on the real data. We attempted several different approaches, ranging from simple regression models such as lasso to complex models such as neural networks, and all had a low $R^2$ score on both the real and synthetic data. The capital gain and capital loss attributes are inherently hard to predict because the data are sparse (mostly zero) in these attributes.

Since the $R^2$ scores did not prove to be a good metric for these features, we instead plotted 1-way feature marginal histograms for each of these three remaining features to check whether the marginal distribution was learned correctly. These 1-way histograms for \dpgan are shown in Figure \ref{fig:1-way-hist-adult}.  The figure shows that \dpgan identifies the marginal distribution of capital gain and capital loss quite well, and it does reasonably well on the hours-per-week feature.

\begin{figure}[h]
        \centering
        \begin{subfigure}{0.24\textwidth}
        \includegraphics[width = \linewidth]{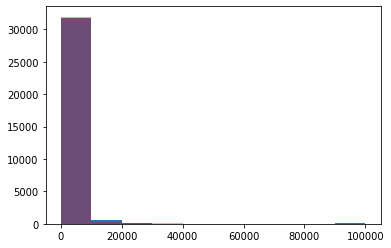}
        \caption{Capital gain, $\eps = \infty$}
        \end{subfigure}
        \begin{subfigure}{0.24\textwidth}
        \includegraphics[width = \linewidth]{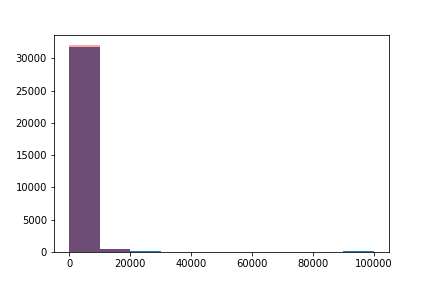}
        \caption{$\eps = 1.01$}
        \end{subfigure}
        \begin{subfigure}{0.24\textwidth}
        \includegraphics[width = \linewidth]{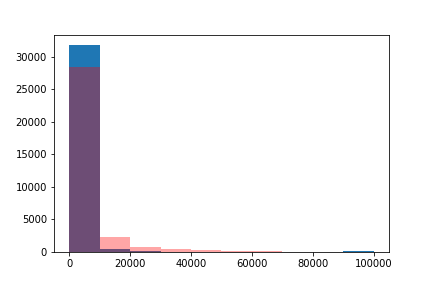}
        \caption{$\eps = 0.51$}
        \end{subfigure}
        \begin{subfigure}{0.24\textwidth}
        \includegraphics[width = \linewidth]{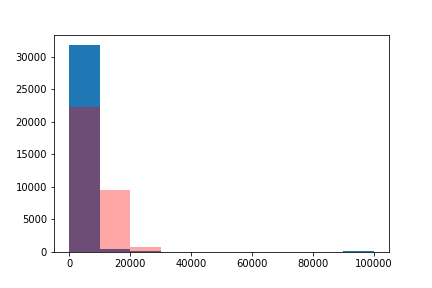}
        \caption{$\eps = 0.36$}
        \end{subfigure}
                \begin{subfigure}{0.24\textwidth}
        \includegraphics[width = \linewidth]{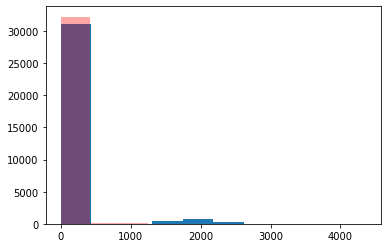}
        \caption{Capital loss, $\eps = \infty$}
        \end{subfigure}
        \begin{subfigure}{0.24\textwidth}
        \includegraphics[width = \linewidth]{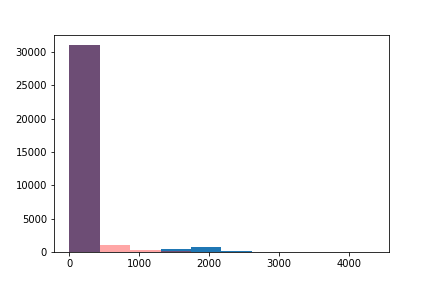}
        \caption{$\eps = 1.01$}
        \end{subfigure}
        \begin{subfigure}{0.24\textwidth}
        \includegraphics[width = \linewidth]{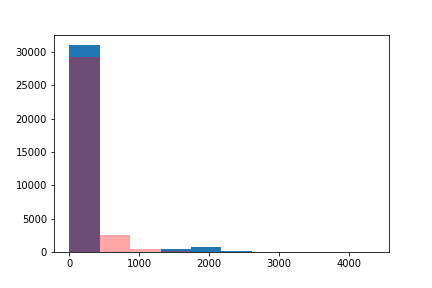}
        \caption{$\eps = 0.51$}
        \end{subfigure}
        \begin{subfigure}{0.24\textwidth}
        \includegraphics[width = \linewidth]{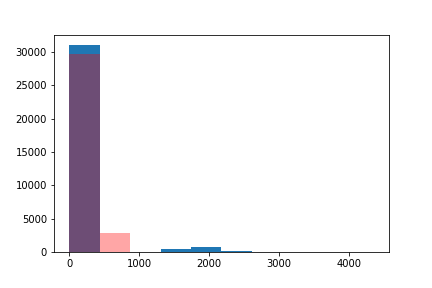}
        \caption{$\eps = 0.36$}
        \end{subfigure}
                \begin{subfigure}{0.24\textwidth}
        \includegraphics[width = \linewidth]{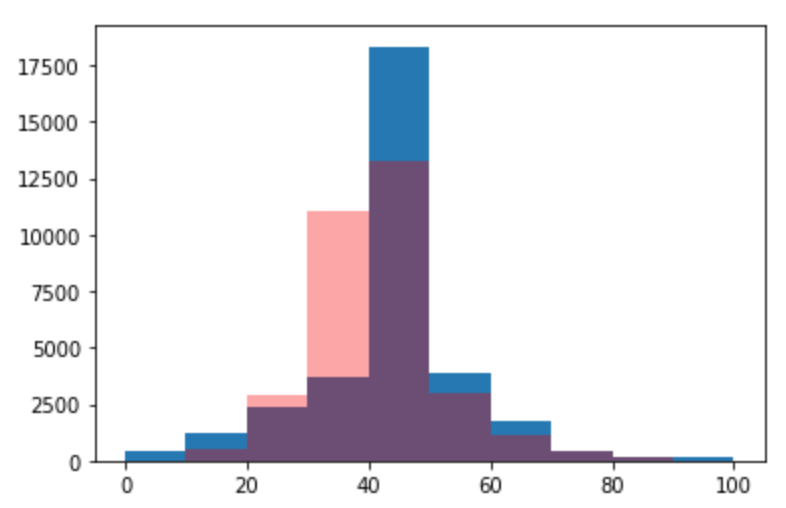}
        \caption{Hours/week, $\eps = \infty$}
        \end{subfigure}
        \begin{subfigure}{0.24\textwidth}
        \includegraphics[width = \linewidth]{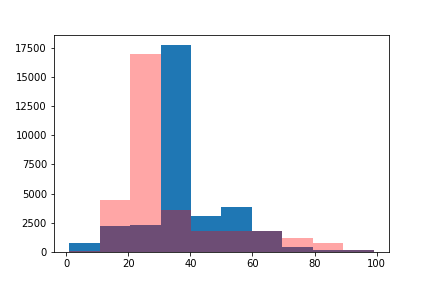}
        \caption{$\eps = 1.01$}
        \end{subfigure}
        \begin{subfigure}{0.24\textwidth}
        \includegraphics[width = \linewidth]{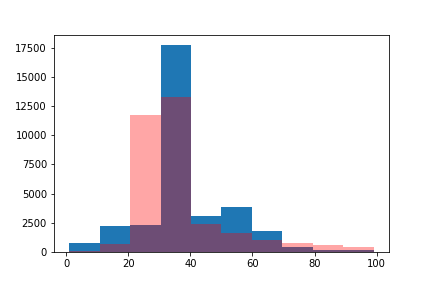}
        \caption{$\eps = 0.51$}
        \end{subfigure}
        \begin{subfigure}{0.24\textwidth}
        \includegraphics[width = \linewidth]{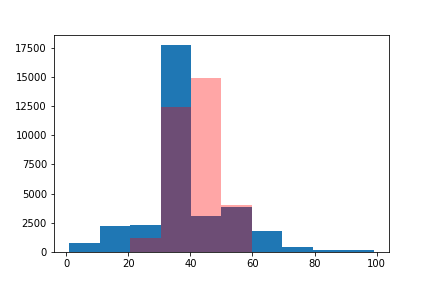}
        \caption{$\eps = 0.36$}
        \end{subfigure}
%
%
%
        \caption{1-way histogram for different values of $\eps$. Each pair of consecutive rows correspond to capital gain, capital loss and weekly work-hours, respectively. Blue corresponds to the histogram of the real dataset, and red corresponds to the histogram of the synthetic dataset generated by DP-auto-GAN with the indicated $\eps$.  The overlap of both histograms is purple.}
        \label{fig:1-way-hist-adult}    
\end{figure} 

\begin{table*}[tbh]
\centering
\caption{Accuracy scores of random forest prediction on salary feature of ADULT dataset by \dpgan and DP-WGAN in  \cite{frigerio2019differentially} over different privacy parameter $\eps$.}
\label{tab:acc_score}
        \begin{tabular}{|c|c|c|c|c|c|c|c|} \hline
                $\eps$ value & Real dataset & $\infty$  & 7 & 3 & 1.01 & 0.51 & 0.36\\ \hline
                DP-auto-GAN Accuracy (ours) & 84.53\%& 79.18\% &  & & 79.19\% & 78.68\% & 74.66\% \\    \hline
              DP-WGAN  Accuracy & 77.2\%& 76.7\% & 76.0\%  & 75.3\% &  &  &  \\    \hline
        \end{tabular}
\end{table*}

\paragraph{Random Forest Prediction Scores.} Following   \cite{frigerio2019differentially}, we also evaluate the quality of synthetic data by the accuracy of a random forest classifier to predict the label ``salary'' feature. In particular, we train a random forest classifier on synthetic data and test on the  holdout original data, and report the $F_1$ accuracy score.
The aim is that a classifier trained on synthetic data should report a similar accuracy score as the one trained on the original data. 

In Table \ref{tab:acc_score}, we report the accuracy of  synthetic datasets generated by \dpgan and DP-WGAN \citep{frigerio2019differentially}.  The results reported in \cite{frigerio2019differentially} use $\eps=3,7,\infty$, whereas our algorithms used parameter values $\eps=0.36, 0.51, 1.01, \infty$, a significant improvement in privacy. We see that our accuracy guarantees are higher than those of \cite{frigerio2019differentially} with smaller $\eps$ values, and DP-auto-GAN achieved higher accuracy in the non-private setting. We note that part of the accuracy discrepancy because \dpgan can handle mixed-typed features, whereas DP-WGAN only handles categorical features. 




\begin{figure}[tbh]
        \centering
        \begin{subfigure}{0.3\textwidth}
        \centering
        \includegraphics[width = \linewidth]{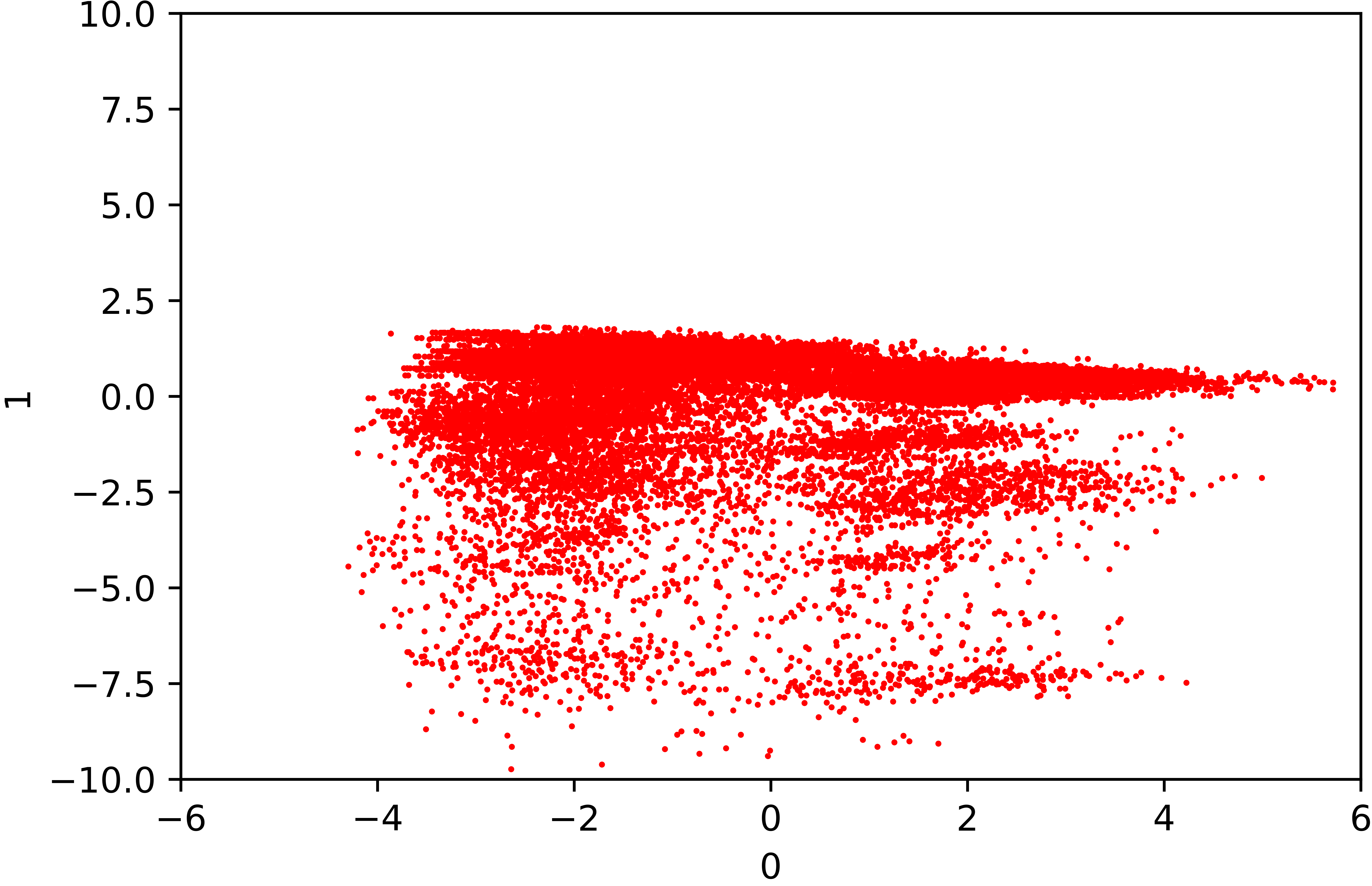}
        \caption{Original data}
        \end{subfigure} 
        
        ~
        
        \begin{subfigure}{0.24\textwidth}
        \includegraphics[width = \linewidth]{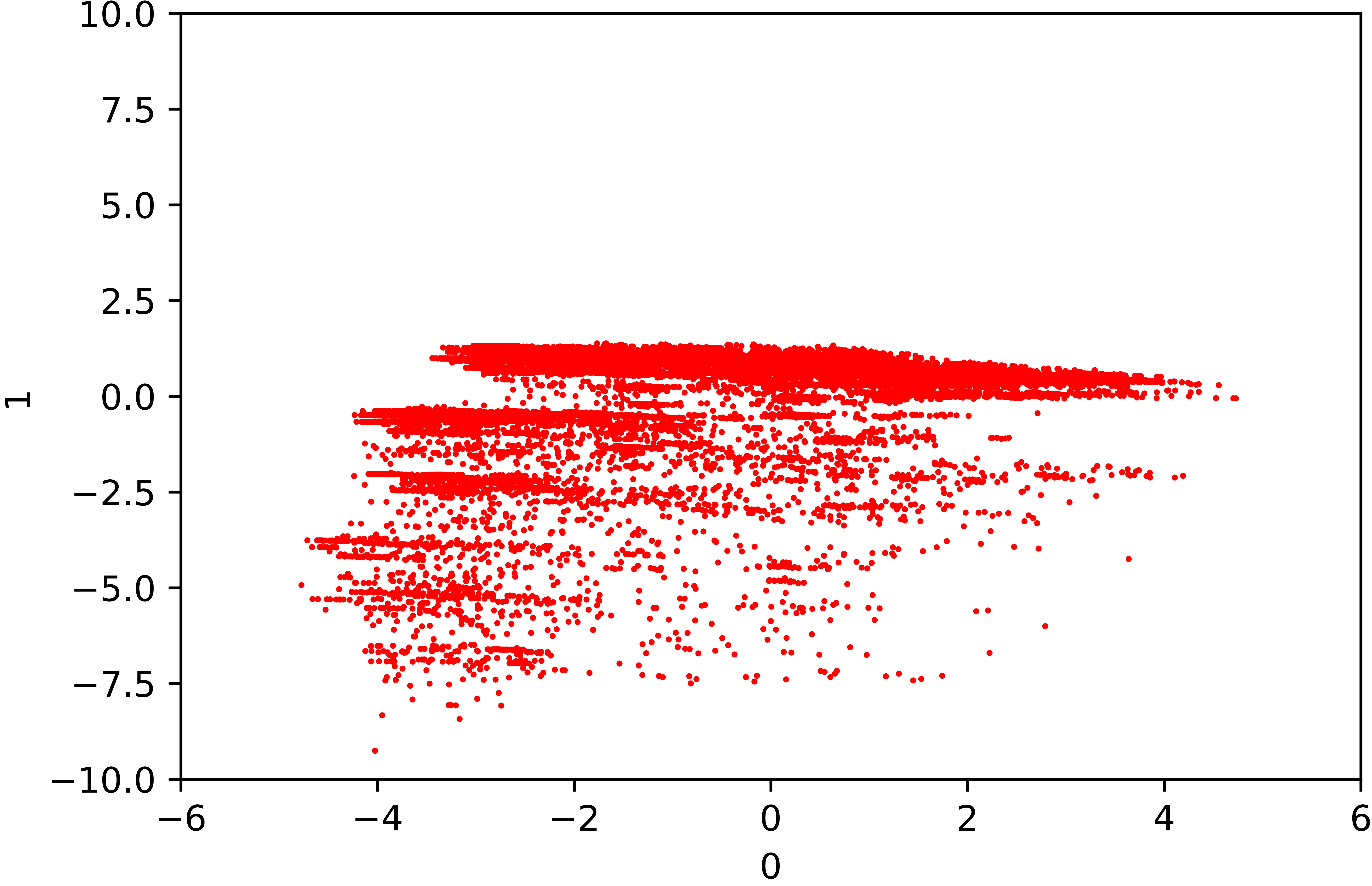}
        \caption{DP-auto-GAN,  $\eps = \infty$}
        \end{subfigure} 
        \begin{subfigure}{0.24\textwidth}
        \includegraphics[width = \linewidth]{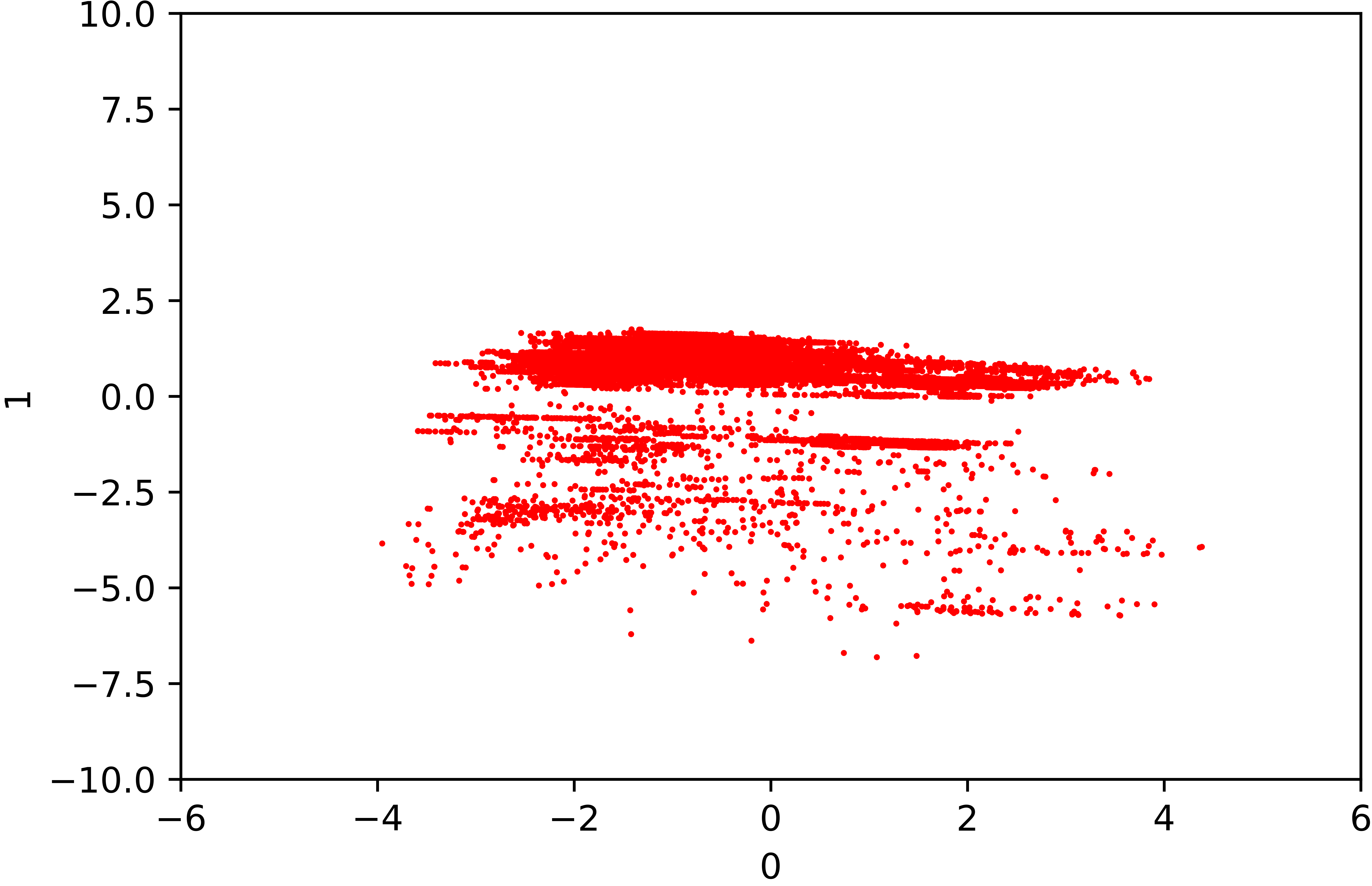}
        \caption{$\eps = 1.01$}
        \end{subfigure}
        \begin{subfigure}{0.24\textwidth}
        \includegraphics[width = \linewidth]{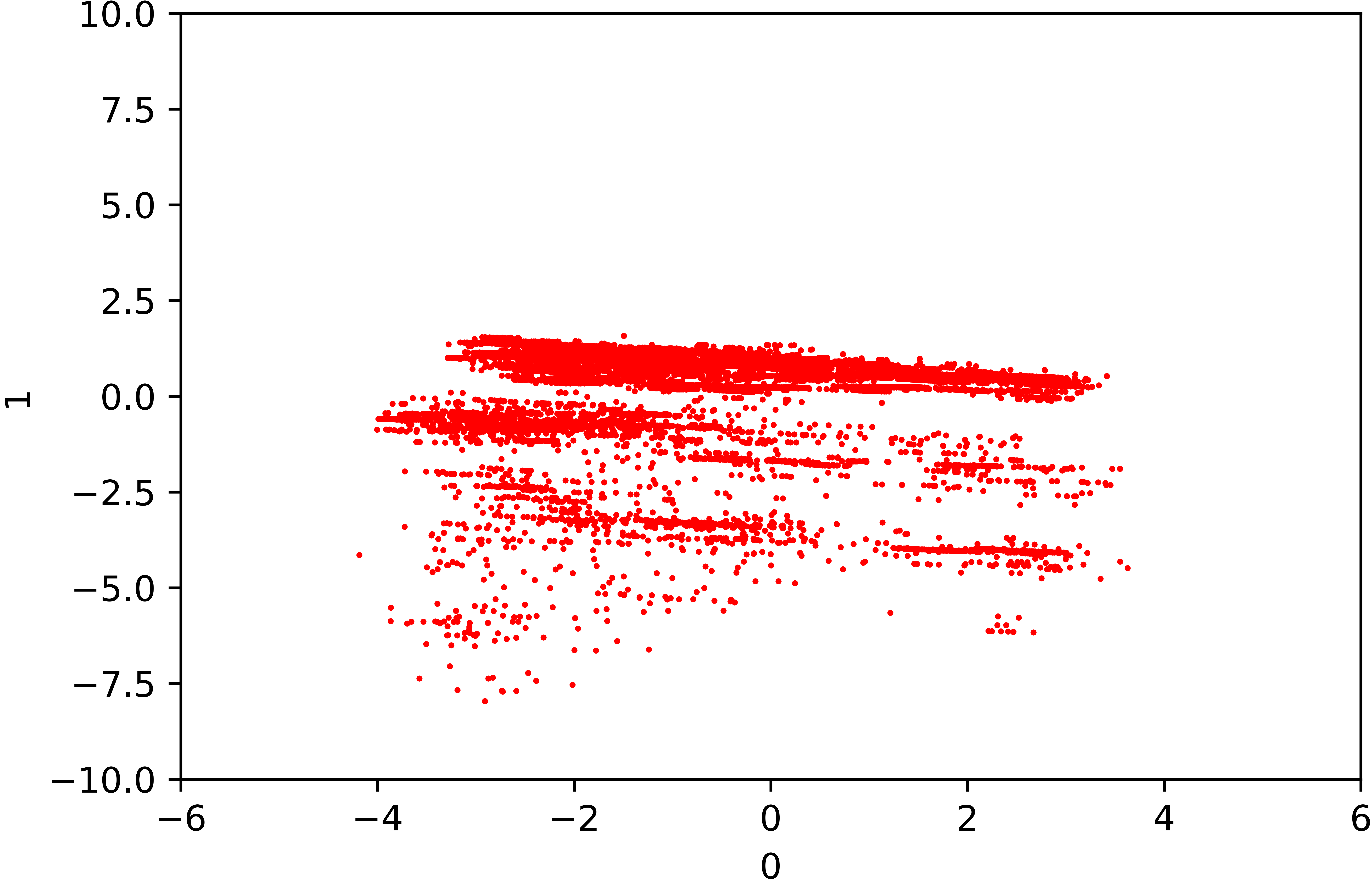}
        \caption{$\eps = 0.51$}
        \end{subfigure}
        \begin{subfigure}{0.24\textwidth}
        \includegraphics[width = \linewidth]{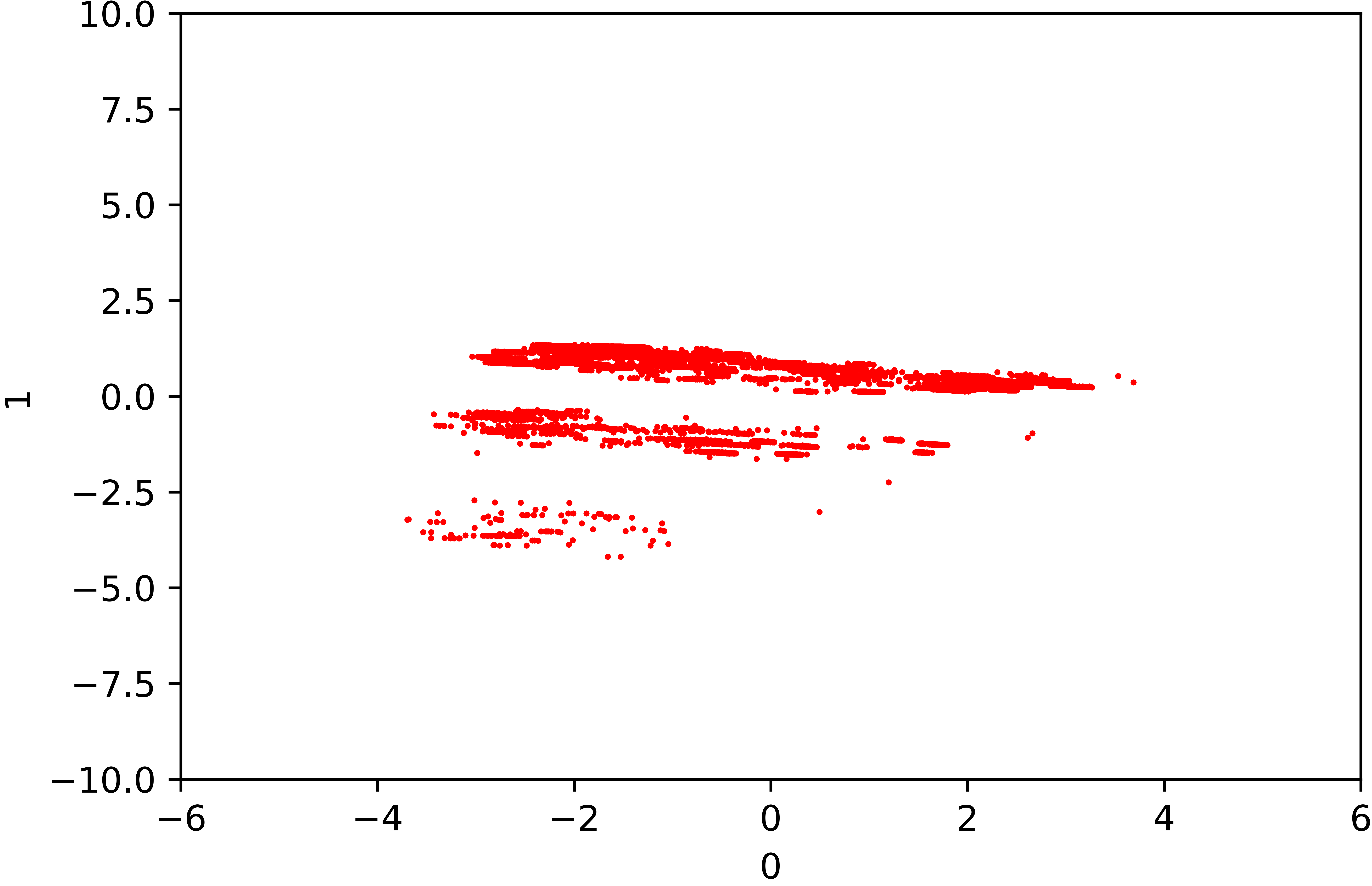}
        \caption{$\eps = 0.36$}
        \end{subfigure} 
        \begin{subfigure}{0.24\textwidth}
        \includegraphics[width = \linewidth]{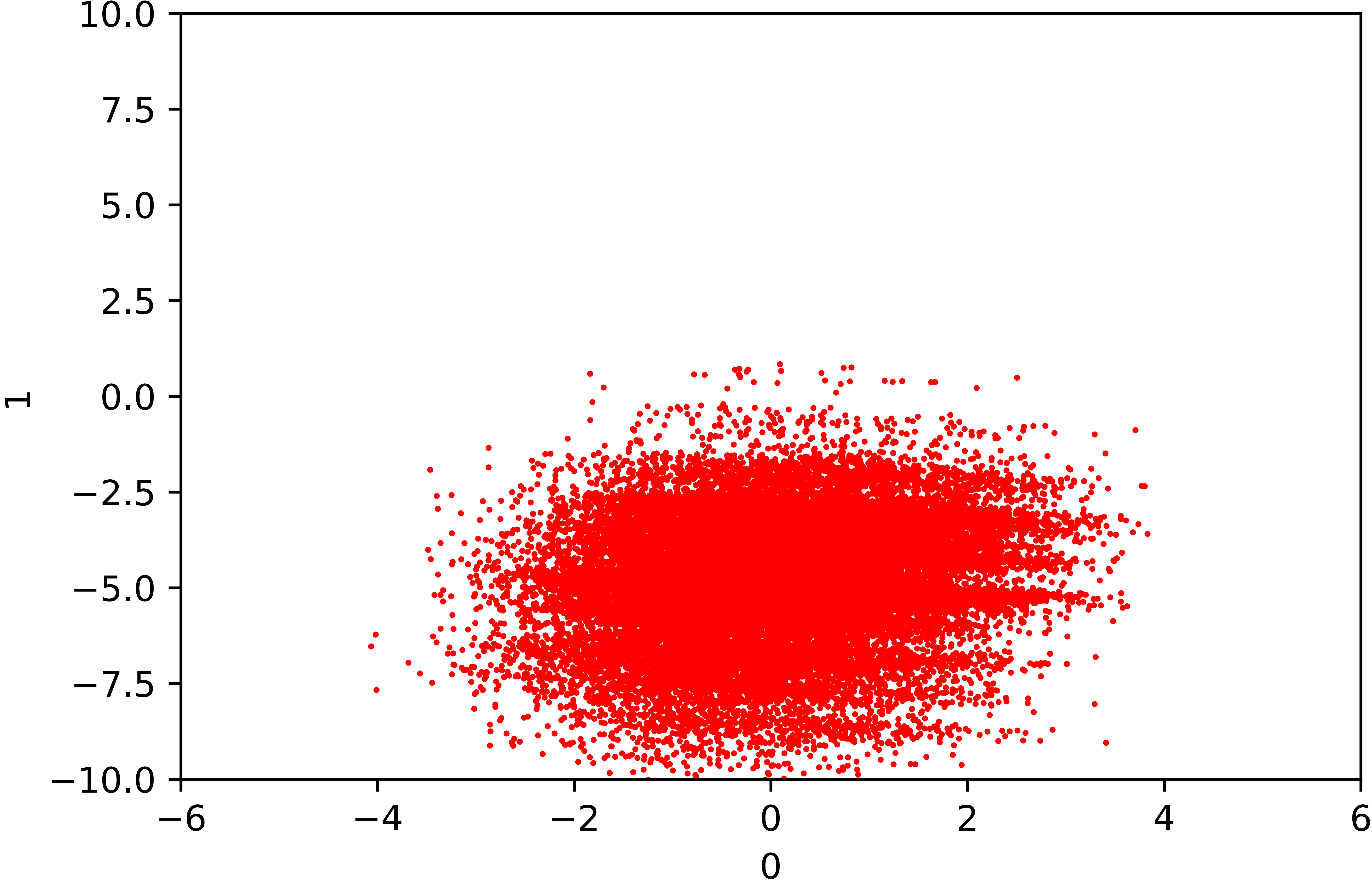}
        \caption{DP-VAE,  $\eps = \infty$}
        \end{subfigure} 
        \begin{subfigure}{0.24\textwidth}
        \includegraphics[width = \linewidth]{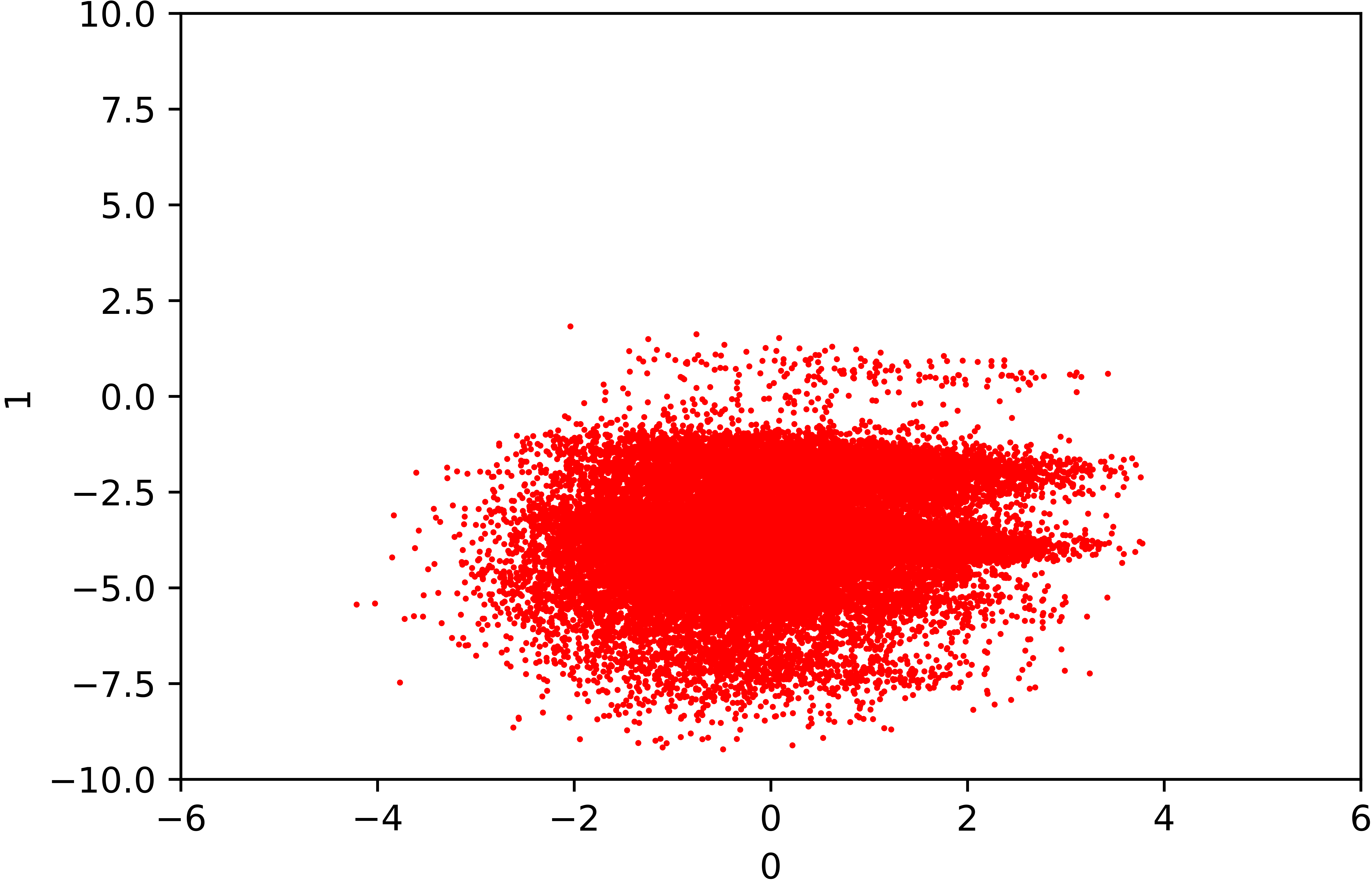}
        \caption{$\eps = 1.01$}
        \end{subfigure}
        \begin{subfigure}{0.24\textwidth}
        \includegraphics[width = \linewidth]{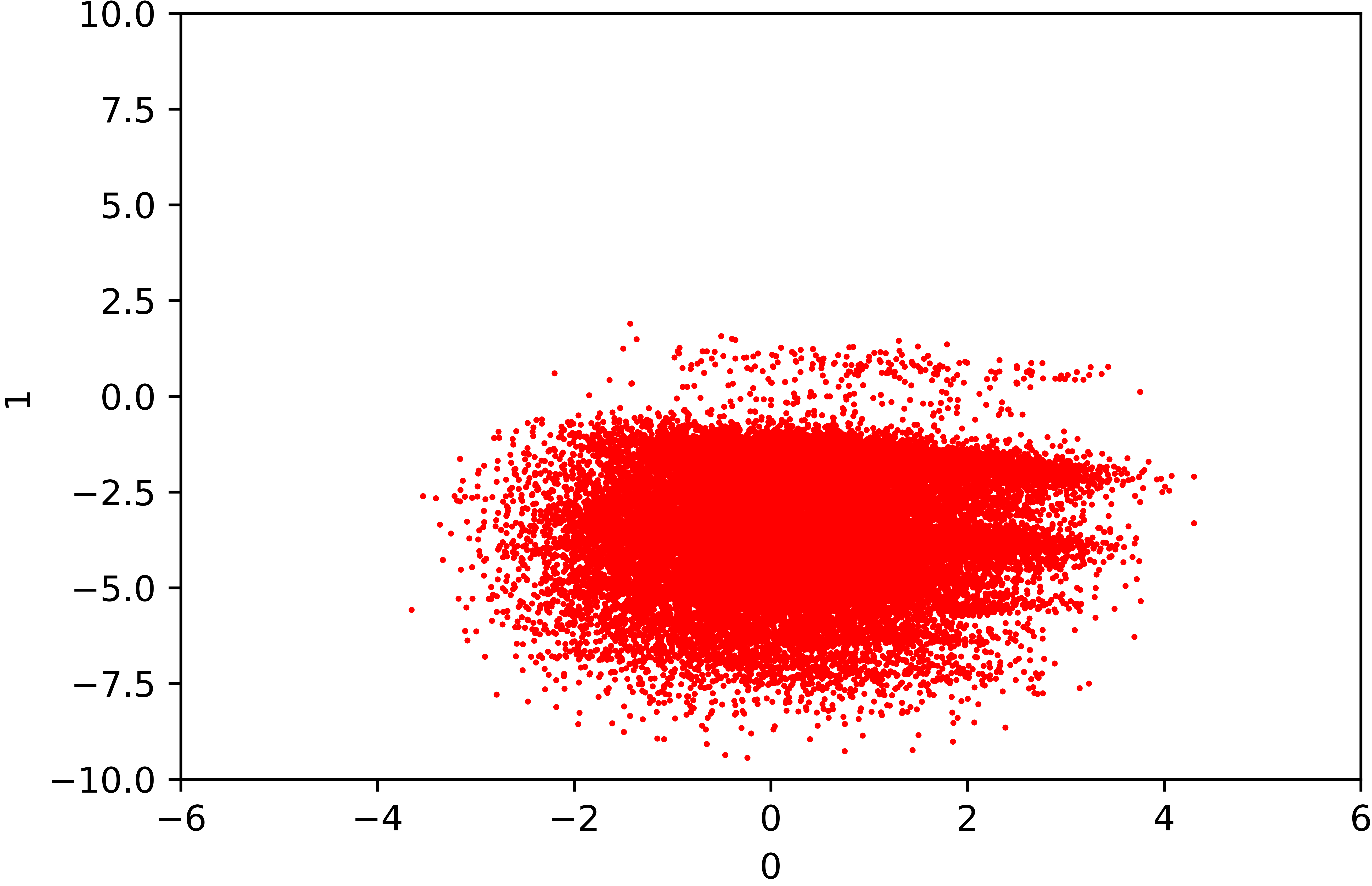}
        \caption{$\eps = 0.51$}
        \end{subfigure}
        \begin{subfigure}{0.24\textwidth}
        \includegraphics[width = \linewidth]{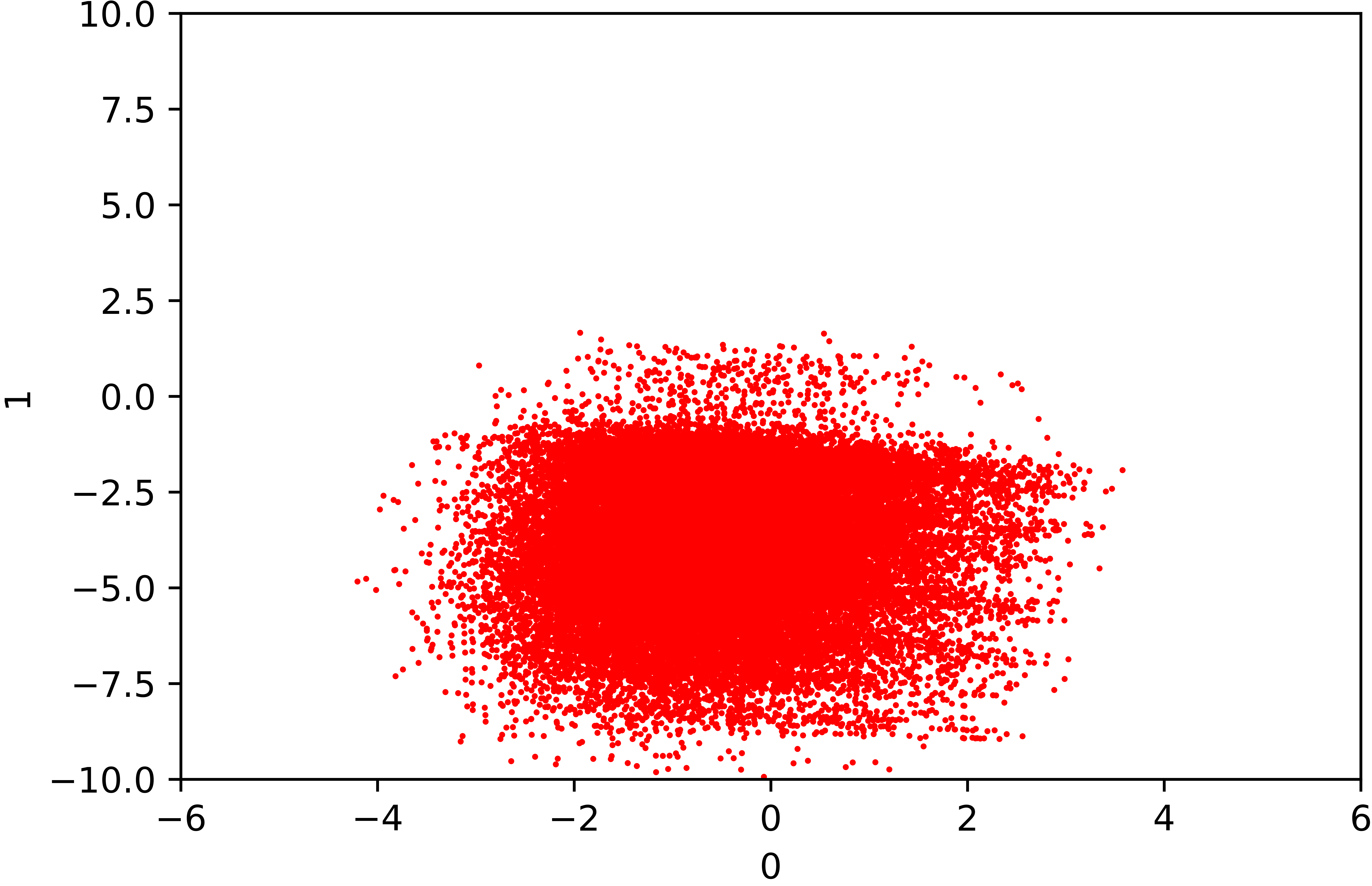}
        \caption{$\eps = 0.36$}
        \end{subfigure}         
        \begin{subfigure}{0.24\textwidth}
        \includegraphics[width = \linewidth]{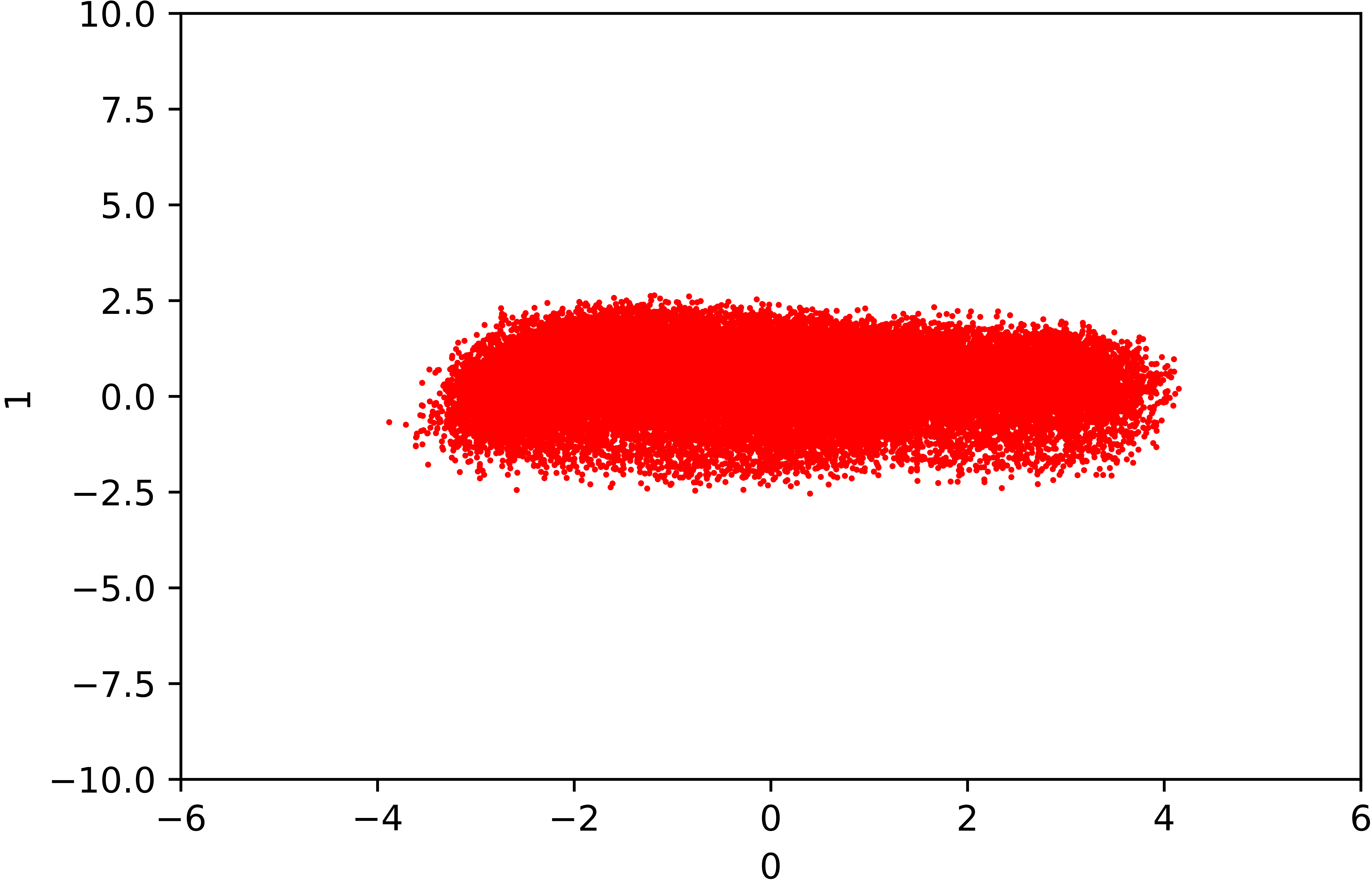}
        \caption{DP-SYN,  $\eps = \infty$}
        \end{subfigure} 
        \begin{subfigure}{0.24\textwidth}
        \includegraphics[width = \linewidth]{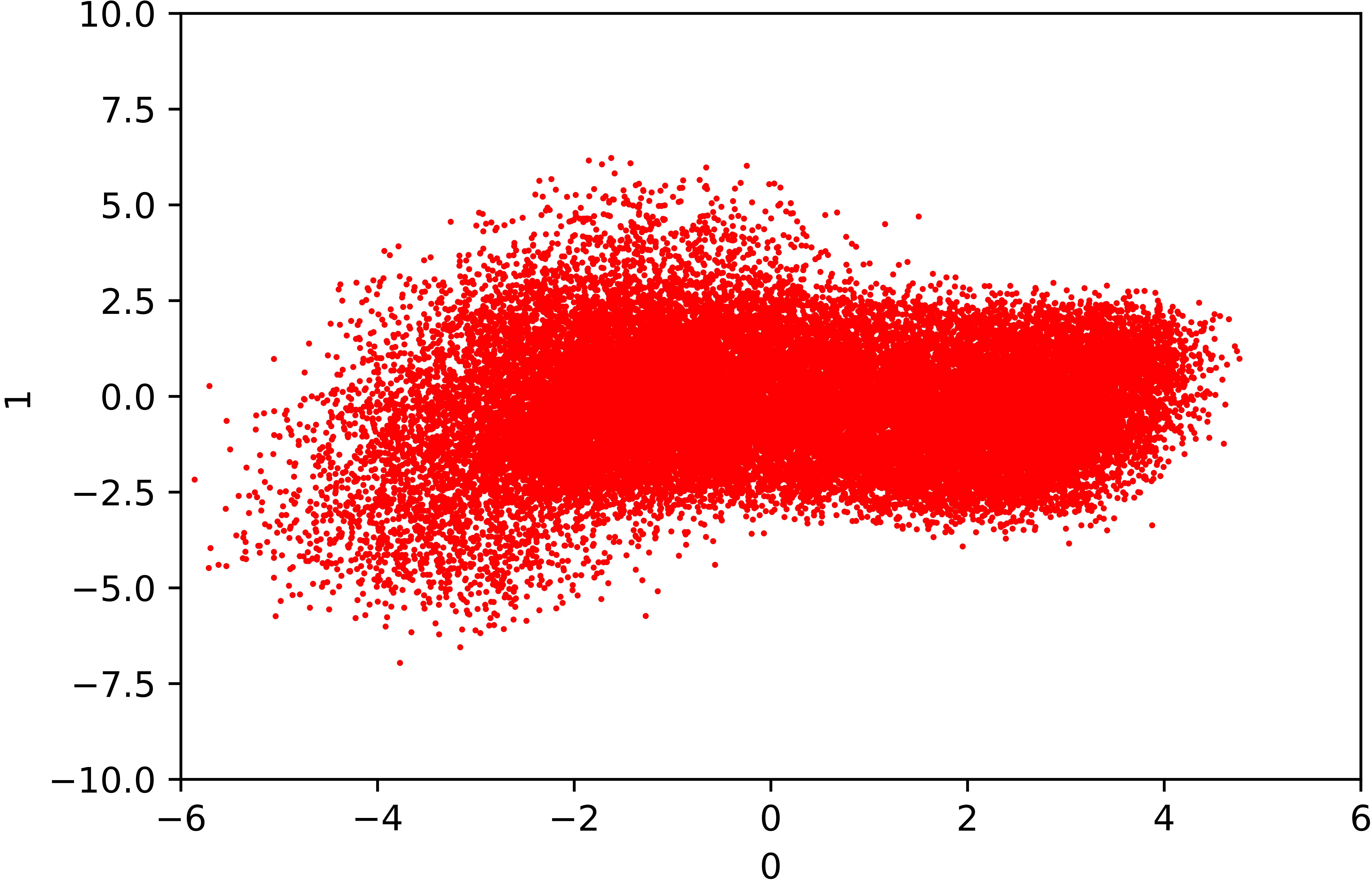}
        \caption{$\eps = 1.40$}
        \end{subfigure}
        \begin{subfigure}{0.24\textwidth}
        \includegraphics[width = \linewidth]{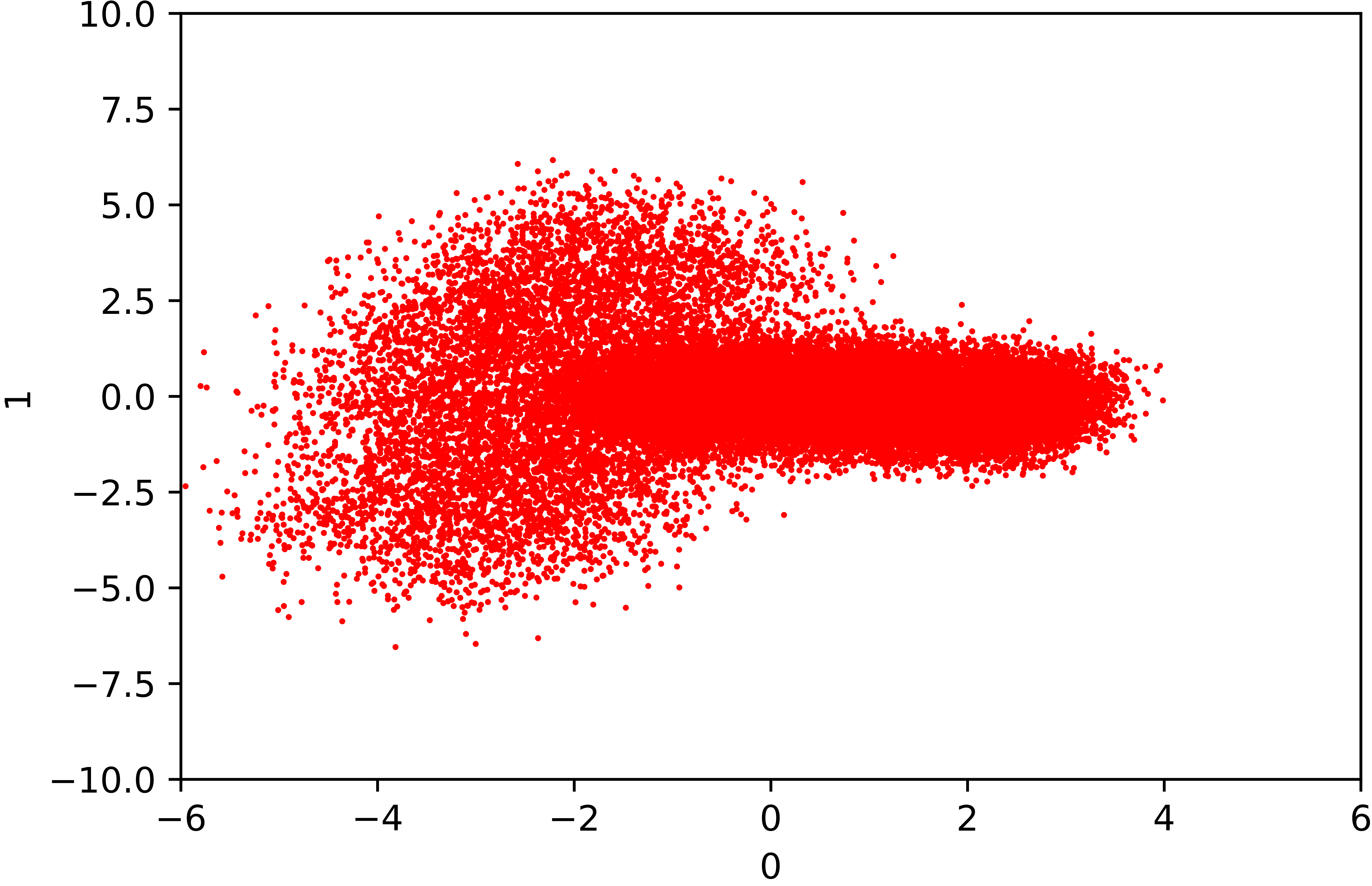}
        \caption{$\eps = 0.80$}
        \end{subfigure}
        \begin{subfigure}{0.24\textwidth}
        \includegraphics[width = \linewidth]{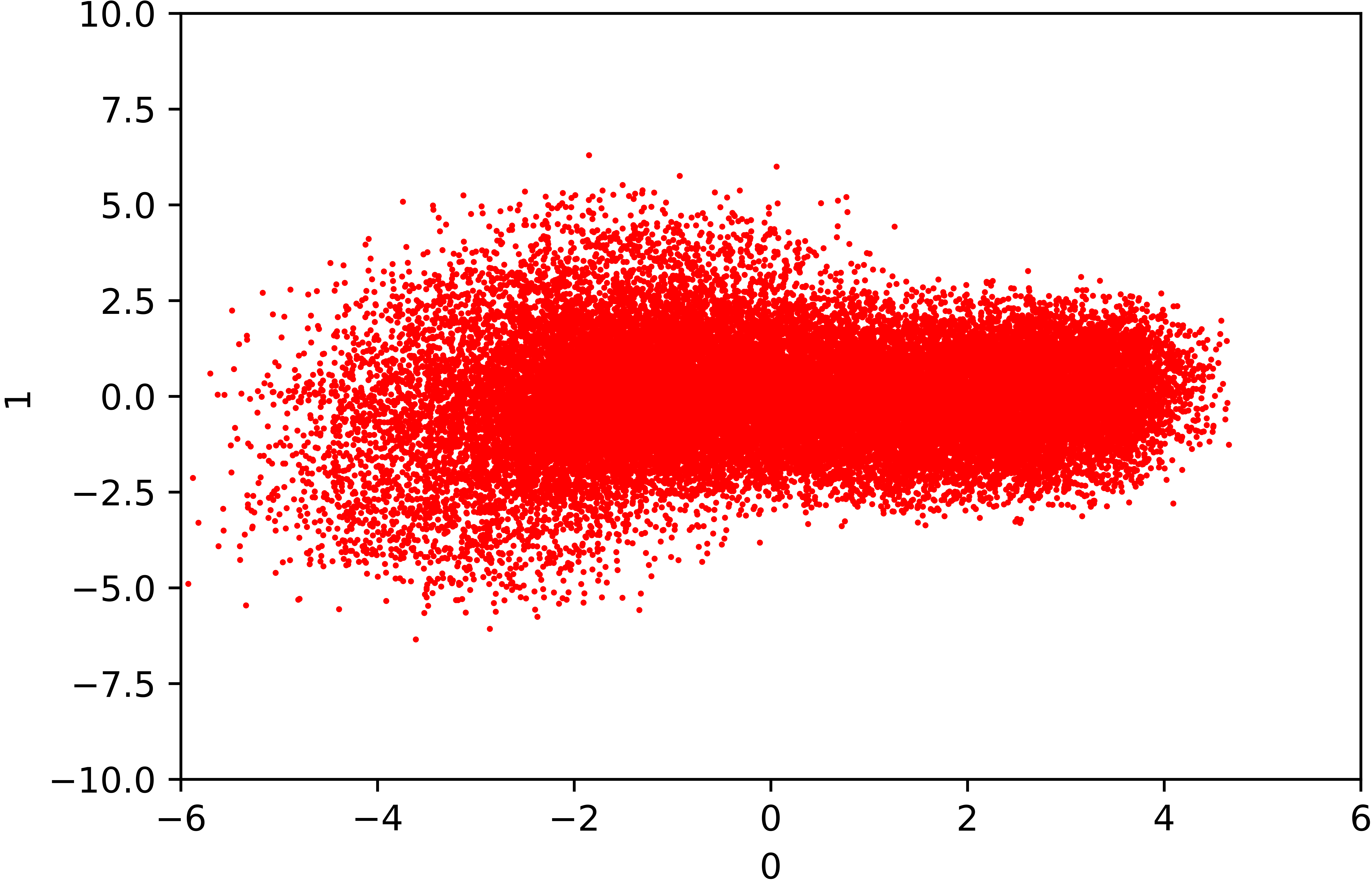}
        \caption{$\eps = 0.50$}
        \end{subfigure}                
        \caption{Scatterplots of projection of ADULT  original and synthetic datasets on first two principle components of the original dataset. Synthetic datasets are generated from several algorithms at different \(\eps\) values.}
\label{fig:2-way-pca-adult}     
\end{figure}

\paragraph{2-Way PCA. } In order to understand combined qualitative performance of all features, we show 2-way PCA marginal in Figure \ref{fig:2-way-pca-adult}. We fix the same projection from the original data, and require the synthetic data to be of the same format and go through the same preprocessing. For this reason, we do not compare to DP-WGAN since the original implementation \cite{frigerio2019differentially} does not handle continuous columns, and we apply our preprocessing rather than the original preprocessing for DP-SYN. A qualitative inspection of the plots clearly shows the similarities of trends between the plots for real dataset and synthetic data generated by \dpgan for different values of $\eps$, as low as $\eps =0.51$. 


Figure \ref{fig:2-way-pca-adult} is able to depict a qualitative description of the synthetic datasets that dimension-wise probability and predictive scores may not capture. \dpgan is able to capture the overall structure of 2-dimension PCA, with more points collapsing into a cluster as privacy budget \(\eps\) decreases. The algorithm's internal GAN structure, however, is able to generate points on different clusters even at smaller \(\eps\), which better matches the projection of the original dataset. DP-VAE has a clear 1-cluster Gaussian-like distribution of 2-way PCA, which is consistent with this method's assumption that the underlying distribution in the latent space is Gaussian. The autoencoder transforms some datapoints near the edge of the cluster to shapes similar to the original 2-way PCA, but most datapoints remain at the center of the cluster. DP-SYN is able to capture the large single cluster in the original data, but is not diverse enough to capture small clusters. This is consistent with our previous observations on diversity under DP-SYN and the fact that DP-SYN assumes a mixture of Gaussian distributions in latent space, which may not be diverse or complex enough to capture smaller clusters of the original distribution.  

DP-SYN and DP-VAE do not improve 2-way PCA plots as \(\eps\) increases, suggesting that the underlying assumptions in latent space are likely a bottleneck. Interestingly, for \(\eps=\infty\), DP-SYN synthetic data collapse to a single flat cluster (it is possible that many clusters are generated, but only one appears due to PCA), suggesting that DP-SYN overfits to the majority, and that adding noise for privacy splits the cluster. \dpgan 2-way PCA does not show structural limitations, but rather a promising result that it is possible to generate more detailed distributions that are closer to the original data.

\section{Conclusion}

We propose \dpgan---a combination of DP-autoencoder and DP-GAN---for differentially private data generation of mixed-type data. The inclusion of the autoencoder improves the efficacy of GANs, especially for high-dimensional data. Our method enjoys a ~5x privacy improvement compared to \cite{frigerio2019differentially} on the ADULT dataset in 14 dimensions and greater 100x improvement compared to \cite{xie2018differentially} on a higher 1071-dimensional dataset, and achieves a meaningful privacy  \(\epsilon<1\) for practical use. This approach is more complex than assuming a standard Gaussian distribution as in DP-VAE \citep{acs2018differentially}, and is better able to learn relationships among features.

\bibliography{reference}
\bibliographystyle{plainnat}

\appendix
\section{Algorithm Description and Pseudocode of DP-Auto-GAN}\label{app.algodetail}

In this appendix, we provide the pseudocode of the subroutines in DP-auto-GAN (Algorithm \ref{alg:all}): \textsc{DPTrain}$_{\textsc{auto}}$, \textsc{DPTrain}$_\textsc{Discriminator}$, and \textsc{Train}$_\textsc{Generator}$. The complete DP-auto-GAN algorithm is specified by the architecture and training parameters of the encoder, decoder, generator, and discriminator. 

After initial data pre-processing, the \textsc{DPTrain}$_{\textsc{auto}}$ algorithm trains the autoencoder.  Details of this training process are fully specified in Algorithm \ref{alg:auto}. As noted earlier, the decoder is trained privately by clipping gradient norm and injecting Gaussian noise in order to obtain the gradient of decoder \(g_\theta\), while the gradient of encoder \(g_\phi\) can be used directly as encoder can be trained non-privately.

The second phase of DP-auto-GAN is to train the GAN. As suggested by \cite{GAN14}, the discriminator trained for several iterations per one iteration of generator training. While the discriminator is being trained, the generator is fixed, and vice-versa. The discriminator and generator training are described in Algorithms \ref{alg:disc} (\textsc{DPTrain}$_\textsc{Discriminator}$) and \ref{alg:gan} (\textsc{Train}$_\textsc{Generator}$) respectively.
Since the discriminator receives real data samples as input for training, the training is made differentially private by clipping the norm of the gradient updates, and adding Gaussian noise to the gradient \(g\). The generator does not use any real data in training (or any functions of the real data that were computed without differential privacy), and hence it can be trained without any need to clip the gradient norm or to inject noise into the gradient.

Finally, the overall privacy analysis of DP-auto-GAN is done via the RDP accountant for each training, and composing at the RDP level (as a function of \(\alpha\)) as described in Section \ref{sec:rdp} and Lemma \ref{lem:rdp-better}.

\begin{algorithm}
    \caption{\textsc{DPTrain}$_{\textsc{auto}}$(\(X\),  \(En_\phi\),  \(De_\theta\), training parameters) }
    \label{alg:auto}
    \begin{algorithmic}[1] 
        \State \textbf{training parameter input}:
Learning rate \(\eta_1\), number of iteration rounds (or optimization steps) \(T_1\), loss function \(L_\text{auto}\), optimization method \textsc{optim}\(_\text{auto}\) batch sampling rate \(q_1\) (for the batch expectation size \(b_1=q_1m\)), clipping norm \(C_1\), noise multiplier \(\psi_1\), microbatch size \(r_1\)
\State \textbf{goal}: train one step of autoencoder
\((En_\phi,De_\theta)\)    \Procedure{\textsc{DPTrain}$_{\textsc{auto}}$}{}
        
                \State \(\BB\leftarrow\) \textsc{SampleBatch}(\(X,q_1\))
                \State Partition \(\BB\) into \(B_1,\ldots,B_k\) each of size \(r\) (ignoring the dividend)
                \State \(\hat k \leftarrow \frac{q_1 m}{r}\) \Comment{ an estimate of \(k\)}
                \For{\(j=1\ldots k\)}
                \LeftComment{\textit{Both \(g_\phi^j,g_\theta^j \) can be computed in one backpropagation}}
                        \State \(g_\phi^j,g_\theta^j \leftarrow \nabla_\phi(L_\text{auto}(De_\theta(En_\phi(B_j)),B_j)),\nabla_\theta(L_\text{auto}(De_\theta(En_\phi(B_j)),B_j)\)
                            
                \EndFor
        \State \(g_\phi\leftarrow \frac{1}{\hat k} \sum_{j=1}^k g_\phi^j\)
        \State \(g_\theta\leftarrow \frac{1}{\hat k} \left( \left(\sum_{j=1}^k \textsc{Clip} (g_\phi^j,C_1) \right)+\NN(0,C_1^2\psi_1^2 I)\right) \) 
        \State \((\phi,\theta)\leftarrow \textsc{optim}_\text{auto}(\phi,\theta,g_\phi,g_\theta,\eta_1)\)
        
    \EndProcedure
    \end{algorithmic}
\end{algorithm}

\begin{algorithm}
    \caption{\textsc{DPTrain}$_\textsc{Discriminator}$(\(X\),  \(Z\), \(G_w,De_\theta\),  $D_y$,  training parameters)}
    \label{alg:disc}
    \begin{algorithmic}[1] 
        \State \textbf{training parameter input}:
Learning rate \(\eta_3\), number of discriminator iterations per generator step \(t_D\), loss function \(L_D\), optimization method \textsc{optim}\(_D\), batch sampling rate \(q_3\) (for the batch expectation size \(b_3=q_3m\)), clipping norm \(C_3\), noise multiplier \(\psi_3\), microbatch size \(r_3\)
\State \textbf{goal}: train one step of discriminator
\(D_y\)    \Procedure{\textsc{DPTrain}$_{\textsc{discriminator}}$}{}
                \State \(\BB\leftarrow\) \textsc{SampleBatch}(\(X,q_3\))
                \State Partition \(\BB\) into \(B_1,\ldots,B_k\) each of size \(r\) (ignoring the dividend)
                \State \(\hat k \leftarrow \frac{q_1 m}{r}\) \Comment{ an estimate of \(k\)}
                \For{\(j=1\ldots k\)}
                        \State \(\{z_i\}_{i=1}^r\sim Z^r\)
                        \State \(B'\leftarrow \{De(G_w(z_i))\}_{i=1}^r\)
                        \State \(g^j\leftarrow \nabla_y(L_D(B_j,B',D_y)) \)
                        \LeftComment{In the case of WGAN, 
                        \[L_D(B_j,B',D_y):=\frac{1}{r}\sum_{b\in B_j} D_y(b)-\frac 1r \sum_{b'\in B'} D_y(b') \] }
                            
                \EndFor
        \State \(g\leftarrow \frac{1}{\hat k} \left( \left(\sum_{j=1}^k \textsc{Clip} (g^j,C_3) \right)+\NN(0,C_3^2\psi_3^2 I)\right) \) 
        \State \(y \leftarrow \textsc{optim}_{D}(y,g,\eta_3)\)
    \EndProcedure
    \end{algorithmic}
\end{algorithm}

\begin{algorithm}
    \caption{\textsc{Train}$_\textsc{Generator}$(\(Z,G_w,De_\theta,D_y\), generator training parameters) }
    \label{alg:gan}
    \begin{algorithmic}[1] 
        \State \textbf{training parameter input}:
Learning rate \(\eta_2\), batch size \(b_2\), loss function \(L_G\), optimization method \textsc{optim}\(_G\), number of generator iteration rounds (or optimization steps) \(T_2\)    
\State \textbf{goal}: train one step of generator \(G_w\)

\Procedure{\textsc{Train}$_{\textsc{generator}}$}{}
                
                        \State \(\{z_i\}_{i=1}^{b_2}\sim Z^{b_2}\)
                        \State \(B'\leftarrow \{De(G_w(z_i))\}_{i=1}^{b_2}\)

                        \State \(g\leftarrow \nabla_w(L_G(B',D_y)) \)
                        \LeftComment{In the case of WGAN, 
                        \[L_G(B',D_y):=-\frac{1}{b_2}\sum_{b'\in B'} D_y(b')  \] }

        \State \(w \leftarrow \textsc{optim}_{G}(w,g,\eta_2)\)
    \EndProcedure
    \end{algorithmic}
\end{algorithm}





\section{Comparison of RDP and Standard Composition and Proof of Lemma \ref{lem:rdp-better}}\label{app.rdp}
In this section, we prove Lemma \ref{lem:rdp-better} and give a discussion for the privacy factor reduction found in practice by Lemma \ref{lem:rdp-better}.

\rdpbetter*

\begin{proof}
Let
\[
\alpha_1^*\in\argmin_{\alpha>1} r_1(\alpha)+\frac{\log(2/\delta)}{\alpha-1} \quad \text{and}\quad \alpha_2^*\in\min_{\alpha>1} r_2(\alpha)+\frac{\log(2/\delta)}{\alpha-1}
\]
and let \(\alpha=\min\{\alpha_1^*,\alpha_2^*\}\). Then, we have
\begin{align*}
\eps &\leq r_1(\alpha)+r_2(\alpha)+\frac{\log(1/\delta)}{\alpha-1} \\
&\leq r_1(\alpha_1^*)+r_2(\alpha_2^*)+\frac{\log(1/\delta)}{\alpha-1} \\
&=\epsilon_1+ \eps_2+\frac{\log(1/\delta)}{\alpha-1}-\frac{\log(2/\delta)}{\alpha_1^*-1}-\frac{\log(2/\delta)}{\alpha_2^*-1} <\eps_1+\eps_2 
\end{align*}
where the two inequalities use the definitions of \(\eps_1,\eps_2,\eps\), and the second inequality uses the fact that \(r_i\) is an increasing function of \(\alpha\) (\cite{van2014renyi}).
\end{proof}

For most settings of training parameters, we found that \(\eps\) by RDP composition in Lemma \ref{lem:rdp-better} is \(\approx30\%\) smaller than that of the standard composition (see Figure \ref{fig:rdp-eps} for this privacy saving in our \dpgan \(\eps=0.51\) ADULT setting). The observation can be support by theoretical analysis as follows. It is observed in \cite{wang2018subsampled} that   \(r_i(\alpha)\) appears linear until a phase transition at some \(\alpha\), and is close to linear again. In our parameter settings, the optimal order to achieve smallest \(\eps\) is before the phase transition, and thus \(r_i(\alpha)\) "practically" behaves linear as the privacy analysis never uses \(r_i(\alpha)\) at \(\alpha\) beyond the phase transition. This is illustrated in Figure \ref{fig:rdp-linear} by an example of our \dpgan \(\eps=0.51\) ADULT setting.   

Assuming linear \(r_i(\alpha)=c_i\alpha\), we can compute the analytical solutions:
\begin{align*}
\eps_1=c_1+2\sqrt{c_1\log 2/\delta},\quad\eps_2=c_2+2\sqrt{c_2\log 2/\delta},\quad\eps=c_1+c_2+2\sqrt{(c_1+c_2)\log 1/\delta}
\end{align*} 
In practice, \(\delta\) is small compared to \(c_i\)'s and the term \(\sqrt{c_i\log 1/\delta}\) dominates. Hence, \(\eps^2\approx\eps_1^2+\eps_2^2\), and for many settings where we set \(\eps_1\) close to \(\eps_2\) (such as in our setting or DP-SYN \cite{abay2018privacy}), this implies \(\eps\approx0.707(\eps_1+\eps_2)\), an approximately \(30\%\) reduction of privacy cost.

\begin{figure}[h]
\begin{center}
\includegraphics[width = 0.5\linewidth]{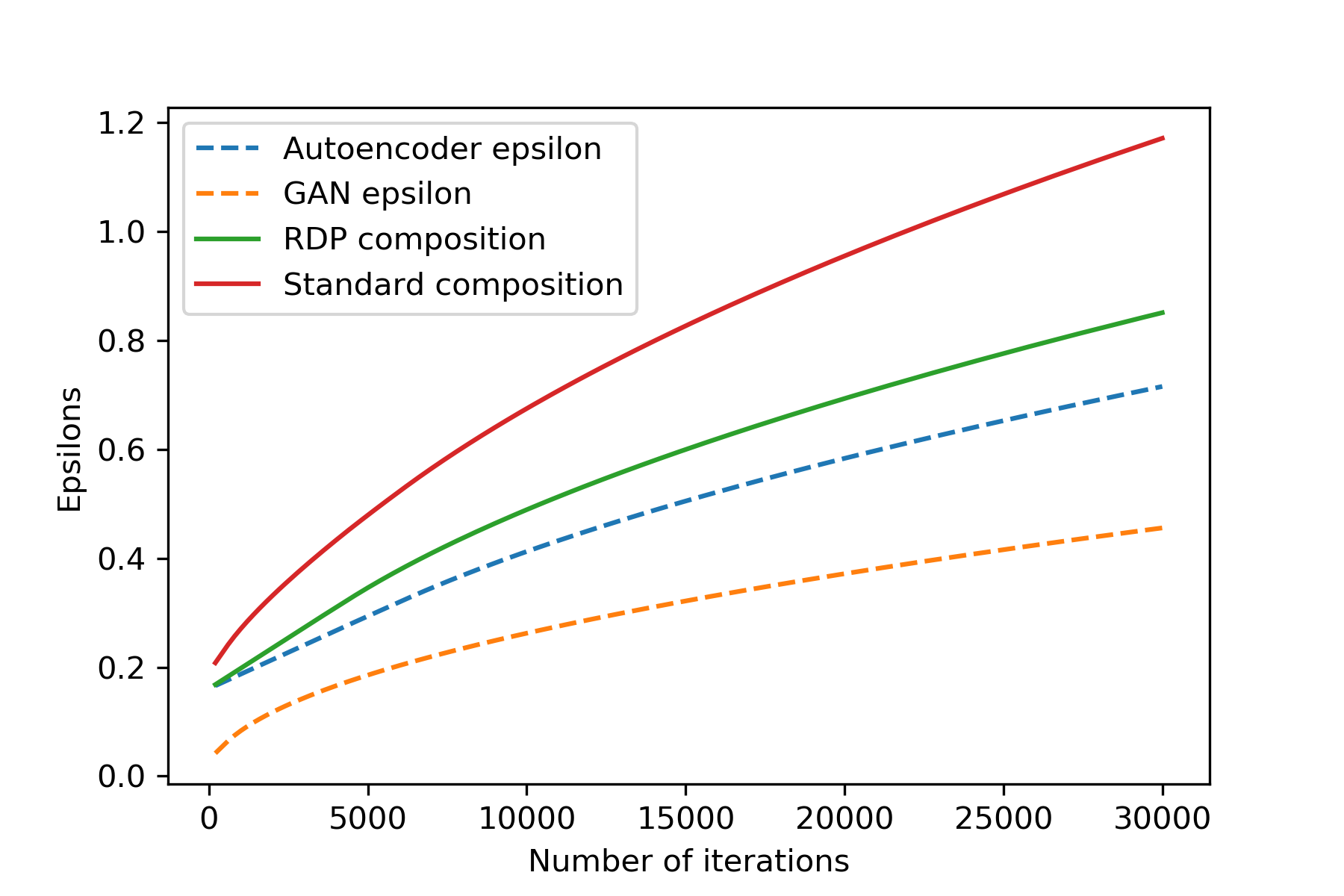}
\end{center}
\caption{Privacy cost \(\epsilon\) for different training phases of the algorithm in \(\eps=0.51\) \dpgan parameter setting for ADULT data: the sampling rate  \(q\) and noise multiplier \(\psi\) for autoencoder and GAN are \(q=\frac{64}{32561},\psi=2.5\) and \(q=\frac{128}{32561},\psi=7.5\), respectively. We target \(\delta=10^{-5}\) overall and  \(\delta=\frac12\cdot10^{-5}\) for each training phase.}
\label{fig:rdp-eps}
\end{figure}

\begin{figure}[h]
\begin{center}
\includegraphics[width = 0.32\linewidth]{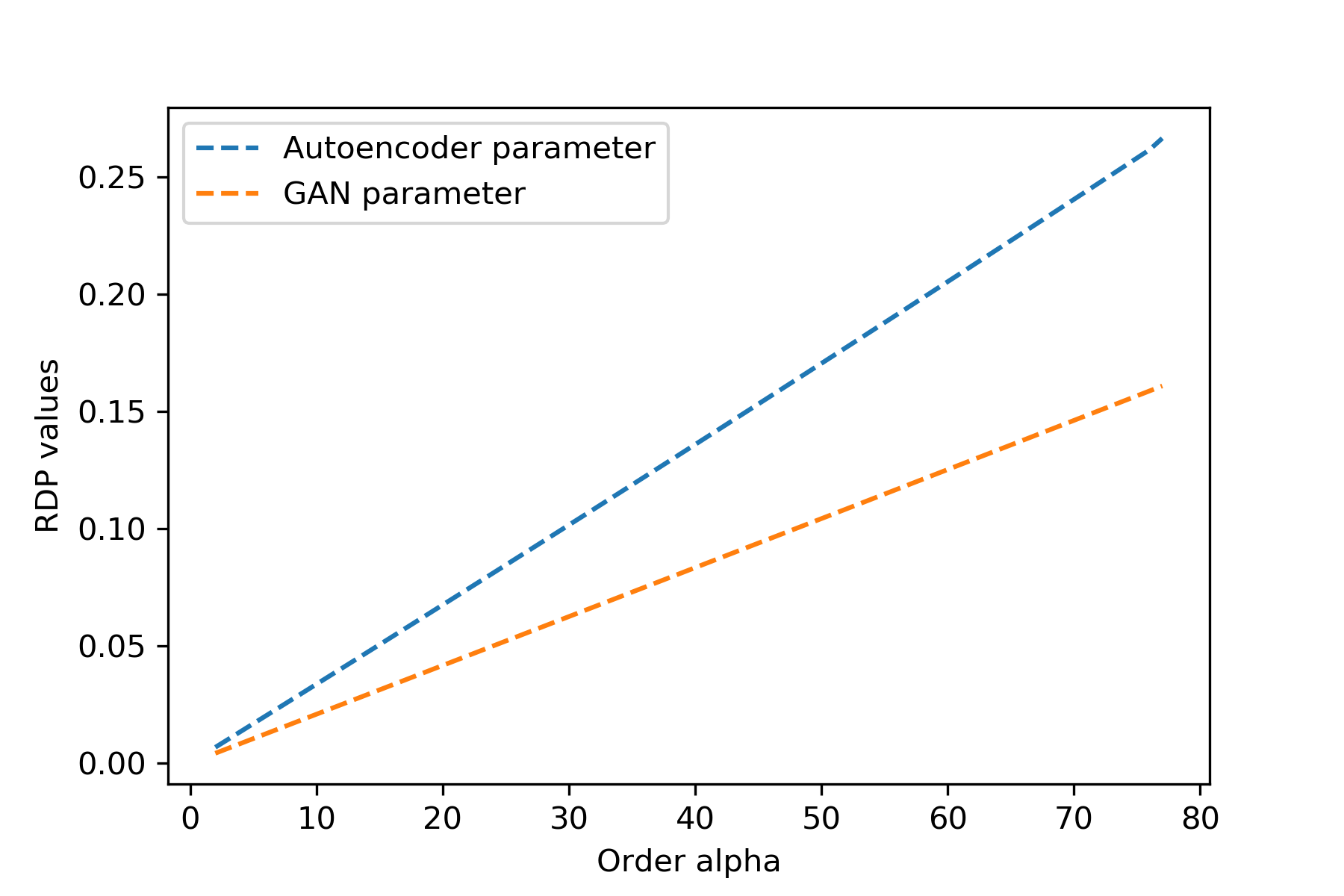}
\includegraphics[width = 0.32\linewidth]{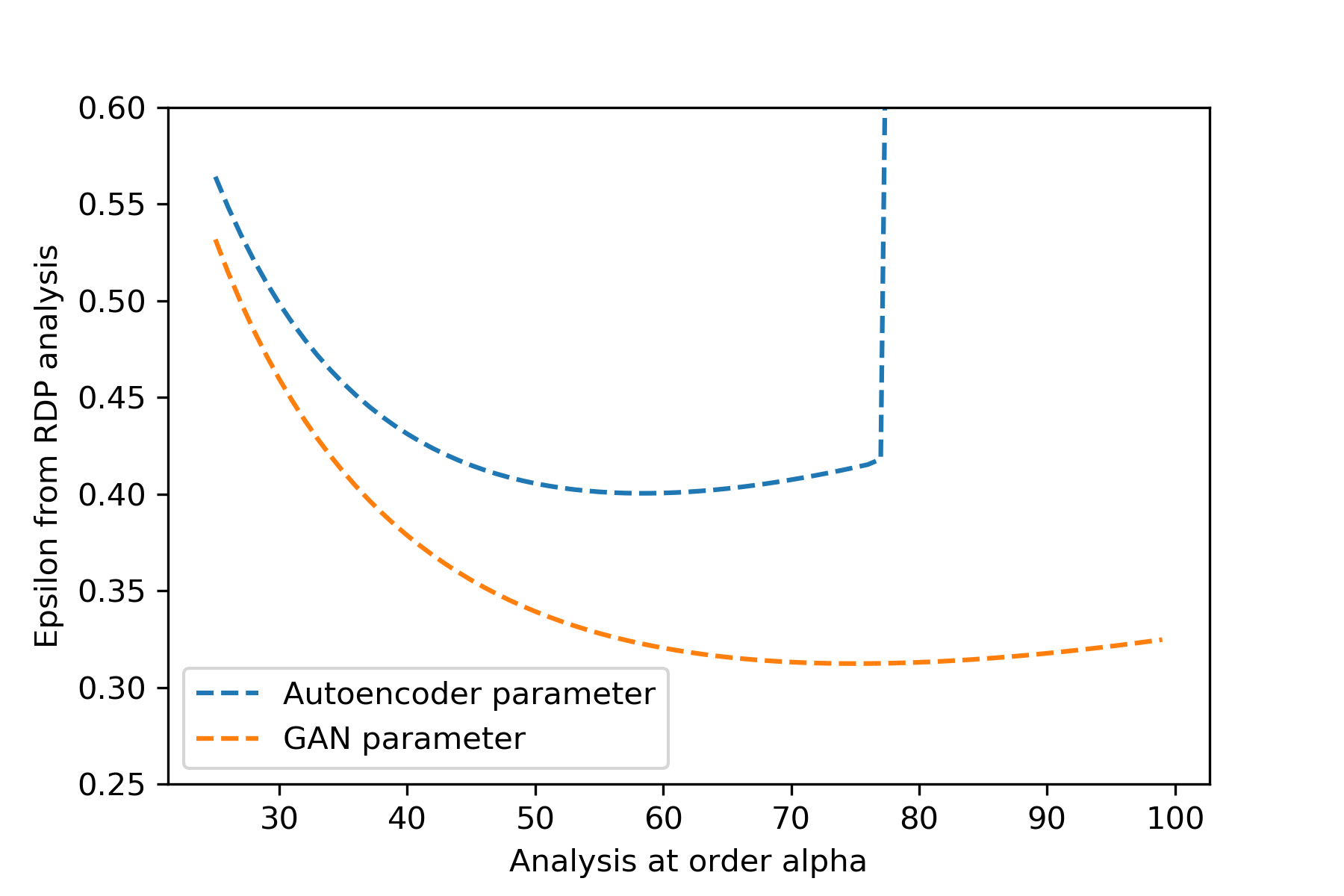}
\includegraphics[width = 0.32\linewidth]{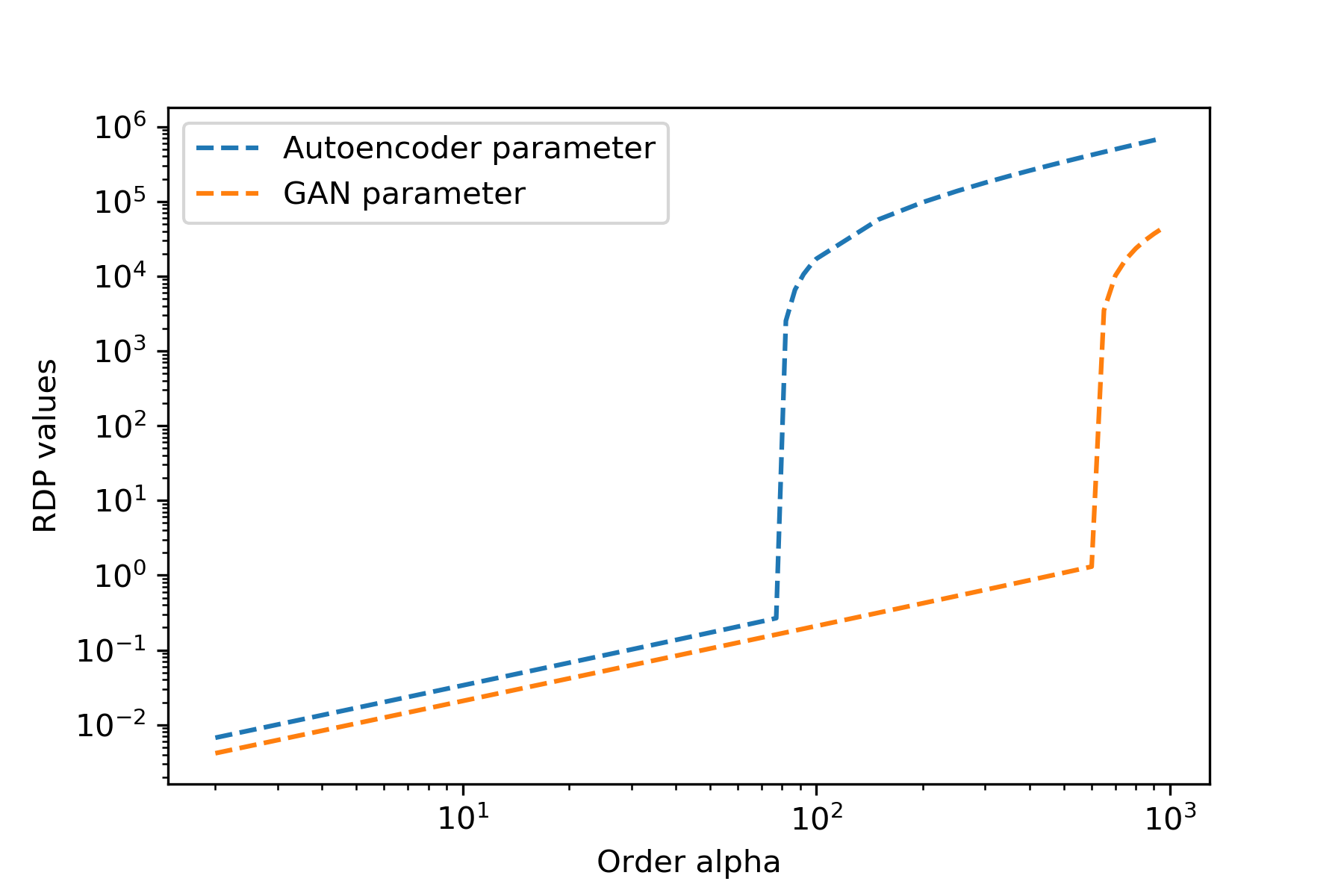}
\end{center}
\caption{RDP values over different order \(\alpha\) for \(\eps=0.51\) \dpgan parameter setting for ADULT data: the sampling rate  \(q\), noise multiplier \(\psi\), and \(T\) the number of training iterations for autoencoder and GAN are \(q=\frac{64}{32561},\psi=2.5,T=10000\) and \(q=\frac{128}{32561},\psi=7.5,T=15000\), respectively. The phase transitions (spikes of RDP value) for autoencoder and GAN appear at \(\alpha=79\) and \(\alpha=624\). The optimal order for smallest \(\eps\) for  autoencoder and GAN analysis targeting \(\delta=\frac1210^{-5}\) are 60 and 77, below the phase transitions.}
\label{fig:rdp-linear}
\end{figure}
\section{Additional Related Work}\label{app.rel}

\subsection{DP-SGD Optimization}\label{app.deeptrain}

There are numerous works on optimizing the performance of differentially private GANs, including data partitioning (either by class of labels in supervised setting or a private algorithm) \citep{yu2019differentially,papernot2016semi,papernot2018scalable,jordon2018pate,abay2018privacy,acs2018differentially,chen2018differentially}; 
reducing the number of parameters in deep models \citep{mcmahan2017learning};
changing the norm clipping for the gradient in DP-SGD during training \citep{mcmahan2017learning,van2018three,thakkar2019differentially};
  changing parameters of the Gaussian noise used during training \citep{yu2019differentially}; 
and using publicly available data to pre-train the private model with a warm start \citep{zhang2018differentially,mcmahan2017learning}.
  Clipping gradients per-layer of models \citep{mcmahan2018general,mcmahan2017learning} and per-dynamic parameter grouping \citep{zhang2018differentially} are also proposed. Additional details for some of these optimization approaches are given below.

\paragraph{Batch Sampling}
Three ways are known to sample a batch from data in each optimization step. These methods are described in \cite{mcmahan2018general}; we summarize them here for completeness.  The first is to sample each individual's data independently with a fixed probability. This sampling procedure is the one used in the analysis of the subsampled moment accountant in \cite{abadi2016deep,mcmahan2018general} and subsampled RDP composition in \cite{mironov2017renyi}. This RDP composition is publicly available at Tensorflow Privacy \citep{TensorflowPrivacy}. We implement this sampling procedure and use  Tensorflow Privacy to account Renyi Divergence during training.  Another sampling policy is to sample uniformly at random a fixed-size subset of all datapoints. This achieves a different RDP guarantee, which was analyzed in \cite{wang2018subsampled}.  Finally, a common subsampling procedure is to shuffle the data via uniformly random permutation, and take a fixed-size batch of the first $k$ points in shuffled order. The process is repeated after a pass over all datapoints (an epoch). Although this batch sampling is most common in practice, to the best of our knowledge, no subsampled privacy composition is known in this case for the centralized model. 

\paragraph{Hyperparameter Tuning.}
Training deep learning models involves hyperparameter tuning to find good architecture and optimization parameters. This process is also done differentially privately, and the privacy budget must be accounted for. \citet{abadi2016deep} accounts for hyperparameter search using the work of \cite{gupta2010differentially}. \citet{beaulieu2019privacy} uses Report Noisy Max \cite{DR14} to private select a model with top performance when a model evaluation metric is known. Some work has also been done to account for selecting high-performance models without spending much privacy budget \citep{chaudhuri2013stability,liu2019private}.  In our experimental work, we omit the privacy accounting of hyperparameter search, as it is done in some previous works such as  DP-SYN \cite{abay2018privacy} that we compare to,  and also as this is not the focus of our contribution. 

\paragraph{Gaussian Differential Privacy.}
We note that Gaussian differential privacy (GDP) \cite{dong2019gaussian} can also be used in place of Renyi differential privacy (RDP) for the privacy analysis. GDP was used in \cite{bu2019deep} for private training of deep learning, but their privacy parameters obtained in the work are approximate. In addition, we  suspect that for our parameter settings with high noise multiplier (in the range of 1.1-3.5) and numbers of epoch, the saving in privacy will not be significant. Indeed, privacy values obtained by both RDP and GDP are observed to converge  as the number of iterations in training increases  (as depicted in Figure 4 of \cite{bu2019deep}).


\subsection{Evaluation Metrics for Synthetic Data} \label{app.met}

In this section, we review the evaluation schemes for measuring quality of synthetic data in existing literature. Various evaluation metrics have been considered in the literature to quantify the quality of synthetic data \citep{charest2011can}. Broadly, evaluation metrics can be divided into two major categories: supervised  and unsupervised. Supervised evaluation metrics are used when clear distinctions exist between features and labels in the dataset, e.g., for healthcare applications, whether a person has a disease or not could be a natural label. Unsupervised evaluation metrics are used when no feature of the data can be decisively termed as a label. For example, a data analyst who wants to learn a pattern from synthetic data may not know what specific prediction tasks to perform, but rather wants to explore the data using an unsupervised algorithm such as Principle Component Analysis (PCA).  Unsupervised metrics can then be divided into three broad types: prediction-based, distributional-distance-based, and qualitative (or visualization-based). We describe supervised evaluation metrics and all three types of unsupervised evaluation metrics below.  Metrics in previous work and our proposed metrics in this paper are summarized in Table \ref{tab:metric}.

\begin{table*}[h]
        \caption{Summary of  evaluation metrics in DP synthetic data generation literature. We list applicability of each metric to each of the data type. Parts in \textbf{bold} are \textbf{our new contributions}. Evaluation methods with asterisk * are predictive-model-specific, and their applicability therefore depends on types of data that the chosen predictive model is appropriate for. 
Methods with asterisks ** are equipped with  any distributional distance of choice such as Wasserstein distance or total variation distance. 
}
        \label{tab:metric}
        \begin{center}
                \begin{tabular}{p{2.6cm}p{5.3cm}p{1.5cm}p{1.5cm}p{1.5cm}}
                        \bf {TYPES}
&\multicolumn{1}{c}{\bf EVALUATION METHODS}  & \multicolumn{3}{c}{\bf DATA\ TYPES} \\
                        \multicolumn{1}{c}{\bf }  &\multicolumn{1}{c}{\bf } & \multicolumn{1}{c}{\bf Binary} & \multicolumn{1}{c}{\bf Categorical} & \multicolumn{1}{c}{\bf Regression}
                        \\ \hline 
                        \multirow{2}{*}{Supervised}         & Label prediction* \citep{chen2018differentially,abay2018privacy, frigerio2019differentially} & Yes & Yes & Yes\\
                        &Predictive model ranking* \citep{jordon2018pate} & Yes & Yes & Yes \\ \hline
                        \multirow{1}{\linewidth}{Unsupervised, prediction-based } 
                        &Dimension-wise prediction plot* & Yes (\cite{choi2017generating}, ours) & \textbf{Yes} & \textbf{Yes} \\ \hline
\multirow{6}{\linewidth}{Unsupervised, distribution-based } 
                        &Dimension-wise probability plot \citep{choi2017generating} & Yes & No & No \\
                        &\(2,3\)-way feature marginal, total variation distance \citep{abay2018privacy,NIST2018Match3} ** & Yes & Yes  & Yes  \\
                        &\textbf{\(1\)-way feature marginal, diversity measure (\(\mu\)-smooth KL divergence) **}  & \textbf{Yes} & \textbf{Yes} 
& \textbf{Yes} 
\\
\hline
\multirow{4}{\linewidth}{{Unsupervised, qualitative}} 
                        &{\(1\)-way feature marginal (histogram) (e.g. in \cite{lin2019generating,bagdasaryan2019differential} and in the implementation of \cite{frigerio2019differentially})}& Yes & {Yes}  & {Yes}  \\
                        &\textbf{\(2\)-way PCA marginal (data visualization)} 
& \textbf{Yes} & \textbf{Yes}  & \textbf{Yes}  \\
                \end{tabular}
        \end{center}
\end{table*}

Various evaluation metrics have been considered in the literature to evaluate the quality of the synthetic data (see \citet{charest2011can} for a survey). The metrics can be broadly categorized into two groups: \emph{supervised} and \emph{unsupervised}.  Supervised evaluation metrics are used when there are clear distinctions between features and labels of the dataset, e.g., for healthcare applications, a person's disease status is a natural label. In these settings, a predictive model is typically trained on the synthetic data, and its accuracy is measured with respect to the real (test) dataset. Unsupervised evaluation metrics are used when no feature of the data can be decisively termed as a label. Recently proposed metrics include \emph{dimension-wise probability} for binary data \citep{choi2017generating}, which compares the marginal distribution of real and synthetic data on each individual feature, and \emph{dimension-wise prediction}, which measures how closely synthetic data captures relationships between features in the real data. This metric was proposed for binary data, and we extend it here to mixed-type data. Recently, \citet{NIST2018Match3} used a 3-way marginal evaluation metric which used three random features of the real and synthetic datasets to compute the total variation distance as a statistical score. 

\paragraph{Supervised Evaluation Metrics.} The main aim of generating synthetic data in a supervised setting is to best understand the relationship between features and labels. A popular metric for such cases is to train a machine learning model on the synthetic data and report its accuracy on the real test data \citep{xie2018differentially}. \citet{zhang2018differentially} used inception scores on the image data with classification tasks. Inception scores were proposed in \citet{salimans2016improved} for images which measure quality as well as diversity of the generated samples. Another metric used in \citet{jordon2018pate} reports whether the accuracy ranking of different machine learning models trained on the real data is preserved when the same machine learning model is trained on the synthetic data. 
Although these metrics are used for classification in the literature, they can be easily generalized to the regression setting.

In the DP setting of synthetic data generation, supervised metrics also differ from unsupervised in that the label feature is sometimes treated as public (e.g. in DP-SYN \cite{abay2018privacy}), whereas in unsupervised setting, all features are treated as private. We note it as this may create a slight difference in privacy accounting.


\paragraph{Unsupervised Evaluation Metrics, Prediction-Based.} 
Rather than measuring accuracy by predicting one particular feature as in supervised-setting, one can predict \textit{every} individual feature using the rest of features. The prediction score is therefore created for each single feature, creating a list of dimension- (or feature-) wise prediction scores.  Good synthetic data should have similar dimension-wise prediction scores to that of the real data. Intuitively, similar dimension-wise prediction shows that synthetic data correctly captures inter-feature relationships in the real data. 

One metric of this type is proposed by \citet{choi2017generating} for binary data. Although it was originally proposed for binary data, we extend this to mixed-type data by allowing varieties of predictive models appropriate for each data type present in the dataset. For each feature, we try predictive models on the real dataset in order of increasing complexity until a good accuracy score is achieved. For example, to predict a real-valued feature, we first used a linear classifier and then a neural network predictor. This ensures that a choice of predictive model is appropriate to the feature. Synthetic data is then evaluated by measuring the accuracy of the same predictive model (trained on the real data) on the synthetic data. Similarly high accuracy scores on synthetic data and real data indicates that the synthetic data closely approximates the real data.

 \citet{zhang2018differentially} provides an unsupervised Jensen-Shannon score metric which measures the Jensen-Shannon divergence between the output of a discriminating neural network on the real and synthetic datasets, and a Bernoulli random variable with $0.5$ probability. This metric differs from dimension-wise prediction in that the predictive model (discriminator) is trained over the whole dataset at once, rather than dimension-wise, to obtain a score.

\paragraph{Diversity Metric.} Inception score is one common metric for evaluating the quality of data generated by GAN \cite{borji2019pros}. Both inception and Jensen-Shannon scores aim to capture both the accuracy and diversity of generated data through comparing the distributions of predictions by a fixed classifier on original and synthetic data. Inception score is similar to \(\mu\)-smoothed KL divergence we propose in Section \ref{s.exp}, but we apply it to discrete distribution and use a smoothing to avoid divergence being undefined. Our metric also differs from inception scores in that it is based on the distributions of synthetic and original data, and not on predictions on those datasets by any classifier. We observed that introducing a classifer can itself be a reason for lack of diversity, and concern that a predictive model in general can introduce bias and unfairness in other forms. For example, we found that in a categorical feature with one strong majority class, the classifier predicts only the majority to maximize a standard notion of "accuracy," hence making a synthetic data that ignore minority classes represent the original data perfectly well, as predictions  on synthetic and original data are identical. Therefore, for diversity applications, we prefer distribution-based metric to a distribution-based metric.

Moreover, we aim our metric to be appropriate in differential privacy setting. A natural metric to penalize missing a minority class is KL divergence, as used in the definitions of inception and Jensen-Shannon scores. However, it is impossible for a private model to recognize if a minority exists if the class is really small, simply due to the definition of differential privacy (unless the algorithm assumes existence of all possible classes in the dataset, but this would greatly impact accuracy as the number of classes increase). \citet{bagdasaryan2019differential}  observed a  phenomena that differentially private training indeed impacts minority classes more than majority class, as we also observed in our work. Missing a minority class, therefore, is sometimes unavoidable with DP guarantees. Since missing any class makes KL divergence undefined, we added a smoothing term to KL divergence so that the penality of missing a minority class is finite, yet significant.      

\paragraph{Unsupervised Evaluation Metrics, Distributional-Based.}

One way to evaluate the quality of synthetic data is computing a dimension-wise probability distribution, which was also proposed in \citet{choi2017generating} for binary data. This metric compares the marginal distribution of real and synthetic data on each individual feature.  Below we survey other metrics in this class that can extend to mixed-type data.

\emph{3-Way Marginal}: Recently, the \citet{NIST2018Match3} challenge used a 3-way marginal evaluation metric in which three random features of the real and synthetic data \(R,S\) are used to compute the total variation distance as a statistical score. This process is repeated a few times and finally, average score is returned. 
In particular, values for each of the three features are partitioned in 100 disjoint bins as follows: 
\[B^i_{R,k} = \adfloor*{\frac{(R^i_k -R_{k, \min})*100}{R_{k,\max} -R_{k, \min}}}\] and  
\[B^i_{S, k} = \adfloor*{\frac{(S^i_k-R_{k,\min})*100}{R_{k,\max}-R_{k, \min}} },\]
where $R^i_k, S^i_k$ is the value of $i$-th datapoint's $k$-th feature in datasets $R$ and $S$, and $R_{k, \min}, R_{k,\max}$ are respectively the minimum and maximum value of the $k$-th feature in $R$. For example, if $k =1,2,3$ are the  selected features then $i$-th data points of $R$ and $S$ are put into bins identified by a 3-tuple, $(B^i_{R,1}, B^i_{R,2}, B^i_{R,3})$ and $(B^i_{S,1}, B^i_{S,2}, B^i_{S,3})$, respectively. 

Let $\BB_R, \BB_S$ be the set of all 3-tuple bins in datasets $R$ and $S$, and let $\abs{B}$ denote number of datapoints in 3-tuple bin $B$, normalized by total number of data points. Then, the 3-way marginal metric reports the $\ell_1$-norm of the bin-wise difference of $\BB_R$ and $\BB_S$ as follows:
\begin{align*}
& \sum_{B_1 \in \BB_R} \sum_{B_2 \in \BB_S}\Ibb_{\{B_1\in \BB_S\}}\Ibb_{\{B_2 = B_1\}} \abs[\big]{\abs{B_1} -\abs{B_2}} \\
 &\quad+ \sum_{B_1 \in \BB_R} (1-\Ibb_{\{B_1\in \BB_S\}})\abs{B_1} + \sum_{B_2 \in \BB_S} (1-\Ibb_{\{B_2\in \BB_R\}})\abs{B_2}.
 \end{align*}

Both aforementioned metrics (dimension-wise probability from \cite{choi2017generating} and 3-way marginal from \cite{NIST2018Match3}) involve two steps. First, a projection (or a selection of features) of data is specified, and second some statistical distance or visualization of synthetic and real data in the projected space is computed. Dimension-wise probability for binary data corresponds to projecting data into each single dimension, and visualizing synthetic and real distributions in projected space by histograms (for binary data, the histogram can be specified by one single number: probability of the feature being 1). The 3-way marginal metric first selects a three-dimensional space specified by three features as a space into which data projected, discretizes the synthetic and real distributions on that space, then computes a total variation distance between discretized distributions. We can generalize these two steps process and conceptually design  a new metric as follows.

\emph{Generalization of Data Projection:} One can generalize selection of \(3\) features (3-way marginal) to any \(k\) features (\(k\)-way marginal). However, one can also select \(k\) \textit{principle components} instead of \(k\) features. We distinguish these as \(k\)-way \textit{feature} marginal (projection onto a space spanned by feature dimensions) and \(k\)-way \textit{PCA} marginal (projection onto a space spanned by principle components of the original dataset). Intuitively, \(k\)-way PCA marginal best compresses the information of the real data into a small \(k\)-dimensional space, and hence is a better candidate for comparing projected distributions.

\emph{Generalization of Distributional Distance:} Total variation distance can be misleading as it does not encode any information on the distance between the supports of two distributions. In general, one can define any metric of choice (optionally with discretization) on two projected distributions, such as Wasserstein distance which also depends on the distance between the supports of the two distributions.

\emph{Computing Distributional Distance:} The distance between two distributions can also be computed without any data projections. Computing an exact statistical score on high-dimensional datasets is likely computationally hard. However, one can, for example, subsample uniformly at random points from two distributions to compute the score more efficiently, then average this distance over many iterations.

\paragraph{Unsupervised Evaluation Metrics, Qualitative.}
As described above, dimension-wise probability is a specific application of comparing histograms under binary data. One can plot histograms of each feature (1-way feature marginal) for inspection. In practice, histogram visualization is particularly helpful when a feature is strongly skewed, sparse (majority zero), and/or hard to predict well by predictive models. An example of this occurred when predictive models do not have meaningful predictive accuracy on certain features of the ADULT dataset, making prediction-based metric inappropriate.  Instead, inspection of histograms of those features on synthetic and real data (as in Figure \ref{fig:2-way-pca-adult}) indicate that synthetic data replicates those features well.

In addition, \(2\)-way PCA marginal is a visual representation of data that explains as much variance as possible in a plane, providing a good trade-off between information and ease of visualization on two datasets. This visualization can be augmented with a distributional distance of choice over the two distributions on these two spaces to get a quantitative metric.

\subsubsection{Background on Evaluation Metrics Used in Experiments}\label{app.expmetrics}

Here, we discuss in more technical details the evaluation metrics that we use in the experiments  in Section \ref{s.exp} and in our paper to empirically measure the quality of the synthetic data. Some of these metrics have been used in the literature, while 2-way PCA is novel  in this work. Another novel metric \(\mu\)-smooth KL divergence is described in Section  \ref{s.exp}.  

For the following two metrics, the dataset should be partitioned into a training set $R \in \Rbb^{m_1 \times n}$ and testing set $T \in \Rbb^{m_2 \times n}$, where $m = m_1+m_2$ is the total number of samples the real data, and $n$ is the number of features in the data. After training, the generative model creates a synthetic dataset $S \in \Rbb^{m_3\times n}$ for sufficiently large $m_3$.



\textbf{Dimension-Wise Probability.} When the feature is binary,  we compares the proportion of $1$'s (which can be thought of as estimators of Bernoulli success probability) in each feature of the training set $R$ and synthetic dataset $S$, i.e.  the marginal distribution of each feature. For each feature, the closer \ the proportion of 1's in the original dataset is to that of synthetic dataset, the better. 

\textbf{Dimension-Wise Prediction.} This metric evaluates whether synthetic data maintain relationships \emph{between} features. For the $k$-th feature of training set $R$ and synthetic dataset $S$, we choose $y_{R_k} \in \Rbb^{m_1}$ and $y_{S_k} \in \Rbb^{m_2}$ as labels of a classification or regression task based on the type of that feature, and the remaining features $R_{-k}$ and $S_{-k}$ are used for prediction. We train either a classification or regression model on $R_{-k}$ and $S_{-k}$, and measure goodness of fit based on the model's accuracy by testing on \(T\). That is, we "train on synthetic, test on original" to evaluate the quality of synthetic data. The closer of accuracy scores obtained from original and synthetic data, the better.

Model accuracy can be reported using AUROC, $F_1$, or $R^2$ scores, as appropriate.  
We describe the model's accuracy as follows: 
\begin{enumerate}
\item Area under the ROC curve (AUROC) score and $F_1$ score for classification: The $F_1$ score of a classifier is defined as
$F_1 := \tfrac{2\times \text{precision} \times \text{recall}}{\text{precision} + \text{recall}},$
where precision is ratio of true positives to true and false positives, and recall is ratio of true positives to total true positives (i.e., true positives plus false negatives). \(F_1\) score on multi-class features are averaged using micro-averaging. AUROC score is a graphical measure capturing the area under ROC (receiver operating characteristic) curve, and is only intended for binary data.  Both metrics take values in interval $[0,1]$ with larger values implying good fit. The ROC curve are pairs of true and false positive rates obtained from setting different thresholds at the classifier's predicted probability. Note that when the classifier is trained on the data with one class and predicts always with probability 0 (or 1), ROC curve is a single pair, and AUROC is thus undefined. 
\item $R^2$ score for regression: The $R^2$ score is defined as
$1 - \tfrac{\textstyle{\sum} (y_i-\wh{y}_i)^2}{\textstyle{\sum} (y_i - \wb{y})^2} ,$
where $y_i$ is the true label, $\wh{y}_i$ is the predicted label, and $\wb{y}$ is the mean of the true labels. This is a popular metric used to measure goodness of fit as well as future prediction accuracy for regression.
\end{enumerate}





\textbf{1-Way Feature Marginal (Histogram).} We compute probability distribution of the feature of interest of both real and synthetic data. For continuous features, we partition the range into intervals.  This can be extended to $k$-way feature marginals by considering joint distribution over \(k\) features and made into a quantitative measure by adding a distance measure between the histograms.

We  also propose the following novel qualitative evaluation metric.

\textbf{2-Way PCA Marginal.} This metric generalizes the 3-way marginal score used in \citet{NIST2018Match3}. In particular, we compute principle components of the original data and evaluate a projection operator for first two principle components. Denote $P\in \Rbb^{n\times 2}$  the projection matrix such that $\wb{R} = RP$ is the projection on first two principle components of $R$. After we fix \(P\), we  project synthetic data $\wb{S} = SP$ and scatterplot 2-D points in $\wb{R}$ and $\wb{S}$ for visual evaluation. That is, we train PCA from the original dataset, and use the same projection from this PCA on (possibly many) synthetic datasets.

\section{Additional Details of Experiments on \mimic data}\label{app.mimictrain}
In all our experiments, all \(\eps\) values reported that are rounded are rounded up to guarantee the validity of privacy guarantee. MIMIC-III data set contains 46,520 data points in total, and is partitioned into train, validation, and test data sets of sizes 27912, 9304, 9304 (60\%, 20\%, 20\%), respectively. Privacy analysis is calculated using the training size. Data are stored in 0/1 format. \dpgan pre-trained models in this paper are available at \url{https://github.com/DPautoGAN/DPautoGAN/tree/master/results/pre-trained%20models}. 

\paragraph{\dpgan Computing Infrastructure.}
\dpgan are run on GCP: n1-highmem-2 (2 vCPUs, 13 GB memory) with 1 x NVIDIA Tesla K80. The  training of  autoencoder (for 15,000 iterations) and of GAN (for 20,000 iterations) each takes about 1.5 hours. The combined training together with performance evaluations (probability and prediction plots) are done in less than 4 hours for each setting of parameter.

\paragraph{\dpgan Training.} The autoencoder was trained via Adam with Beta 1 = 0.9, Beta 2 = 0.999, and a learning rate of 0.001. It was trained on minibatches of size 100 and microbatches of size 1. L2 clipping norm was selected to be the median L2 norm observed in a non-private training loop, set to 0.8157. The noise multiplier was then calibrated to achieve the desired privacy guarantee.

The GAN was composed of two neural networks, the generator and the discriminator. The generator was a simple feed-forward neural network, trained via RMSProp with alpha = 0.99 with a learning rate of 0.001. The discriminator was also a simple feed-forward neural network, also trained via RMSProp with the same parameters, with minibatches of size 128. The L2 clipping norm of the discriminator was set to 0.35. The pair was trained on minibatches of size 1,000 and a microbatch size of 1, with 2 updates to the discriminator per 1 update to the generator. Again, the noise multiplier was then calibrated to achieve desired privacy guarantees.

\paragraph{Selecting the Noise Multipliers and the Numbers of Iterations.} Noise  multipliers are finally set at \(\psi=3.5,2.3,1.3\) simultaneously to both autoencoder and GAN to achieve \(\eps=0.81,1.33,2.70\) respectively. Training is first done for 20000 iterations, and the generated data every 1000 iterations are saved. We then inspect whether an earlier trained model may be used as follows. For \(\psi=2.3,1.3\), the number of features where the model outputs all zero converges to 181 out of 1071 features and stabilize at 181 for the remaining of training. We picked the second saved model which has 181 such features. For \(\psi=3.5\), the number of such features quickly drops to 181-182 and then fluctuates between 181-182 in the remaining of the training. We pick the third saved model which has the number of such features being 181 or 182. The final iterations picked is \(T=6000,7000,7000\) for \(\psi=3.5,2.3,1.3\), respectively, and these numbers of iterations are then used to calculate the privacy parameters.         

\paragraph{Model Architecture.} A serialization of the (non-private, i.e. \(\eps=\infty\)) model architectures used in the experiment can be found below. For the private version, we change the latent dimension from 128 to 64. 

(encoder): Sequential(\\
(0): Linear(in-feature=1071, out-feature=128, bias=True)\\
(1): Tanh()\\
)\\
(decoder): Sequential(\\
(0): Linear(in-feature=128, out-feature=1071, bias=True)\\
(1): Sigmoid()\\
)\\

Generator(\\
(model): Sequential(\\
(0): Linear(in-feature=128, out-feature=128)\\
(1): LeakyReLU(negative-slope=0.2)\\
(2): Linear(in-feature=128, out-feature=128)\\
(3): Tanh()\\
)\\
)

Discriminator(\\
(model): Sequential(\\
(0): Linear(in-feature=1071, out-feature=256, bias=True)\\
(1): LeakyReLU(negative-slope=0.2)\\
(2): Linear(in-feature=256, out-feature=1, bias=True)\
)\\
)


\subsection{Dimension-Wise Prediction on Sparse Features}\label{app.dwpsparse}

\mimic dataset contains several sparse features, i.e. features with small number of 1's. In fact, we found that 146 features do not have any 1's, and 706 more features have proportion of 1's less than 1\% in the original dataset. The presence of sparse features is a challenge for prediction-based evaluation since the classifier accuracy is unstable on sparse features. Moreover, the prediction score is unmeaningful when the train or test dataset has only one class present, such as AUROC being undefined when one class is presented in the test data. Even when AUROC is defined, the sparsity of 1's in dataset makes a classifier unable to learn anything and give a score of 0.5, or perform worse than a random classifer and give a score below 0.5, which is arguably unmeaningful. 

In our experiment, after an 80\%/20\% split into the training and testing datasets, AUROC scores are not defined on those 146 features with no 1's in the original dataset.
Of the rest 925 features, 35 and 69 of those have  AUROC prediction scores exactly at and below  0.5, respectively, on the original dataset. All 104 of those with unmeaningful AUROC have less than 0.2\% proportion of 1's in the original dataset. Figure \ref{fig:dim_pred_mimic_full} shows dimension-wise  prediction scatterplots of full 925 of 1071 features where AUROC is defined.

\begin{figure}[h]
        \centering
        \begin{subfigure}{0.245\textwidth}
        \includegraphics[width = \linewidth]{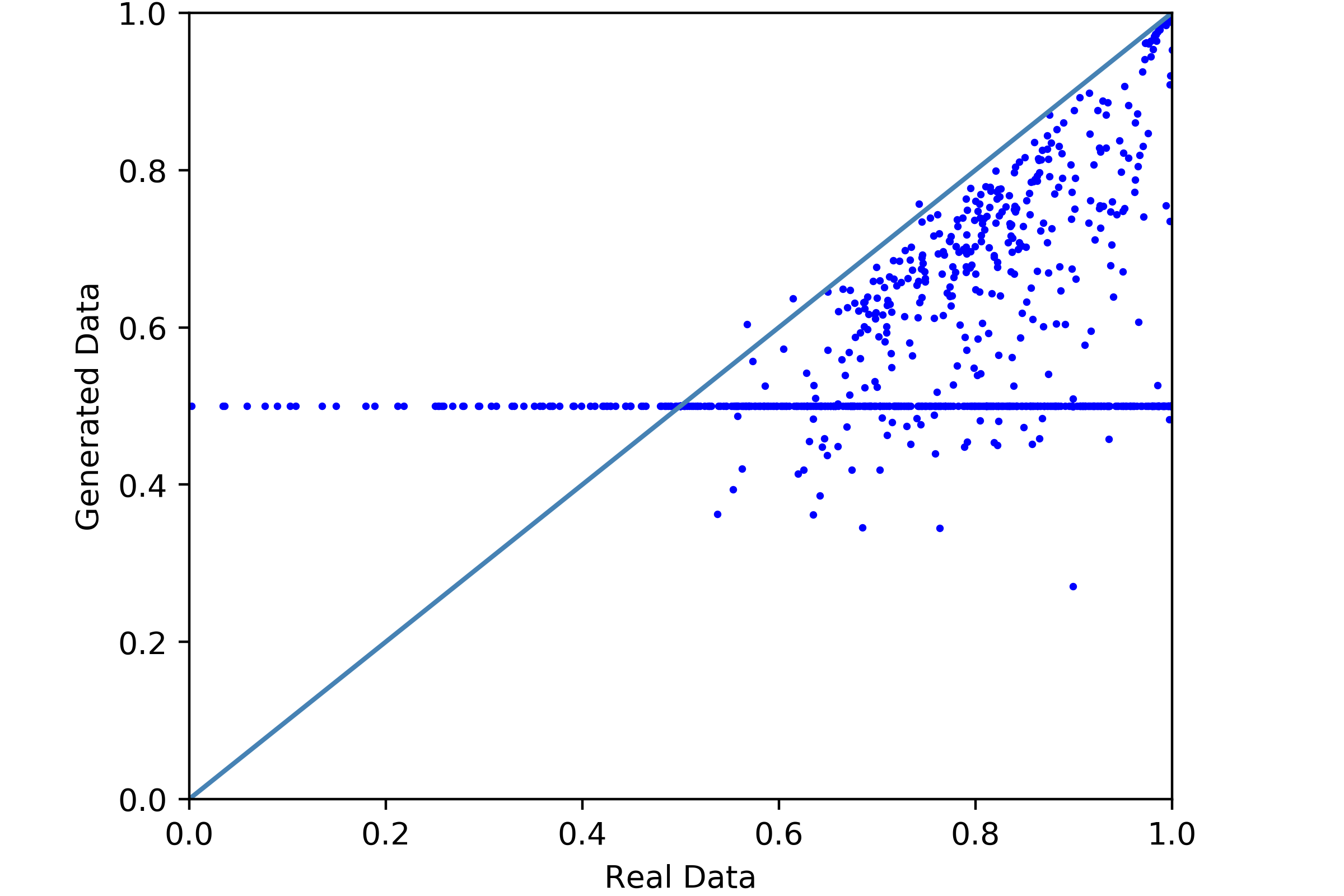}
        \caption{$\eps = \infty$}
        \end{subfigure}
        \begin{subfigure}{0.245\textwidth}
        \includegraphics[width = \linewidth]{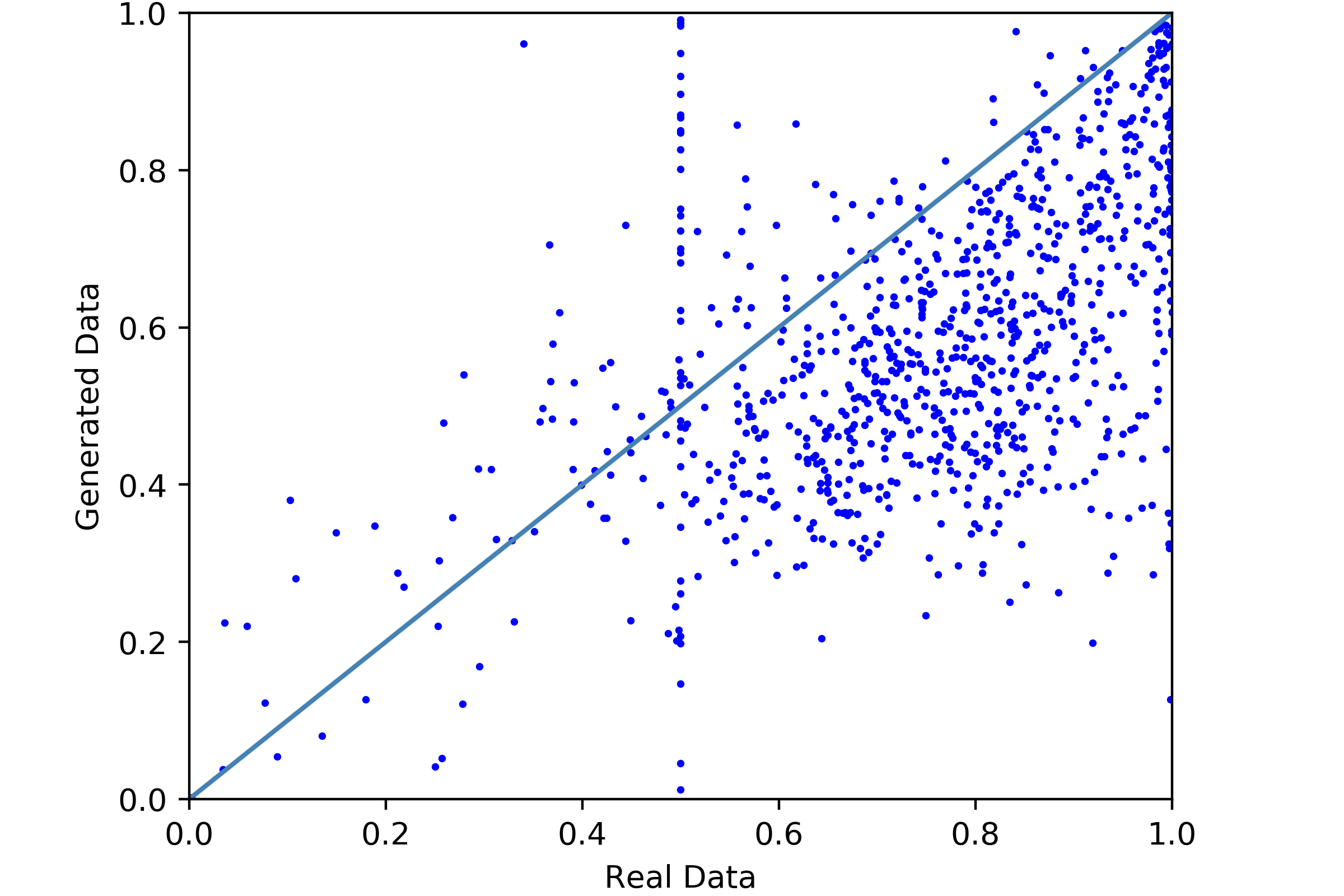}
        \caption{$\eps = 2.70$}
        \end{subfigure}
        \begin{subfigure}{0.245\textwidth}
        \includegraphics[width = \linewidth]{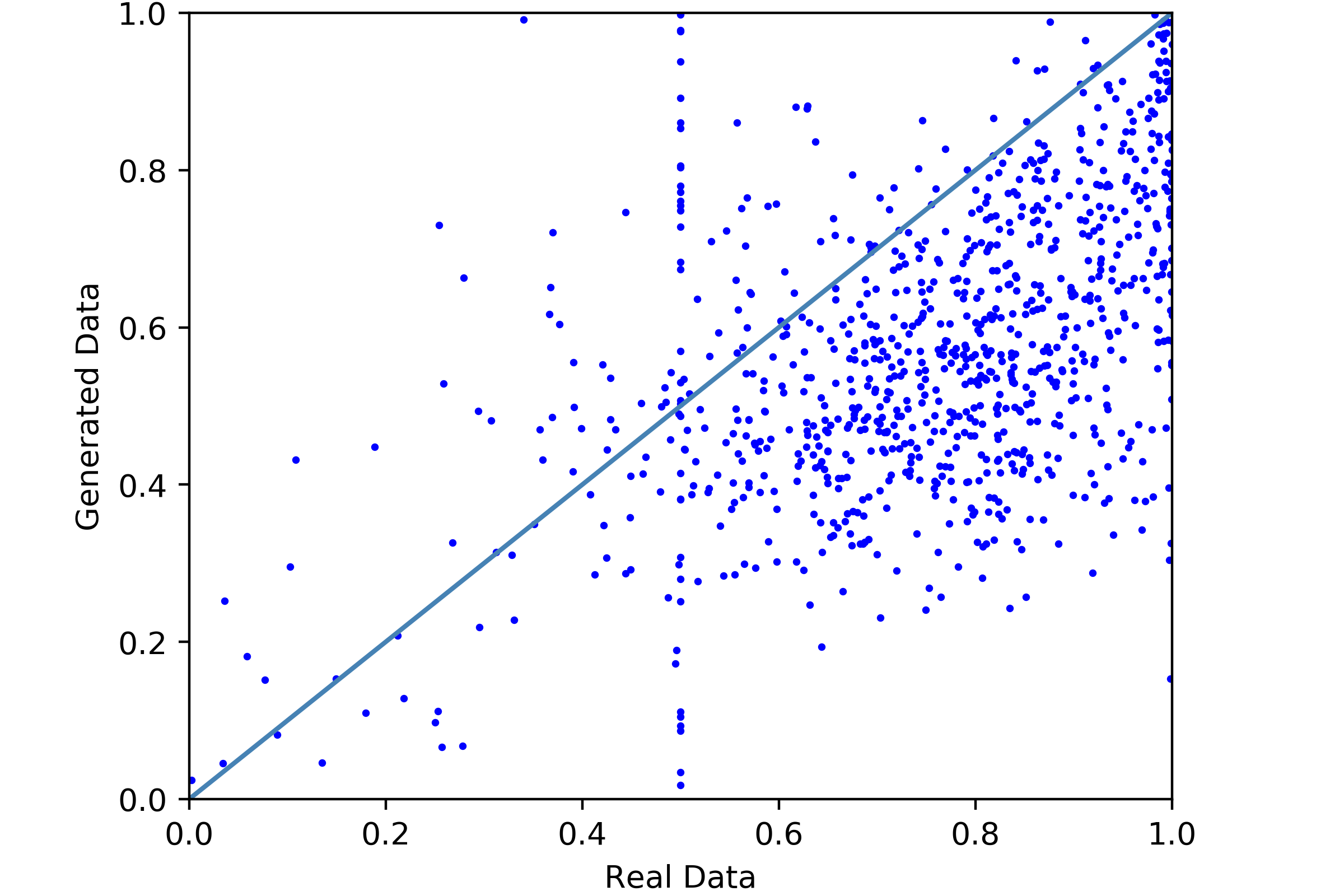}
        \caption{$\eps = 1.33$}
        \end{subfigure}
        \begin{subfigure}{0.245\textwidth}
        \includegraphics[width = \linewidth]{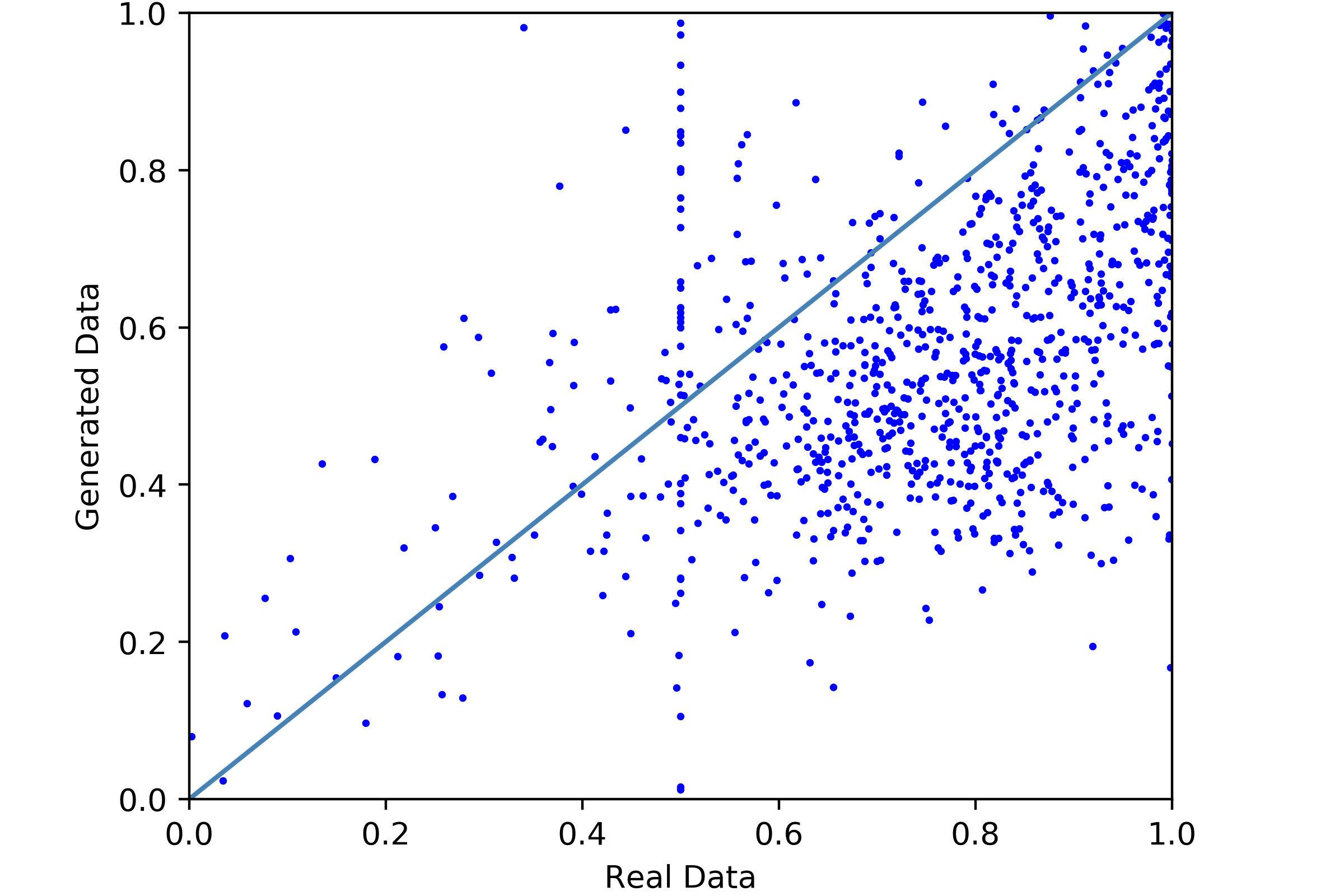}
        \caption{$\eps = 0.81$}
        \end{subfigure}     
                \caption{Dimension-wise  prediction scatterplots of 925 of 1071 features which AUROC prediction scores are defined on the original dataset. AUROC is not defined when the test set of the original data has only one class on that feature. Prediction scores are by a logistic regression classifier trained on the original binary dataset \mimic and on synthetic datasets generated by \dpgan at different privacy parameters $\eps$.                   
                }
                \label{fig:dim_pred_mimic_full}
\end{figure}

 When \(\eps=\infty\), \dpgan generates all 0's in many of the sparse features in the synthetic dataset, including  all 146 features with no 1's in the original datasets. Those features obtain scores of 0.5, giving the horizontal line \(y=0.5\) in Figure \ref{fig:dim_pred_mimic_full} (a). When noise is injected, \dpgan generates few but enough of 1's in all features that AUROC is not 0.5, and thus the horizontal line is no longer present in Figure \ref{fig:dim_pred_mimic_full} (b)-(d). The vertical line \(x=0.5\) represents 69  features that the classifier is unable to learn in the orignal dataset due to the sparsity and learn from random noise injected in the synthetic data generator.

\cite{xie2018differentially} presented dimension-wise prediction by deleting features which synthetic dataset contain no 1's. While this conveniently deletes points on the horizontal and vertical lines, the features being deleted are dependent on the synthetic dataset, which can obscure the presentation of its quality. For example, if the generator performs poorly on non-sparse features by outputting only one class, that feature is not on the plot rather than being presented as far away from the line \(y=x\). Since we observe that prediction scores on sparse features are unstable, thus necessarily  showing large variances in the context of differential privacy, we propose to delete features whose  proportions of 1's in the original dataset are below a threshold. The set of deleted feature is therefore fixed and independent of synthetic data to be evaluated. In Figure \ref{fig:dim_pred_mimic_cut_1}, we show dimension-wise  prediction scatterplots of \mimic dataset with the thresholds set at 1\%, 2.5\%, and 5\%.  We are able to more clearly see the performance of \dpgan than plotting all 925 features and better observe a slight change of performance as \(\eps\) decreases across three \(\eps\) values.               

\begin{figure}[h]
        \centering
        \begin{subfigure}{0.245\textwidth}
        \includegraphics[width = \linewidth]{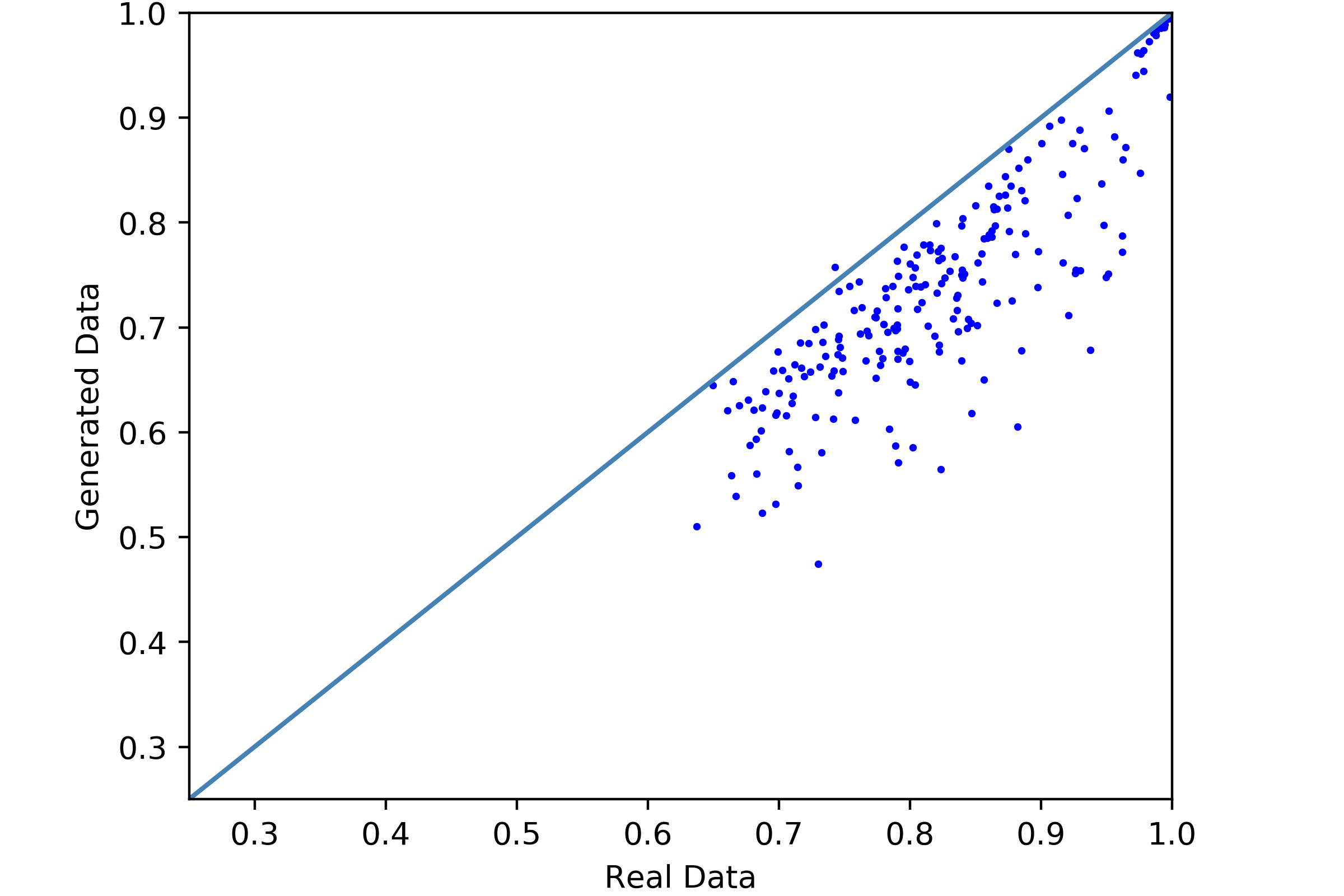}
        \caption{Features with \(\geq1.0\)\% of 1's, $\eps = \infty$}
        \end{subfigure}        
        \begin{subfigure}{0.245\textwidth}
        \includegraphics[width = \linewidth]{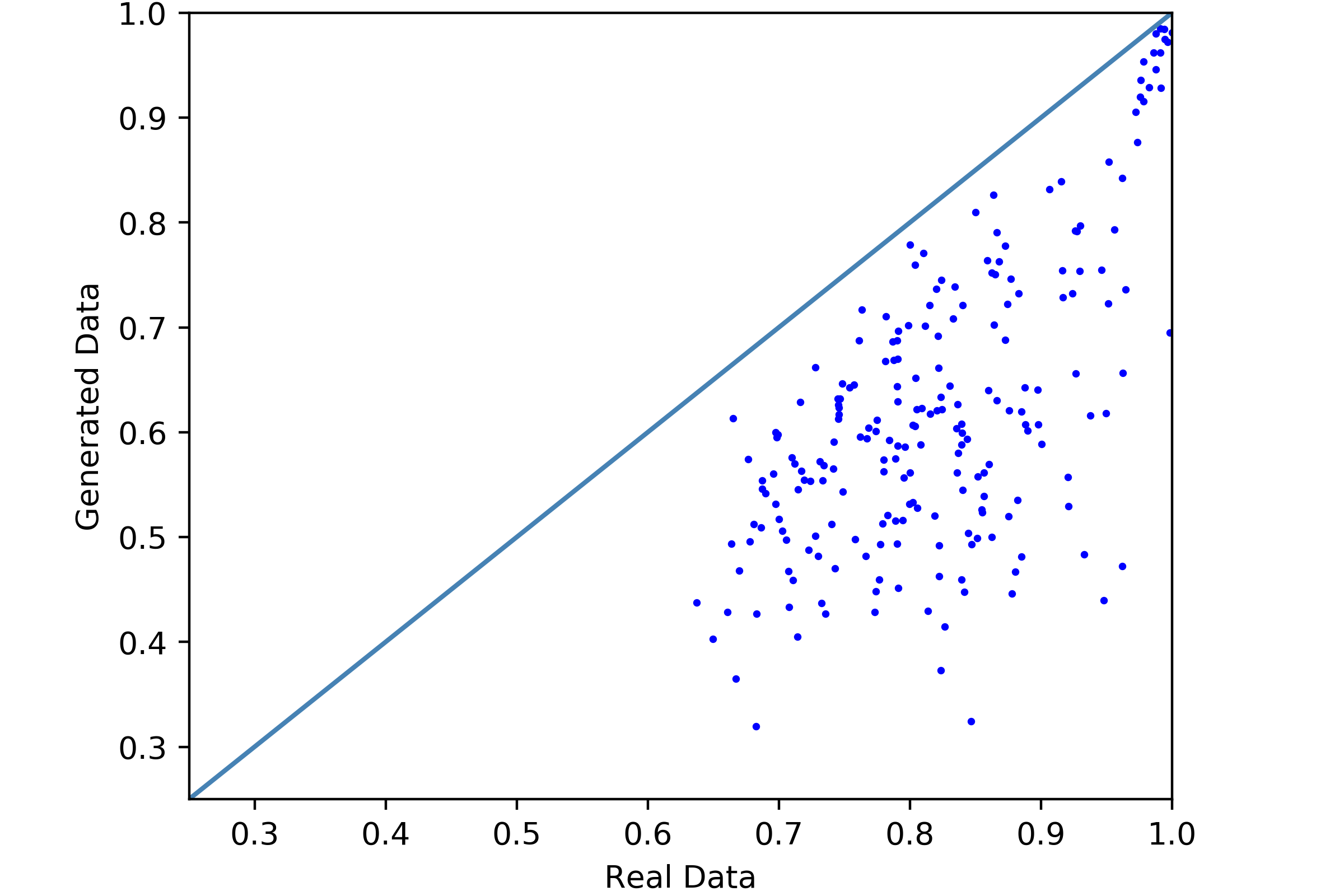}
        \caption{$\eps = 2.70$}
        \end{subfigure}        
        \begin{subfigure}{0.245\textwidth}
        \includegraphics[width = \linewidth]{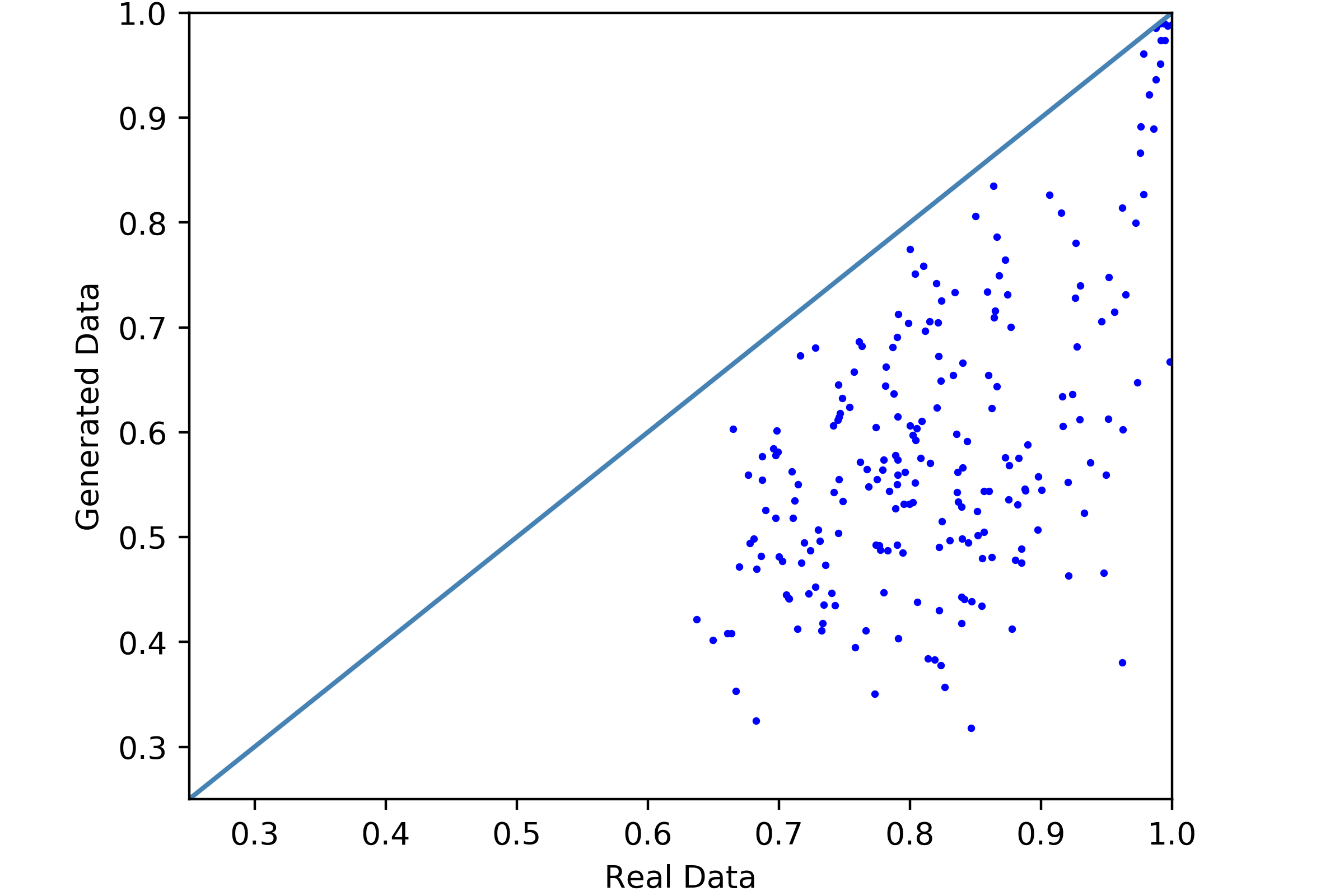}
        \caption{$\eps = 1.33$}
        \end{subfigure}
        \begin{subfigure}{0.245\textwidth}
        \includegraphics[width = \linewidth]{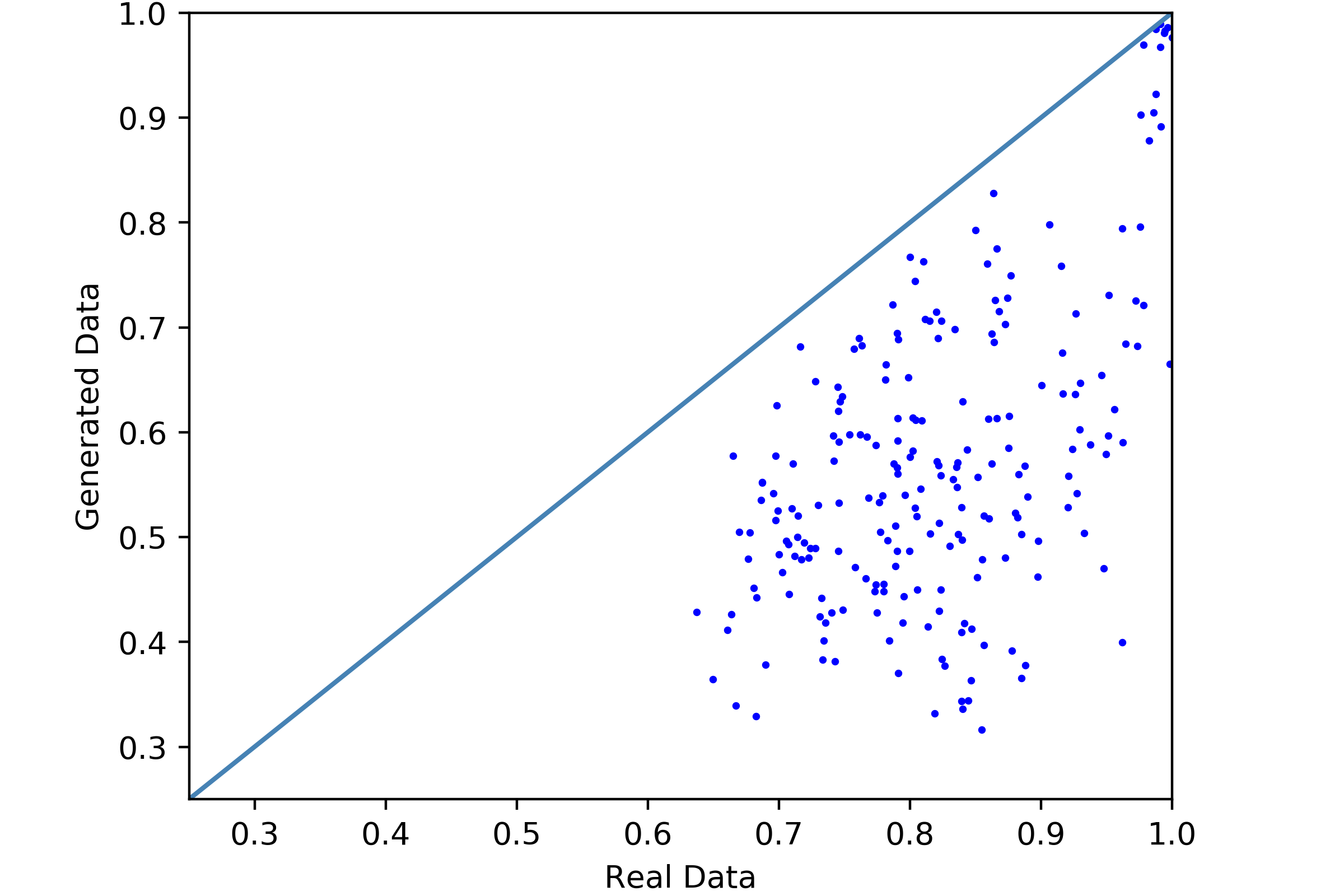}
        \caption{$\eps = 0.81$}
        \end{subfigure}
        \begin{subfigure}{0.245\textwidth}
        \includegraphics[width = \linewidth]{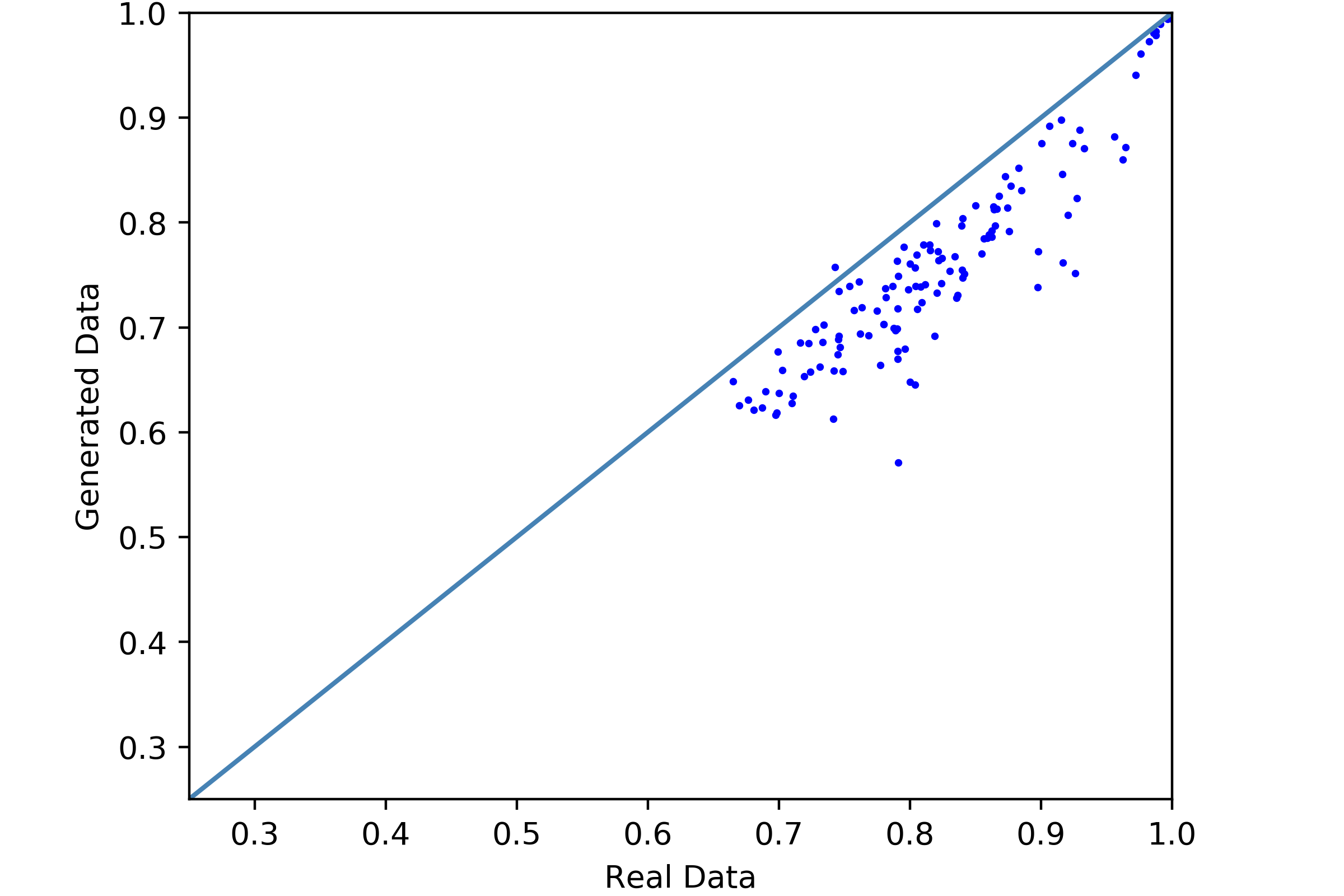}
        \caption{Features with \(\geq2.5\)\% of 1's, $\eps = \infty$}
        \end{subfigure}        
        \begin{subfigure}{0.245\textwidth}
        \includegraphics[width = \linewidth]{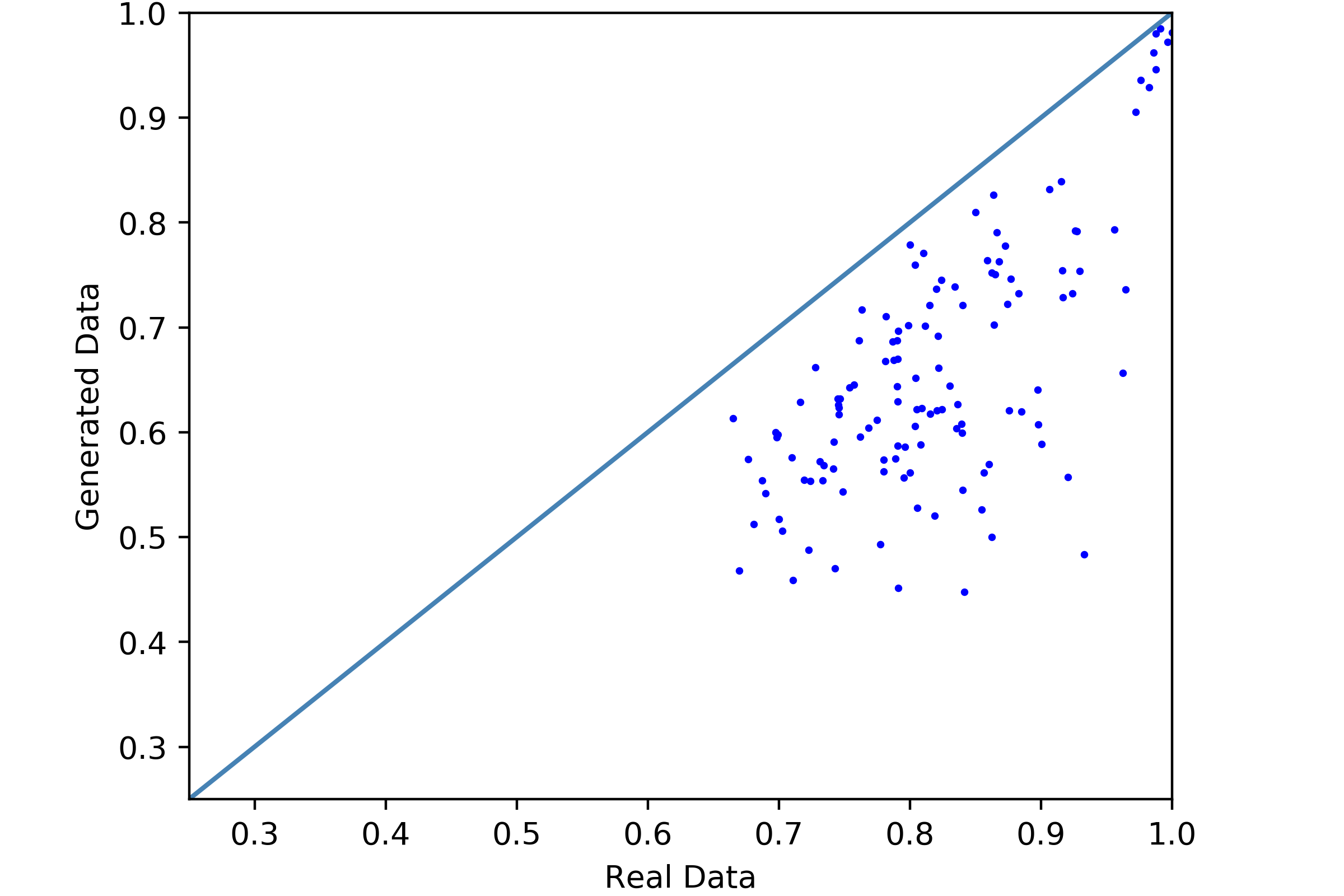}
        \caption{$\eps = 2.70$}
        \end{subfigure}        
        \begin{subfigure}{0.245\textwidth}
        \includegraphics[width = \linewidth]{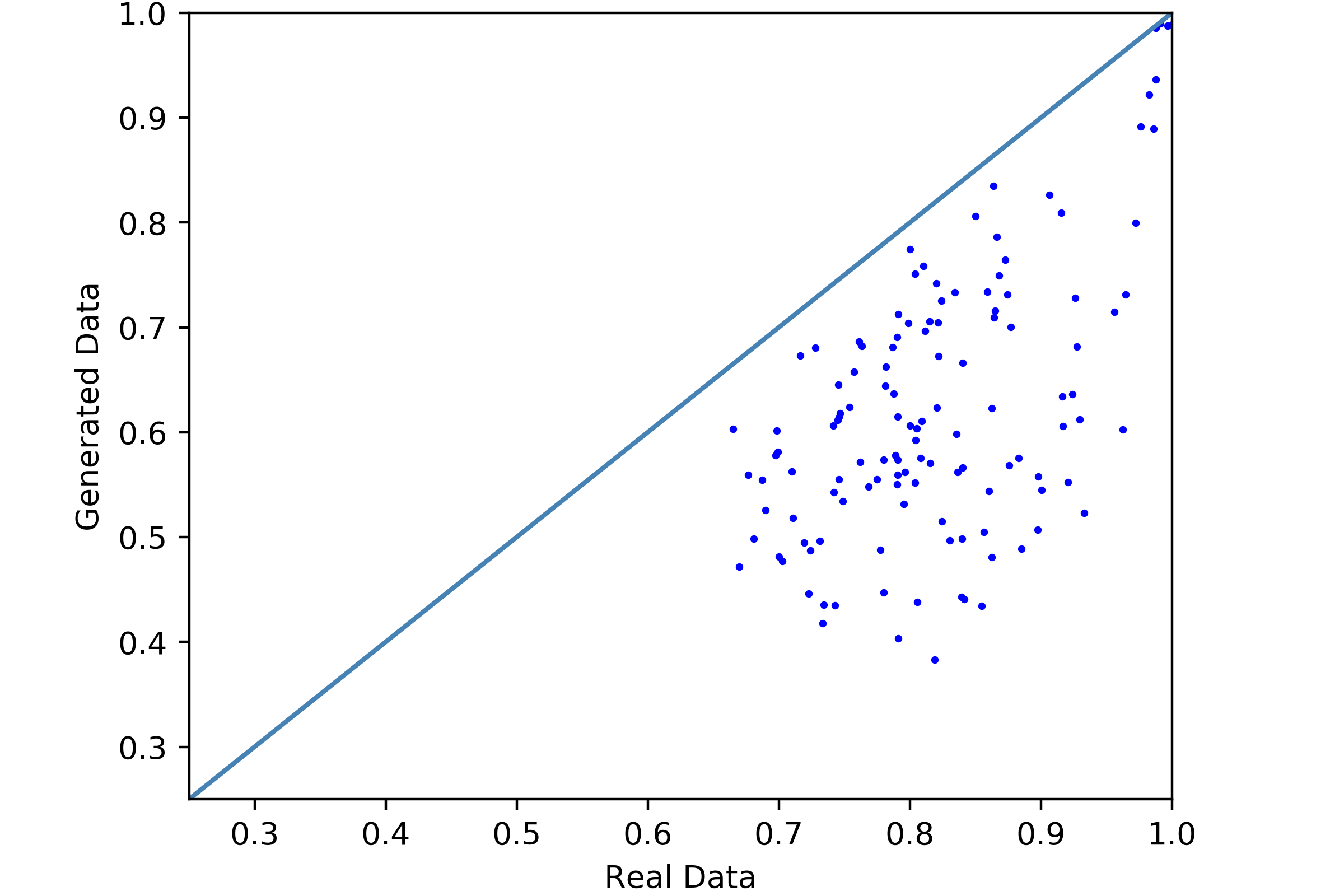}
        \caption{$\eps = 1.33$}
        \end{subfigure}
        \begin{subfigure}{0.245\textwidth}
        \includegraphics[width = \linewidth]{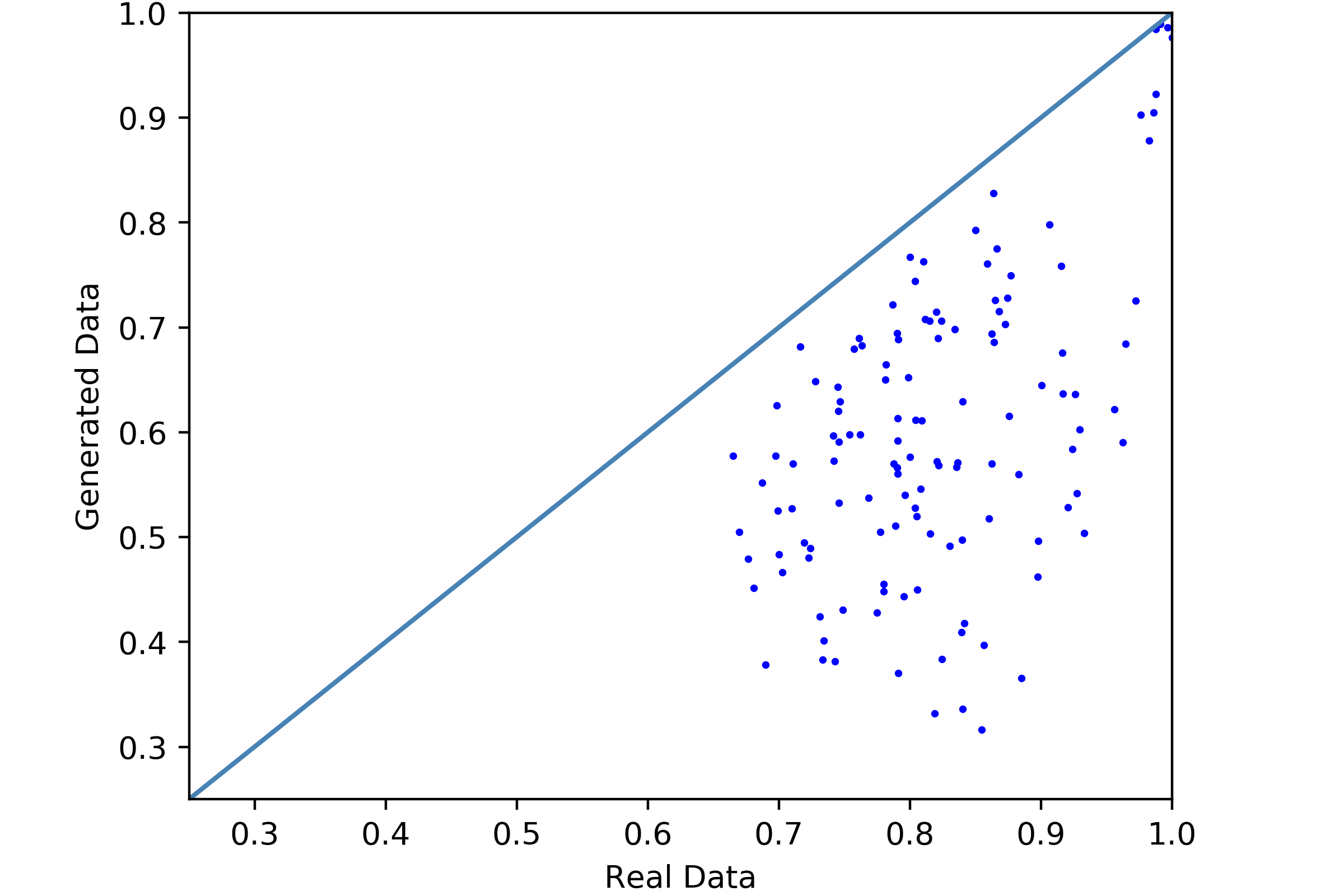}
        \caption{$\eps = 0.81$}
        \end{subfigure}
        \begin{subfigure}{0.245\textwidth}
        \includegraphics[width = \linewidth]{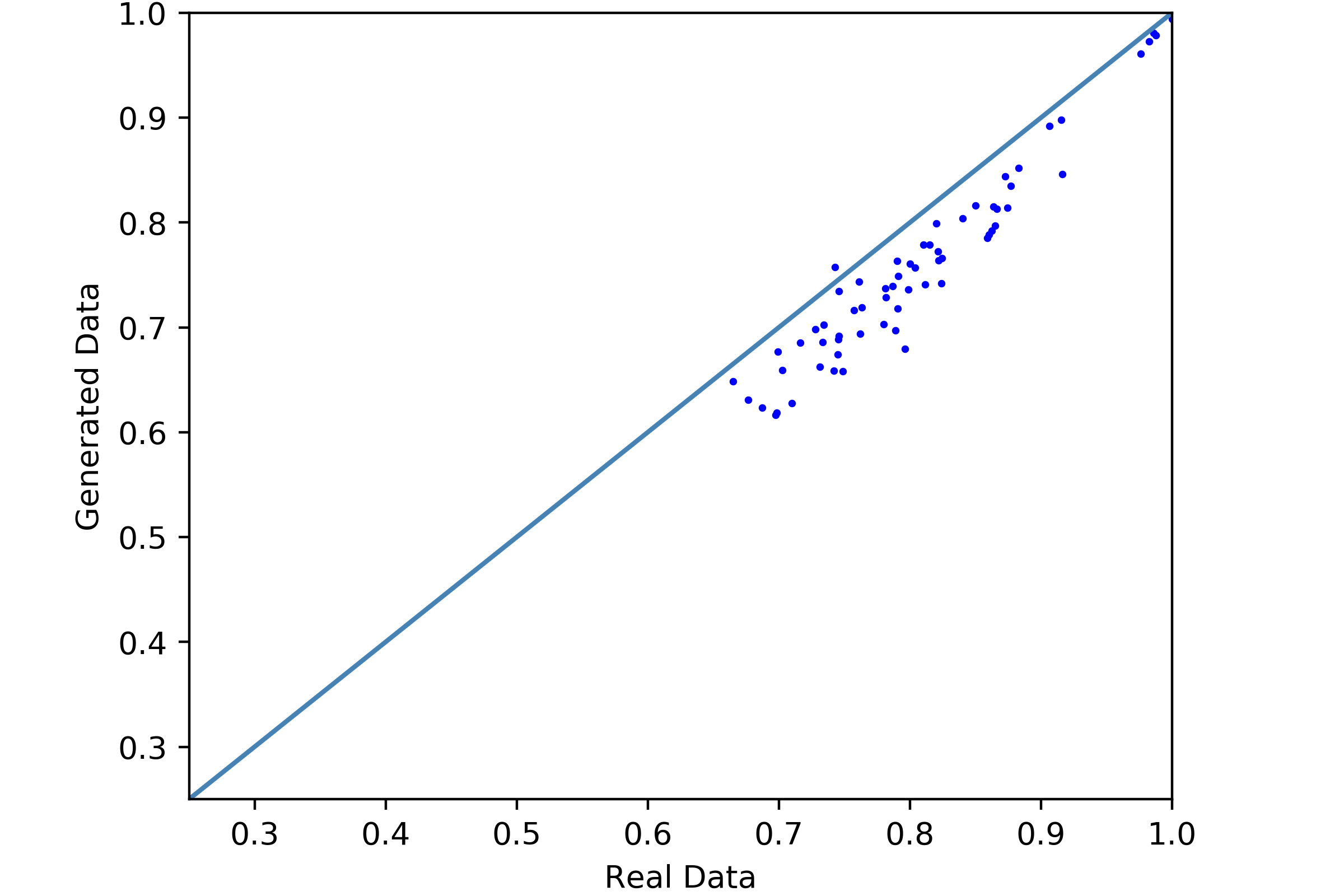}
        \caption{Features with \(\geq5.0\)\% of 1's, $\eps = \infty$}
        \end{subfigure}        
        \begin{subfigure}{0.245\textwidth}
        \includegraphics[width = \linewidth]{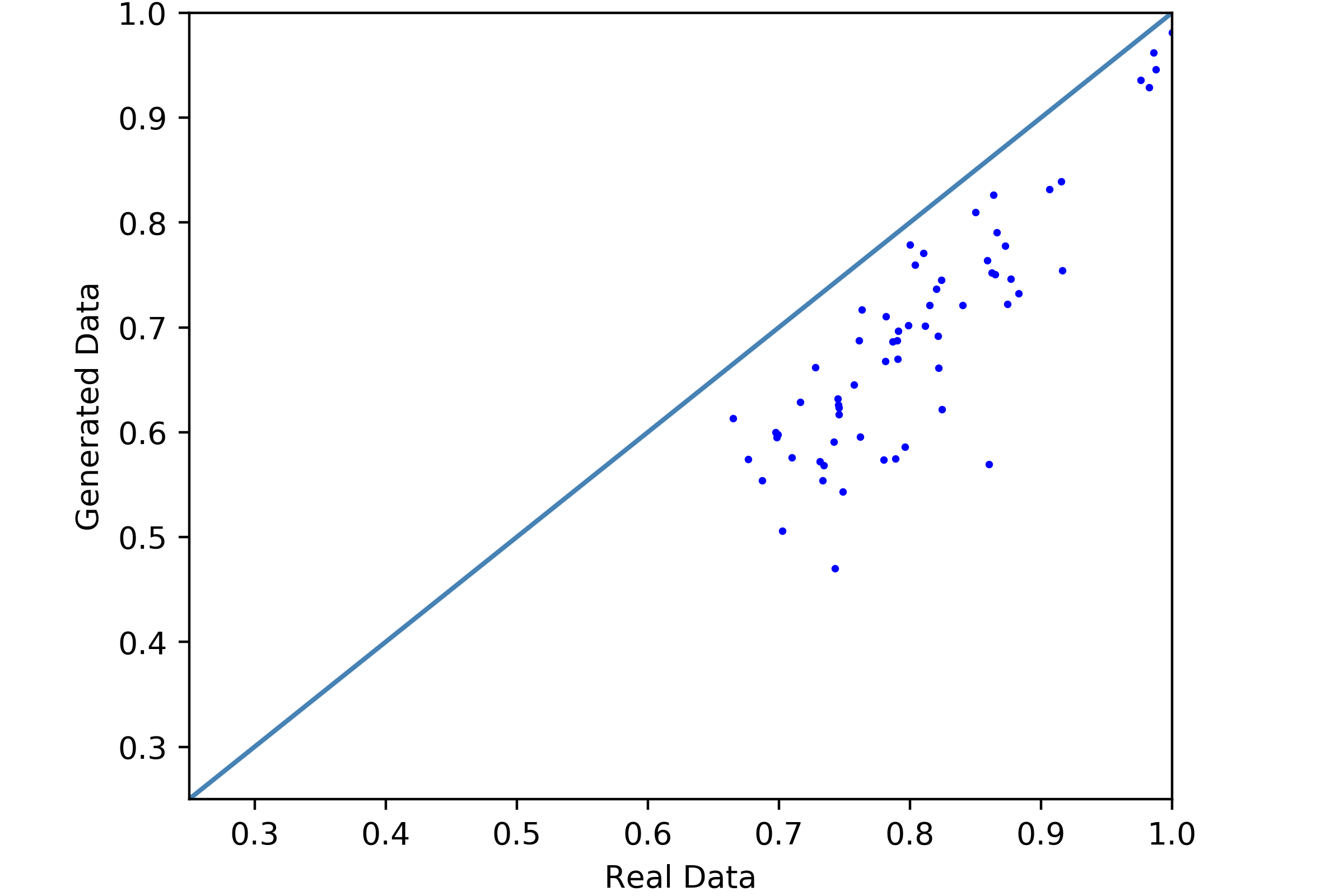}
        \caption{$\eps = 2.70$}
        \end{subfigure}        
        \begin{subfigure}{0.245\textwidth}
        \includegraphics[width = \linewidth]{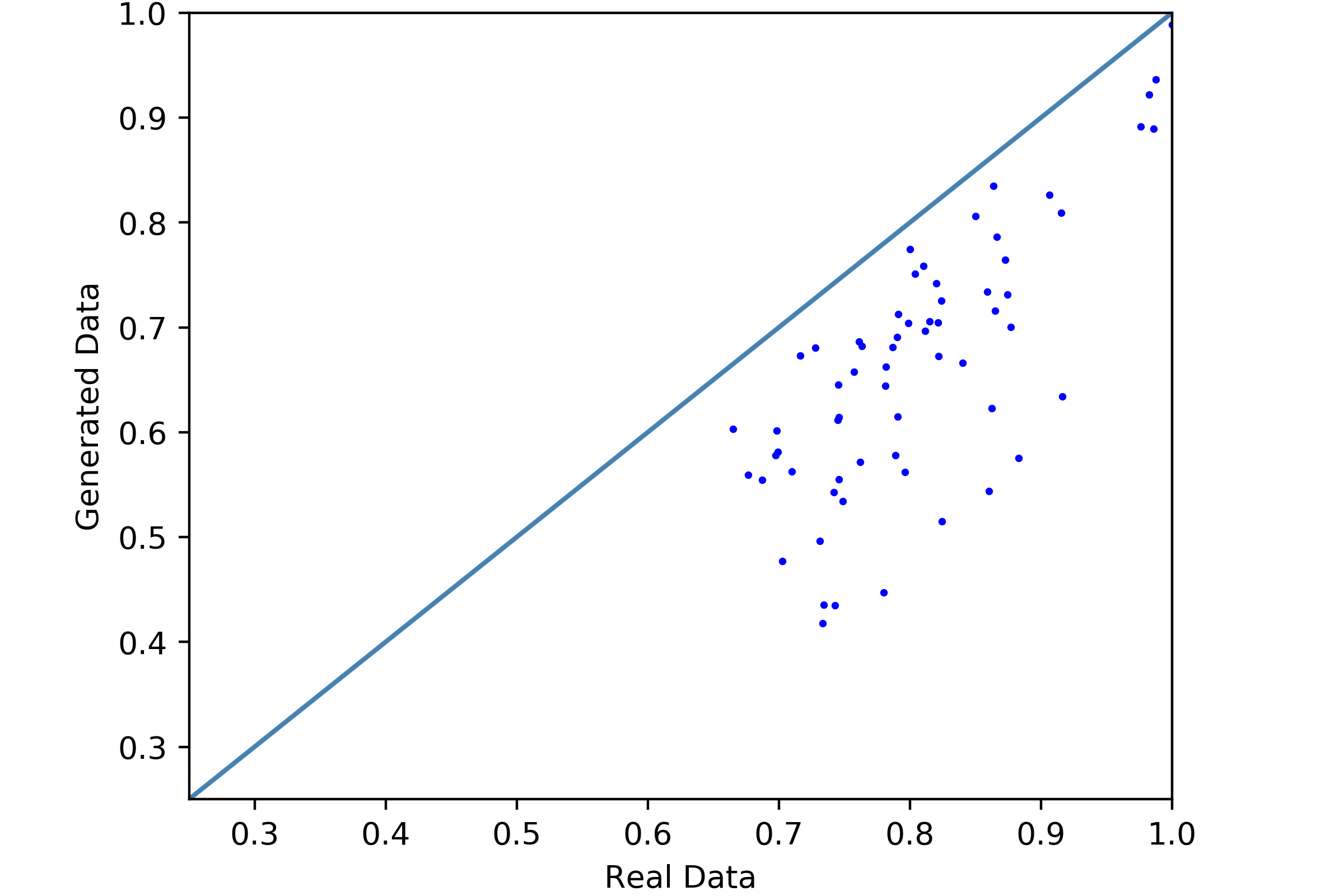}
        \caption{$\eps = 1.33$}
        \end{subfigure}
        \begin{subfigure}{0.245\textwidth}
        \includegraphics[width = \linewidth]{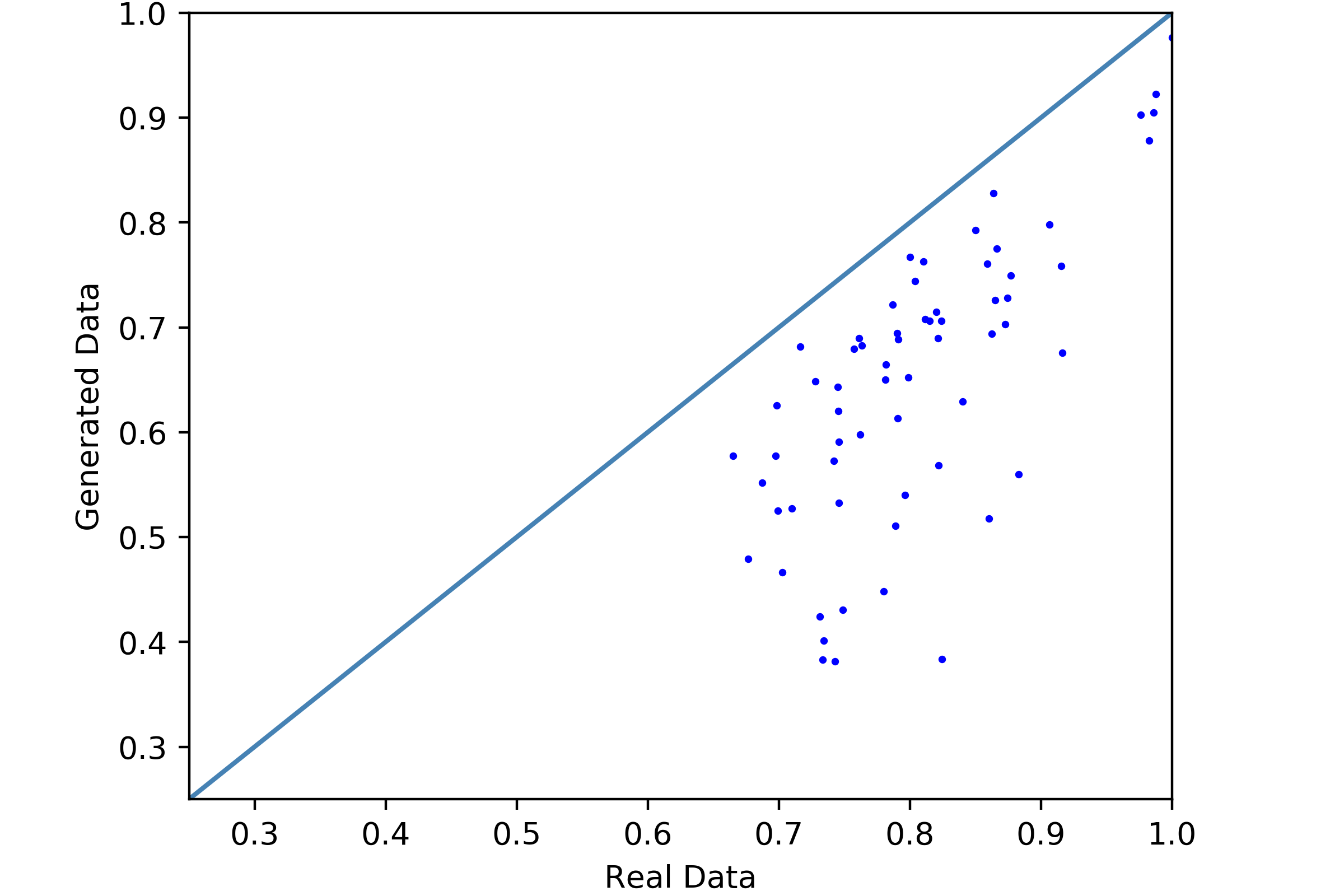}
        \caption{$\eps = 0.81$}
        \end{subfigure}
                \caption{Dimension-wise  prediction scatterplots of 219 of 1071 features with at least 1\% of 1's,  127 features with at least 2.5\% of 1's, and 64 features with at least 5\% of 1's in the original dataset. Prediction scores are by a logistic regression classifier trained on the original binary dataset \mimic and on synthetic datasets generated by \dpgan at different privacy parameters $\eps$.                   
                }
                \label{fig:dim_pred_mimic_cut_1}
\end{figure}

Tables of all probability and prediction scores on each features and additional dimension-wise prediction plots on features with different thresholds are available at \url{https://github.com/DPautoGAN/DPautoGAN/tree/master/results/prediction-plots/DP-auto-GAN%20MIMIC-III}. 

\section{Additional Details of Experiments on ADULT data}\label{app.adulttrain}

In this section, we describe experimental details of \dpgan, DP-WGAN, DP-VAE, and DP-SYN for reproducibility.

For \dpgan and in original implementation of existing algorithms, data are preprocessed by one-hot encoding categorical features, and by max-min scalar on continuous features, i.e. mapping maximum to 1 and minimum to 0. While maximum and minimum are technically leaking privacy, and hence not public as assumed in our framework, they are sometimes treated as publicly available such as in synthetic data challenge \cite{NIST2018Match3} and assumed in implementation of existing works. In other usage, reasonable cap can be assumed on features, or a standard differentially private query on minimum and maximum of a feature on a bounded range can be used.


Synthetic data are generated from each trained generative model to a size of 32561, the training size (which is two thirds) of ADULT data. (Results of dimension-wise prediction we observed are similar with synthetic data of full ADULT size 48842).

In parameter tuning of pre-existing methods from original authors, we keep the whole framework including preprocessing of ADULT data, architecture, optimizer, and other hyper-parameters, except the noise multiplier. We attempted to keep original architectures as they are likely optimized in original work for generating synthetic ADULT data. We tune the noise multiplier on several higher values to achieve a smaller \(\eps\) values needed, and pick the est performing model across noise multiplier values used. Details for each algorithm can be found in the remainder of this section.
 
\paragraph{Computing Infrastructure and Runtime.}
\dpgan are run on GCP: n1-highmem-2 (2 vCPUs, 13 GB memory) with 1 x NVIDIA Tesla K80. The combined training on  autoencoder and GAN are done in approximately 2-3 hours for each setting of parameter. DP-WGAN, DP-VAE, DP-SYN are run on a personal computer with processor Intel(R) Core(TM) i7-6600U CPU @ 2.60GHz 2.81 GHz and RAM 16.0 GB. Training of DP-WGAN, DP-VAE, DP-SYN for each parameter setting finishes in  between approximately15 minutes to 2 hours. Any of all evaluation metrics to an ADULT synthetic dataset finishes in less than 1-2 minutes. 

\subsection{\dpgan Training}
\paragraph{Preprocessing.} Original ADULT dataset contains 15 features, one of which is a positive integer feature named ``fnlwgt'' (final weight). This feature is discarded as unrelated to each individual person in the census, but rather the additional feature US census created by mapping a person to an estimated weight  of another demographic dataset. Two features ``education'' and ``education-num'' are the same feature representing in a different format -- string or a positive integer. In particular, education consists of 16 levels of education, and education-num represents them as numbers 1,2,...,16. Hence, we remove one of these two features and treat this column as one categorical feature. In the end, we have 9 categorical features, one of which is a binary label named ``salary'', and four continuous features.

ADULT data consists of 48842 datapoints, partitioned into 32561 (two-thirds)\ for training and 16281 (one-third)\  for testing. We follow the same partitioning by training our \dpgan on 32561 samples and holding the rest only for synthetic data evaluation.

\paragraph{Training.}The autoencoder was trained via Adam with Beta 1 = 0.9, Beta 2 = 0.999, and a learning rate of 0.005 for 10,000 minibatches of size 64 and a microbatch size of 1. The L2 clipping norm was selected to be the median L2 norm observed in a non-private training loop, equal to 0.012. The noise multiplier was then calibrated to achieve the desired privacy guarantee. The final noise multiplier used for \(\eps=0.36,0.51,1.01\) are \(\psi=5,2.5,1.5\), respectively.

The GAN was composed of two neural networks, the generator and the discriminator. The generator used a ResNet architecture, adding the output of each block to the output of the following block. It was trained via RMSProp with alpha = 0.99 with a learning rate of 0.005. The discriminator was a simple feed-forward neural network with LeakyReLU hidden activation functions, also trained via RMSProp with alpha = 0.99. The L2 clipping norm of the discriminator was set to 0.022. The pair was trained on 15,000 minibatches of size 128 and a microbatch size of 1, with 15 updates to the discriminator per 1 update to the generator. Again, the noise multiplier was then calibrated to achieve the desired privacy guarantee. The final noise multiplier used for \(\eps=0.36,0.51,1.01\) are \(\psi=8,7.5,3.5\), respectively.

\paragraph{Model Architecture.}A serialization of the model architectures used in the experiment can be found below. Note that the number of latent dimension, 64, is the same as in the  implementation of DP-SYN.

Autoencoder(\\
(encoder): Sequential(\\
0: Linear(in-features=106, out-feature=60, bias=True)\\
(1): LeakyReLU(negative-slope=0.2)\\
(2): Linear(in-feature=60, out-feature=15, bias=True)\\
(3): LeakyReLU(negative-slope=0.2)\\
)\\
(decoder): Sequential(\\
(0): Linear(in-feature=15, out-feature=60, bias=True)\\
(1): LeakyReLU(negative-slope=0.2)\\
(2): Linear(in-feature=60, out-feature=106, bias=True)\\
(3): Sigmoid()\\
)\\
)

Generator(\\
(block-0): Sequential(\\
(0): Linear(in-feature=64, out-feature=64, bias=False)\\
(1): BatchNorm1d()\\
(2): LeakyReLU(negative-slope=0.2)\\
)\\
(block-1): Sequential(\\
(0): Linear(in-feature=64, out-feature=64, bias=False)\\
(1): BatchNorm1d()\\
(2): LeakyReLU(negative-slope=0.2)\\
)\\
(block-2): Sequential(\\
(0): Linear(in-feature=64, out-feature=15, bias=False)\\
(1): BatchNorm1d()\\
(2): LeakyReLU(negative-slope=0.2)\\
)\\
)

Discriminator(\\
(model): Sequential(\\
(0): Linear(in-feature=106, out-feature=70, bias=True)\\
(1): LeakyReLU(negative-slope=0.2)\\
(2): Linear(in-feature=70, out-feature=35, bias=True)\\
(3): LeakyReLU(negative-slope=0.2)\\
(4): Linear(in-feature=35, out-feature=1, bias=True)\
)\\
)
\subsection{DP-WGAN Training}\label{app.dpwgan}
\paragraph{Preprocessing.} The algorithm of WGAN \cite{frigerio2019differentially} (their implementation can be found at \url{https://github.com/SAP-samples/security-research-differentially-private-generative-models}) is used and implemented only for discrete data. The preprocessing automatically delete continuous columns. Hence, DP-WGAN preprocesses ADULT data into 9 categorical features, one of which is the binary ``salary'' label. Each categorical feature is then one-hot encoded before feeding into DP-WGAN training.  

\paragraph{Parameter Tuning.} The original implementation uses noise multiplier \(\psi=7\), originally for a higher values of epsilons as mentioned in \cite{frigerio2019differentially}. We found that the training cannot achieve \(\eps=0.51\) just after one epoch, while the typical training requires multiple (tens to almost a hundred) epochs. Hence, we train with higher noise parameters \(\psi=7,9,11,13,15,19,23,27.5,35\). For \(\eps=0.36\), we also attempted \(\psi=40,45,50,60,70,80,100\) which still gives \(\eps>0.36\) only after one epoch. However, if we relax to \(\eps=0.37\), we are able to train a few epochs at high noise level to achieve the privacy guarantee, so we allow \(\eps=0.37\).

For each of the noise parameter, we select the generated data from the epoch before privacy busget is exhausted. The synthetic data is evaluated by dimension-wise prediction. We exclude a few cases where prediction score of the salary feature is zero, indicating possibly a mode collapse. Then, we pick the model for each \(\epsilon\) setting across different noise multiplier with highest total prediction score. We note that overall performance of DP-WGAN are comparable across \(\psi\in\{7,9,11,13\}\) for \(\eps\geq0.8\) and less predictable for smaller \(\eps\). The final \(\psi\) we use for \(\eps=0.36,0.51,1.01\) are \(\psi=27.5,19,9\), respectively. Dimension-wise prediction across different \(\epsilon\) and \(\psi\) are available at \url{https://github.com/DPautoGAN/DPautoGAN/tree/master/results/prediction-plots/DP-WGAN}.

\subsection{DP-VAE Training}\label{app.dpvae}
\paragraph{Preprocessing.} An implementation of DP-VAE (not by the  author of original work \cite{acs2018differentially}) can be found at  \url{https://github.com/SAP-samples/security-research-differentially-private-generative-models/blob/master/Tutorial_dp-VAE.ipynb} for training on ADULT data.
We slightly modify the size of the training set to 32561 sample points, as  used in original dataset and our  \dpgan training. Though the original implementation uses all 15 features, we preprocess the data exactly the same way as \dpgan preprocessing for a fairer reporting. 

We, however, observed that dimension-wise prediction is either of similar overall quality when using original 15 dimensions instead of what is reported. Dimension-wise predictions of 13 and 15 features on several values of \(\eps\) is available at \url{https://github.com/DPautoGAN/DPautoGAN/tree/master/results/prediction-plots/DP-VAE/vae-13-vs-15-features}.     
    \paragraph{Parameter Tuning.} The tutorial defaults noise multiplier at \(\psi=1\). To get smaller \(\eps\), we test \(\psi=1,1.5,2,2.5,3,3.5,4,4.5,5\) and additionally \(\psi=5.5,6,6.5,7,7.5,8\) for \(\eps=0.36\). Standard validation accuracy score of VAE in the training process from the keras package are used. We observed an expected pattern that accuracy increases from very small \(\psi\) until it drops again at some high \(\psi\) value, where the peak of \(\psi\) is larger for smaller \(\eps\). As a result, we extended \(\psi\) as mentioned for smaller \(\eps\) to be certain that we have reached such peak. Then, the model from noise multiplier which gives highest accuracy score is used. The final \(\psi\) used for \(\eps=0.36,0.51,1.01\) are \(\psi=5,4,2\), respectively.  

\subsection{DP-SYN Training}\label{app.dpsyn}
\paragraph{Preprocessing.} The original implementation of DP-SYN 
deletes ``fnlwgt''
and a redundant ``education-num'' as in our preprocessing. However, it group some similar educational levels into one category, resulting in 8 categories rather than 16. It preprocesses capital-gain and capital-loss into categorical features with 3 classes: ``low'', ``medium'', and ``high.'' Hence, the final preprocessed data has 2 continuous  and 11 categorical features. Because one categorical feature, education, is preprocessed differently from other algorithms, it is excluded from reporting the sum of diversity divergence scores in Table \ref{tab:diversity}.

\begin{figure}[h]
        \centering
        \begin{subfigure}{0.32\textwidth}
        \includegraphics[width = \linewidth]{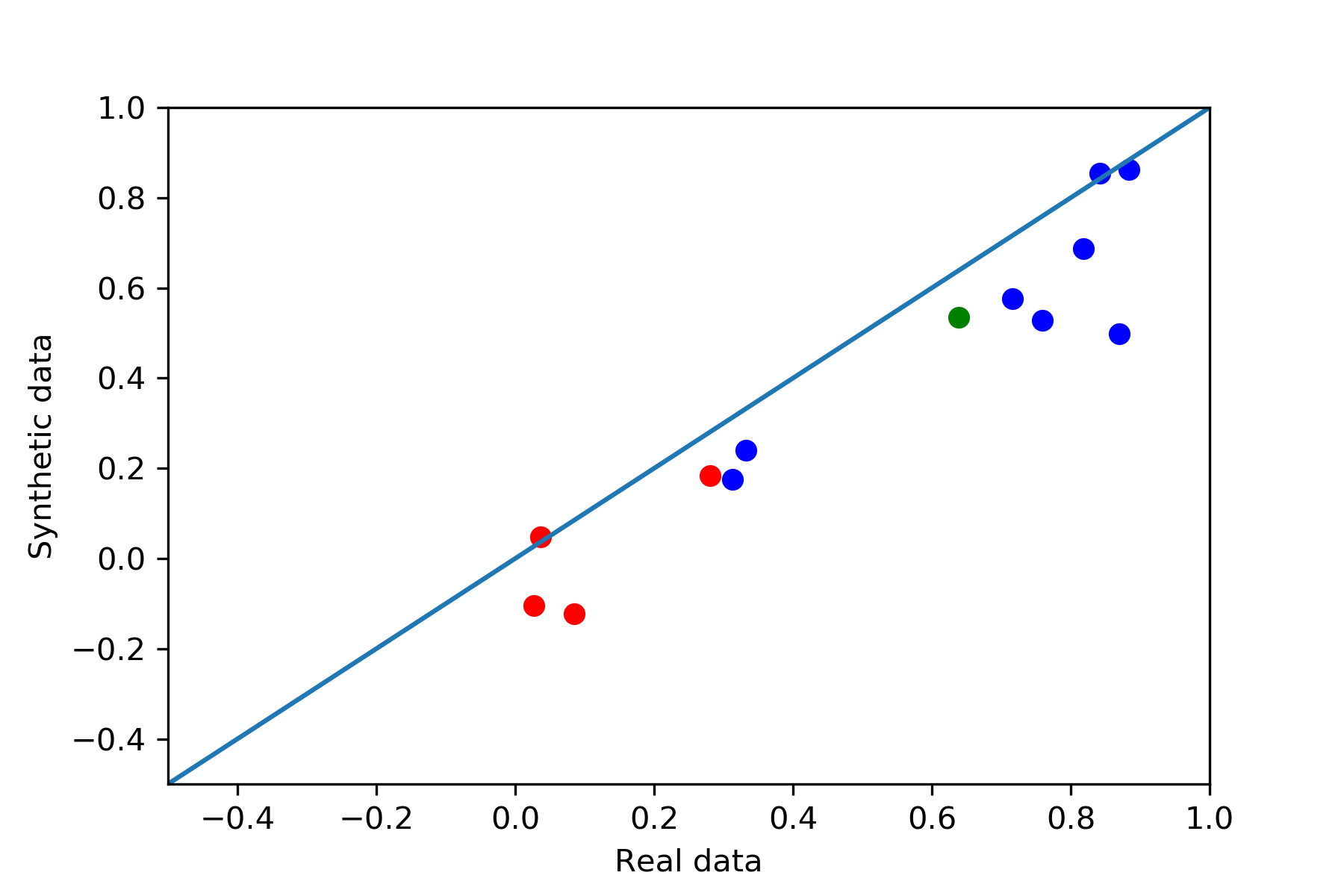}
        \caption{Our preprocessing, $\eps = 0.5$}
        \end{subfigure}
        \begin{subfigure}{0.32\textwidth}
        \includegraphics[width = \linewidth]{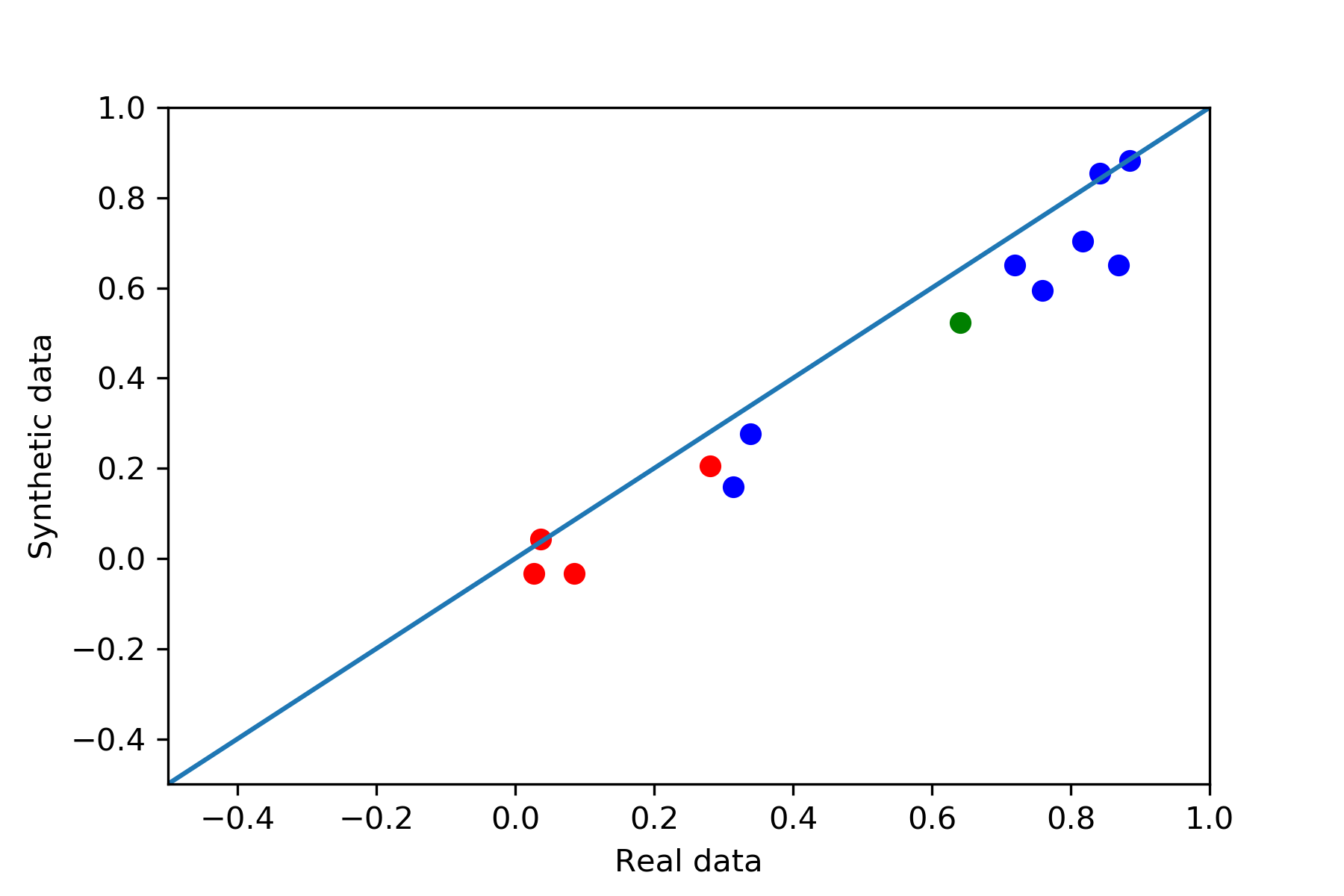}
        \caption{$\eps = 0.8$}
        \end{subfigure}
        \begin{subfigure}{0.32\textwidth}
        \includegraphics[width = \linewidth]{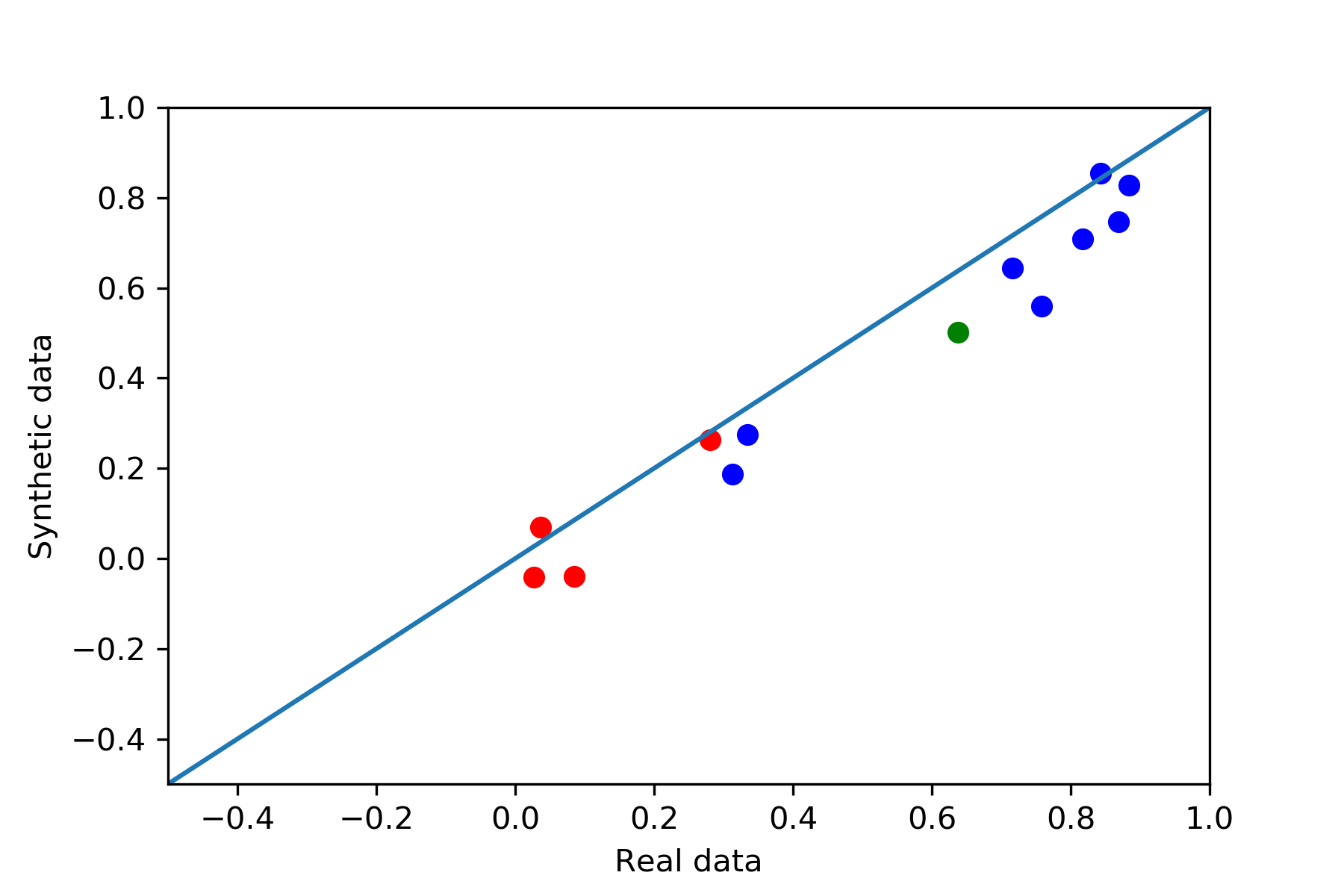}
        \caption{$\eps = 1.4$}
        \end{subfigure}   
        \begin{subfigure}{0.32\textwidth}
        \includegraphics[width = \linewidth]{abay/0_5_4_1.png}
        \caption{Original preprocessing, $\eps = 0.5$}
        \end{subfigure}
        \begin{subfigure}{0.32\textwidth}
        \includegraphics[width = \linewidth]{abay/0_8_4_1.png}
        \caption{$\eps = 0.8$}
        \end{subfigure}
        \begin{subfigure}{0.32\textwidth}
        \includegraphics[width = \linewidth]{abay/1_4_4_1.png}
        \caption{$\eps = 1.4$}
        \end{subfigure}           
                \caption{Dimension-wise prediction of synthetic data generated by DP-SYN using our preprocessing (top row) and the original preprocessing (bottom row) for different \(\eps\). The performance is measured by proximity to \(y=x\) line (the closer, the better). Both methods of preprocessing give  comparable performance.              
                }
                \label{fig:dim_pred_abay__preprocessing}
\end{figure}

When using the original preprocessing, we update the test set for dimension-wise prediction score to the same preprocessed format. We observed, however, that despite the architecture likely tuned for the original preprocessing, we ran DP-SYN using preprocessing and obtained a similar result; see Figure \ref{fig:dim_pred_abay__preprocessing}.  Dimension-wise prediction for both settings for several noise multiplier parameter and \(\eps\) is available at \url{https://github.com/DPautoGAN/DPautoGAN/tree/master/results/prediction-plots/DP-SYN/DP-SYN%20original%20vs%20our%20preprocessing}. Similar result for histograms are also observed and are available at \url{https://github.com/DPautoGAN/DPautoGAN/tree/master/results/hist_plot/DP-SYN%20original%20vs%20our%20preprocessing%20histogram}. For example, With regards to lack of diversity, we still observe that DP-SYN outputs only one class for race on all three \(\eps\) settings. DP-SYN no longer outputs one class for native-country as in the original preprocessing, yet we observed only two minority classes that are sampled, and they are largely oversampled; see Figure \ref{fig:DP-SYN-country} for histogram on native-country. The diversity scores using our preprocessing  improve on some features, but degrades more often, and overall degrades compared to the original preprocessing.

\begin{figure}[h]
        \centering
        \begin{subfigure}{0.083\textwidth}
        \includegraphics[width = \linewidth]{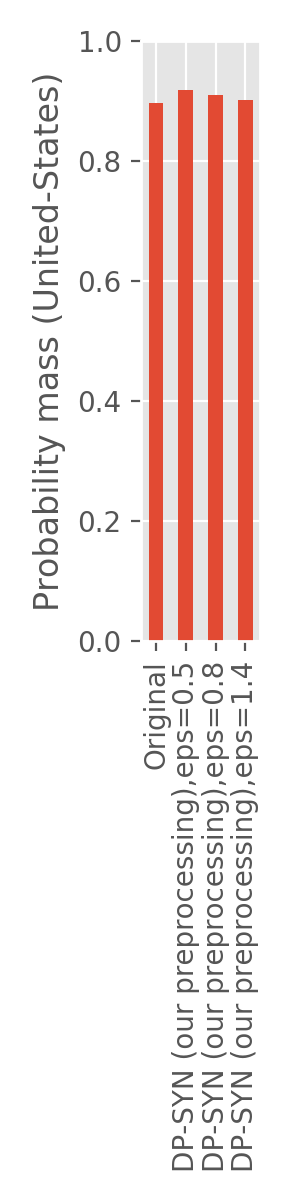}
        \caption{US}
        \end{subfigure} 
        \begin{subfigure}{0.88\textwidth}
        \includegraphics[width = \linewidth]{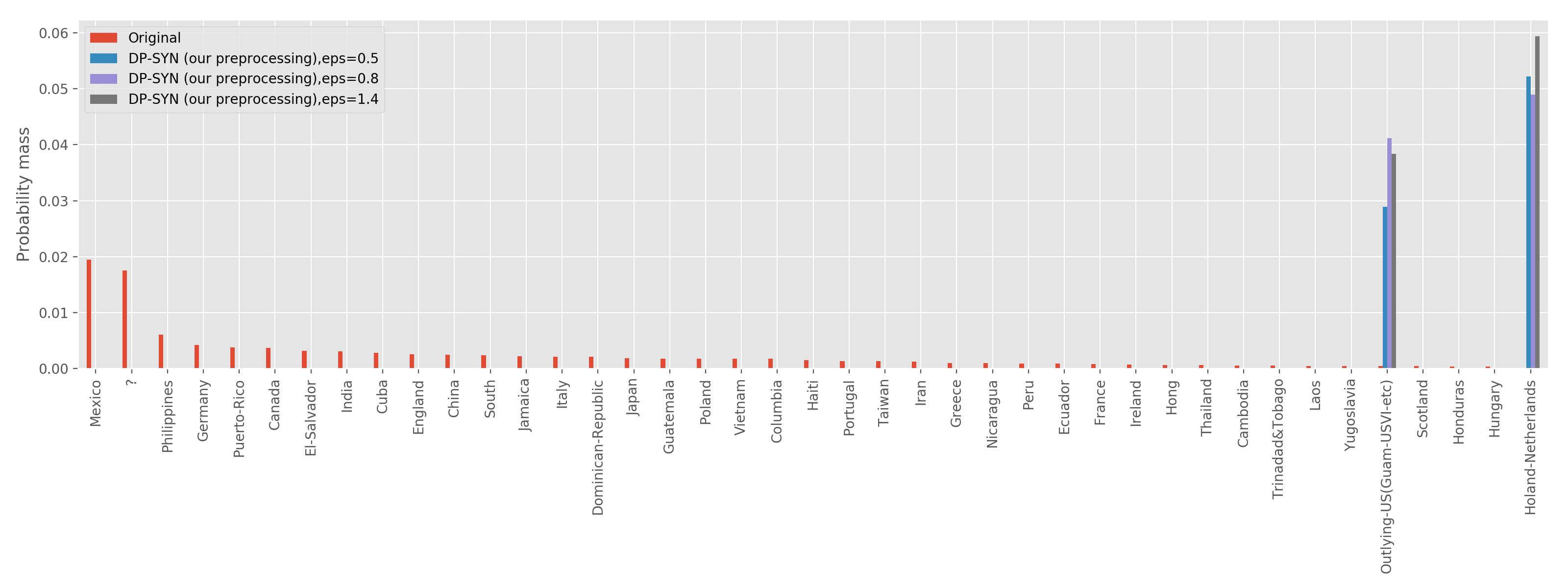}
        \caption{Minority classes}
        \end{subfigure}       
        \caption{Histogram of native-country features of the original data and synthetic data generated by DY-SYN using our preprocessing on different \(\eps\).}
\label{fig:DP-SYN-country}     
\end{figure}

In Table \ref{tab:diversity-total}, we report the diversity divergence of all eight applicable features (all nine categorical features except education due to a different preprocessing of DP-SYN) for different \(\eps\), tested on DP-SYN using our \dpgan preprocessing compared with the original. We  found that the original preprocessing of DP-SYN performs better on some features and worse on others, but better overall.   The same observation is found with \(\textrm{JSD}\) measure. Additional diversity scores can be found at \url{https://github.com/DPautoGAN/DPautoGAN/tree/master/results/diversity_divergence}. 

\begin{table*}[h]
        \caption{Diversity divergence measures \(D_{KL}^\mu\) on several features and the sum across all eight applicable  categorical features of DP-SYN \cite{abay2018privacy} on ADULT using the original preprocessing and our preprocessing of ADULT data. We report salary score x100 as scores are small on this feature. Smaller values imply more diverse synthetic data. The original preprocessing of DP-SYN performs slightly better on some features and worse or comparable on others, but better overall. The better performance between two types of preprocessing for each \(\eps\) value is highlighted in bold.}
        \label{tab:diversity-total}
        \centering
       \begin{tabular}{ccccccc}\hline
Preprocessing & \multicolumn{3}{c}{\textbf{Original}} & \multicolumn{3}{c}{\textbf{Ours} } \\
\(\eps\) \emph{values} &\emph{0.50} & \emph{0.80} & \emph{1.40} & \emph{0.50} & \emph{0.80} & \emph{1.40}\\\hline
Workclass & \textbf{0.074} & \textbf{0.074} & \textbf{0.084} & 0.103 & 0.123 & 0.099 \\ \hline
Marital-status & 0.017 & \textbf{0.011} & 0.012 & \textbf{0.016} & 0.014 & \textbf{0.010} \\ \hline
Occupation & \textbf{0.039} & 0.053 & \textbf{0.053} & 0.065 & \textbf{0.026} & 0.054 \\ \hline
Relationship & \textbf{0.029} & 0.032 & \textbf{0.036} & 0.034 & \textbf{0.017} & 0.037 \\ \hline
Race & \textbf{0.465} & \textbf{0.465} & \textbf{0.465} & 0.465 & 0.465 & 0.465 \\ \hline
Sex & \textbf{0.000} & \textbf{0.001} & \textbf{0.003} & 0.039 & 0.017 & 0.054 \\ \hline
Native-country & \textbf{0.364} & \textbf{0.364} & \textbf{0.364} & 0.436 & 0.445 & 0.452 \\ \hline
Salary (x100) & 0.027 & 0.027 & 0.027 & \textbf{0.001} & \textbf{0.001} & \textbf{0.001} \\ \hline
All & \textbf{0.99} & \textbf{1.00} & \textbf{1.02} &    1.16 & 1.11 & 1.17  \\\hline
\end{tabular}
\end{table*}

As we aim to report the optimized and most original of previous work, we choose the original preprocessing in reporting results in this work. We suspect, however, that the empirical results and conclusion would be similar even if we used ours preprocessing.
    
\paragraph{Parameter Tuning.} Original \(\eps\) and noise multiplier \(\psi\) used are \(\psi=2,4\) for \(\eps=2.4,3.2\) and \(\psi=4\) for \(\eps=1.6\) and \(\psi=4,8\) for \(\eps=1.2\). We tested on \(\psi=2,4,6,8\) for \(\eps=0.5,0.8,1.4\). \(\psi=2\) was not possible at \(\eps=0.5\) due to large \(\eps\) incurred even within the first epoch of training.  The implementation includes its own accuracy metric \cite{abay2018privacy}, which is an SVM classifier on the synthetic data. The model is trained 10 times for each noise and \(\eps\) setting, and the noise which gives the highest average accuracy score for each \(\eps\) setting is used. We saw a pattern of accuracy score either stay the same or increase from smallest \(\psi\) then later degrades, so we extended the range of \(\psi\) (up to 8 as mentioned) until we were certain that we have found a peak of accuracy score. We note that dimension-wise prediction performs similarly for \(\psi=2,4,6\) and slightly better at these \(\psi\) than the higher \(\psi=8\). Final \(\psi\) used for \(\eps=0.5,0.8,1.4\) is \(\psi=4\) for all three settings (and \(6,4,4\) if using our preprocessing). The results across all noise multiplier values can be found at \url{https://github.com/DPautoGAN/DPautoGAN/tree/master/results/prediction-plots/DP-SYN}. 

DP-SYN first partitions the data into groups based on number of unique labels, which is 2 in ``salary'' label of the ADULT data. DP-EM then chooses the number of clusters in a mixture of Gaussian in each group by Calinski-Harabasz criterion. The range of numbers of clusters tested (which is also from the original implementation) is \(K=1,2,\ldots,7\).

\paragraph{Privacy Accounting.} We keep the original privacy accounting, which is to split \(\eps,\delta\) into halves, each for autoencoder and DP-EM, in the original implementation of DP-SYN. We allow higher \(\eps\) which is computed by using standard composition instead of RDP composition on two phases of our \dpgan  training. We also note that DP-SYN treats the label column as public, which is used in partitioning the original data into groups based on the label, whereas DP-auto-GAN, DP-WGAN, and DP-VAE treat all features as private.  

The original implementation of DP-SYN was not able to finish a single epoch even for a  large noise multiplier \(\psi=16,32\) to achieve \(\eps=0.5\) by a standard composition. We found that this is due to a loose analysis of RDP in the original implementation. We increase the moment order of 32 in the original implementation to 96 to obtain a tighter DP analysis, which allows the \(\eps=0.5\) (standard composition) results reported in this work.

\subsection{Additional Details on Diversity Divergences}

Tables \ref{tab:diversity-full-JSD} and \ref{tab:diversity-full} show the divergence scores \(D_{KL}\) of all eight applicable features (all nine categorical features except ``education'' due to a difference in preprocessing; see Appendix \ref{app.adulttrain}). The smaller the divergence, the closer that the distribution of synthetic data resemble diversity of the original data. Both measures give similar results. Most features except race and native country have small divergence scores, which is consistent with the histogram plots that these two features have a strong majority and that generated data can underreresented the minority.  At small \(\eps\), DP-SYN performs well on several features, but seems to overfit and does not improve its diversity as \(\eps\) is larger. For larger \(\eps\), DP-VAE and \dpgan achieve better diversity. Overall, \dpgan achieves the best diversity on three \(\eps\) settings.  

Divergence scores across algorithms, \(\eps\) settings, and  \(\mu\) values are also available at \url{https://github.com/DPautoGAN/DPautoGAN/tree/master/results/diversity_divergence}.

\begin{table*}[h]
        \caption{Diversity divergence measure \(\textrm{JSD}\) on each of all eight applicable  features of ADULT data and the sum of divergences over them. Smaller values imply more diverse synthetic data. For each row (feature), the smallest value for each setting of \(\eps\) is highlighted in bold. Many values on the salary feature are in the order of \(10^{-5}\) and \(10^{-4}\), and  we report the values times 100 in order to distinguish better algorithms.}
        \label{tab:diversity-full-JSD}
        \centering
        \begin{tabular}{ccccccccccccc}\hline
 & \multicolumn{3}{c}{\textbf{DP-auto-GAN}} & \multicolumn{3}{c}{\textbf{DP-WGAN} } & \multicolumn{3}{c}{\textbf{DP-VAE}  } & \multicolumn{3}{c}{\textbf{DP-SYN} } \\
\(\eps\) \emph{values} & \emph{0.36} & \emph{0.51} & \emph{1.01}  & \emph{0.36} & \emph{0.51} & \emph{1.01} & \emph{0.36} & \emph{0.51} & \emph{1.01} & \emph{0.50} & \emph{0.80} & \emph{1.40}\\\hline
Workclass & .086 & \textbf{.040} & .032 & .271 & .119 & .090 & .085 & .067 & \textbf{.028} & \textbf{.040} & .043 & .048 \\ \hline
Marital\ & .025 & .043 & \textbf{.014} & .119 & .624 & .136 & .139 & .043 & .021 & \textbf{.017} & \textbf{.013} & .017 \\ \hline
Occupation & .107 & .064 & .089 & .100 & .522 & .228 & \textbf{.066} & \textbf{.048} & \textbf{.024} & .081 & .099 & .098 \\ \hline
Relationship & .042 & .035 & .011 & .190 & .614 & .187 & \textbf{.019} & \textbf{.011} & \textbf{.006} & .024 & .025 & .028 \\ \hline
Race & \textbf{.021} & \textbf{.014} & .016 & .081 & .053 & .040 & .095 & .031 & \textbf{.011} & .053 & .053 & .053 \\ \hline
Sex & .010 & .008 & .001 & .014 & .066 & .006 & .024 & .001 & \textbf{.000} & \textbf{.000} & \textbf{.000} & .001 \\ \hline
Nat.-country & .037 & \textbf{.028} & .032 & .516 & .036 & \textbf{.021} & .361 & .236 & .134 & \textbf{.037} & .037 & .037 \\ \hline
Salary (x 100) & .539 & \textbf{.002} & \textbf{.003} & \textbf{.000} & 37.4 & 2.537 & .829 & .003 & .196 & .007 & .007 & .007 \\ \hline
Total & {0.33} & \textbf{0.23} & \textbf{0.19} & 1.29 & 2.41 & 0.73 & 0.80 & 0.44 & 0.23 & \textbf{0.25} & 0.27 & 0.28 \\\hline
\end{tabular}  
\end{table*}

\begin{table*}[h]
        \caption{Same setting as Table \ref{tab:diversity-full-JSD} but using diversity divergence measure \(D_{KL}^\mu\) with \(\mu=e^{-\frac{1}{1-p_1}}\) where \(p_1\) is the maximum probability across all categories of that feature in the original data}
        \label{tab:diversity-full}
        \centering
        \begin{tabular}{ccccccccccccc}\hline
 & \multicolumn{3}{c}{\textbf{DP-auto-GAN}} & \multicolumn{3}{c}{\textbf{DP-WGAN} } & \multicolumn{3}{c}{\textbf{DP-VAE}  } & \multicolumn{3}{c}{\textbf{DP-SYN} } \\
\(\eps\) \emph{values} & \emph{0.36} & \emph{0.51} & \emph{1.01}  & \emph{0.36} & \emph{0.51} & \emph{1.01} & \emph{0.36} & \emph{0.51} & \emph{1.01} & \emph{0.50} & \emph{0.80} & \emph{1.40}\\\hline
Workclass & .155 & \textbf{.054} & .055 & .046 & .208 & .177 & .196 & .124 & \textbf{.039} & \textbf{.074} & .074 & .084 \\ \hline
Marital & .019 & .053 & \textbf{.005} & .166 & .164 & .290 & .207 & .044 & .017 & \textbf{.017} & \textbf{.011} & .012 \\ \hline
Occupation & .059 & .031 & .064 & .065 & .408 & .210 & .044 & \textbf{.022} & \textbf{.010} & \textbf{.039} & .053 & .053 \\ \hline
Relationship & .036 & .027 & \textbf{.006} & .375 & .079 & .321 & .034 & \textbf{.012} & .009 & \textbf{.029} & .032 & .036 \\ \hline
Race & \textbf{.125} & \textbf{.064} & .090 & .262 & .465 & .277 & .315 & .102 & \textbf{.038} & .465 & .465 & .465 \\ \hline
Sex & .037 & .028 & .003 & .051 & .255 & .020 & .085 & .003 & \textbf{.000} & \textbf{.000} & \textbf{.001} & .003 \\ \hline
Nat.-country & \textbf{.361} & \textbf{.227} & .304 & .298 & .351 & \textbf{.142} & .606 & .865 & .463 & .364 & .364 & .364 \\ \hline
Salary (x100) & 2.005 & \textbf{.007} & \textbf{.013} & \textbf{.002} & 246. & 9.421 & 3.081 & .013 & .736 & .027 & .027 & .027 \\ \hline
Total & \textbf{0.81} & \textbf{0.48} & \textbf{0.53} & 5.26 & 6.39 & 1.53 & 2.52 & 1.17 & 0.58 & 0.99 & 1.00 & 1.02 \\\hline
\end{tabular}  
\end{table*}

\end{document}